\documentclass[journal]{IEEEtran}

\usepackage[T1]{fontenc}
\usepackage[utf8]{inputenc}
\usepackage{microtype}

\usepackage{amsmath,amssymb,amsfonts,amsthm}

\usepackage{mathtools}
\usepackage{bm}

\usepackage{cite}

\usepackage{graphicx}
\usepackage{subcaption}
\usepackage{booktabs}
\usepackage{array}

\usepackage{xcolor}

\usepackage{url}
\usepackage{hyperref}
\hypersetup{hidelinks}

\usepackage{tikz}

\usepackage{enumitem}
\newcommand{\Eaff}{\mathbb E}
\newcommand{\Vtrans}{\mathbb V}
\newcommand{\Vdual}{\mathbb V^*}
\newcommand{\R}{\mathbb R}
\newcommand{\Rthree}{\mathbb R^3}

\newcommand{\SphTwo}{\mathbb S^2}
\newcommand{\Spp}{\mathbb S_{++}}

\newcommand{\Sym}{\operatorname{Sym}}

\newcommand{\norm}[1]{\left\lVert#1\right\rVert}

\newcommand{\inner}[2]{\left\langle#1,#2\right\rangle}

\newcommand{\col}[1]{\operatorname{col}\!\left(#1\right)}
\newcommand{\vech}{\operatorname{vech}}

\newcommand{\im}{\operatorname{im}}

\newcommand{\topT}{^{\mathsf T}}

\newcommand{\eye}[1]{\bm I_{#1}}
\newcommand{\zerovec}[1]{\bm 0_{#1}}

\newcommand{\argmin}{\operatorname*{arg\,min}}

\newcommand{\flatop}{^{\flat}}
\newcommand{\sharpop}{^{\sharp}}

\newcommand{\disp}[2]{\overrightarrow{#1#2}}

\newcommand{\Frame}[1]{\mathcal F_{\mathsf{#1}}}
\newcommand{\fw}{w}
\newcommand{\frameW}{\Frame{\fw}}
\newcommand{\originW}{O_{\mathsf w}}

\newcommand{\coord}[2]{\bm{#2}^{\mathsf{#1}}}

\newcommand{\pointcoord}[2]{\bm p_{#2}^{\mathsf{#1}}}

\newcommand{\loadPoint}{P}
\newcommand{\loadPointEq}{P_{\mathrm e}}
\newcommand{\actualAnchorPoint}[1]{P_{a,#1}}
\newcommand{\actualAnchorPointEq}[1]{P_{a,#1,\mathrm e}}
\newcommand{\desiredAnchorPoint}[1]{P_{d,#1}}

\newcommand{\setpointTuple}{Q}

\newcommand{\loadPos}{\pointcoord{\fw}{\loadPoint}}
\newcommand{\loadPosEq}{\pointcoord{\fw}{\loadPointEq}}
\newcommand{\actualAnchorPos}[1]{\pointcoord{\fw}{\actualAnchorPoint{#1}}}

\newcommand{\desiredAnchorPos}[1]{\pointcoord{\fw}{\desiredAnchorPoint{#1}}}

\newcommand{\setpointCoord}{\bm q^{\mathsf w}}
\newcommand{\setpointVariation}{\delta\bm q^{\mathsf w}}
\newcommand{\setpointVelocity}{\dot{\bm q}^{\mathsf w}}

\newcommand{\virtualCableVector}[1]{r_{#1}}

\newcommand{\cableLength}[1]{l_{0,#1}}
\newcommand{\cableTension}[1]{T_{#1}}
\newcommand{\cableDir}[1]{u_{#1}}
\newcommand{\cableDirCoord}[1]{\coord{\fw}{u_{#1}}}
\newcommand{\cableForceCovector}[1]{\varphi_{#1}}
\newcommand{\cableForceVector}[1]{f_{#1}}
\newcommand{\cableForceCoord}[1]{\coord{\fw}{f_{#1}}}

\newcommand{\CableLengthMap}[1]{\ell_{#1}}

\newcommand{\CableManifold}[1]{\mathcal M_{#1}}

\newcommand{\PerpProjector}[1]{\mathcal P_{u_{#1}}^{\perp}}
\newcommand{\PerpProjectorMat}[1]{\bm P_{u_{#1}}^{\perp,\mathsf w}}

\newcommand{\ParallelProjectorMat}[1]
{\bm P_{u_{#1}}^{\parallel,\mathsf w}}

\newcommand{\geometricStiffnessCoeff}[1]
{\kappa_{\mathrm{geo},#1}}

\newcommand{\AnchorStiffness}[1]{\mathcal K_{p,#1}}
\newcommand{\AnchorCompliance}[1]{\mathcal C_{p,#1}}

\newcommand{\AnchorStiffnessMat}[1]{\bm K_{p,#1}^{\mathsf w}}
\newcommand{\AnchorComplianceMat}[1]{\bm C_{p,#1}^{\mathsf w}}

\newcommand{\GeometricCompliance}[1]{\mathcal C_{\mathrm{geo},#1}}
\newcommand{\LegStiffness}[1]{\mathcal K_{#1}}
\newcommand{\TotalStiffness}{\mathcal K}

\newcommand{\GeometricComplianceMat}[1]{\bm C_{\mathrm{geo},#1}^{\mathsf w}}
\newcommand{\LegStiffnessMat}[1]{\bm K_{#1}^{\mathsf w}}
\newcommand{\TotalStiffnessMat}{\bm K^{\mathsf w}}
\newcommand{\DesiredStiffnessMat}{\bm K_{d}^{\mathsf w}}
\newcommand{\EndEffectorComplianceMat}{\bm C_{\mathrm{EE}}^{\mathsf w}}

\newcommand{\isotropicStiffness}[1]{k_{#1}}
\newcommand{\transverseStiffness}[1]{\gamma_{#1}}
\newcommand{\virtualDistance}[1]{d_{#1}}

\newcommand{\loadMass}{m}
\newcommand{\gravityScalar}{g}

\newcommand{\gravityVector}{f_g}
\newcommand{\externalVector}{f_{\mathrm{ext}}}
\newcommand{\gravityCoord}{\coord{\fw}{f_g}}
\newcommand{\externalCoord}{\coord{\fw}{f_{\mathrm{ext}}}
}

\newcommand{\EquilibriumSet}{\mathcal E_{\mathrm{eq}}}
\newcommand{\EquilibriumResidual}{\mathcal H}
\newcommand{\EquilibriumResidualCoord}{\bm h^{\mathsf w}}
\newcommand{\EquilibriumBranch}[1]{P_{\mathrm e}^{(#1)}}
\newcommand{\EquilibriumBranchCoord}[1]{\bm p_{\mathrm e}^{\mathsf w,(#1)}}

\newcommand{\StiffnessMap}[1]{\boldsymbol\Phi^{(#1)}}

\newcommand{\StiffnessCoordMap}[1]{\bm\Phi^{\mathsf w,(#1)}}
\newcommand{\StiffnessJacobian}[1]{\bm J_K^{\mathsf w,(#1)}}
\newcommand{\StiffnessError}{\bm e_K}

\newcommand{\CombinedJacobian}{\bm J_y^{\mathsf w}}

\newcommand{\StiffnessWeightMat}{\bm W_K}
\newcommand{\StiffnessErrorEnergy}{V_K}
\newcommand{\WeightedNorm}[2]
{\left\lVert #1\right\rVert_{#2}}

\newcommand{\NumConstraints}{N_c}

\newcommand{\AdmissibilityCoordFunction}[1]
{c_{#1}^{\mathsf w}}
\newcommand{\AdmissibilityRateFunction}[1]{\alpha_{#1}}

\newcommand{\AdmissibleBranchCoordDomain}[1]
{\mathcal U_{\mathrm{adm}}^{\mathsf w,(#1)}}

\newcommand{\MinTensionConstraint}[1]
{c_{T,#1}^{\min}}
\newcommand{\MaxTensionConstraint}[1]
{c_{T,#1}^{\max}}
\newcommand{\MinTension}{T_{\min}}
\newcommand{\MaxTension}{T_{\max}}

\newcommand{\CandidateSetpointVelocity}
{\bm\nu^{\mathsf w}}
\newcommand{\OptimalSetpointVelocity}
{\big(\bm\nu^{\mathsf w}\big)^\star}
\newcommand{\MinSetpointVelocity}
{\dot{\bm q}_{\min}^{\mathsf w}}
\newcommand{\MaxSetpointVelocity}
{\dot{\bm q}_{\max}^{\mathsf w}}

\newcommand{\StiffnessRateGain}{\lambda_K}
\newcommand{\VelocityRegularization}{\varepsilon_v}
\newcommand{\NominalStiffnessVelocity}
{\dot{\bm k}_{\mathrm{nom}}}
\newcommand{\DesiredStiffnessVelocity}
{\dot{\bm k}_d}

\newcommand{\StiffnessErrorTolerance}{\varepsilon_K}
\newcommand{\DescentTolerance}{\varepsilon_d}
\newcommand{\VelocityTolerance}{\varepsilon_{\nu}}

\newcommand{\EmpiricalStiffnessMat}
{\bm K_{\mathrm{emp}}^{\mathsf w}}

\newcommand{\PredictedStiffnessMat}
{\bm K_{\mathrm{pred}}^{\mathsf w}}

\newcommand{\StiffnessRelativeError}
{\varepsilon_{K,\mathrm{rel}}}

\newcommand{\NumPerturbations}
{N_{\mathrm p}}

\newcommand{\ForcePerturbation}[1]
{\Delta\bm f_{#1}^{\mathsf w}}

\newcommand{\PositionPerturbation}[1]
{\Delta\bm p_{#1}^{\mathsf w}}

\newcommand{\PredictedDisplacement}[1]
{\Delta\bm p_{\mathrm{pred},#1}^{\mathsf w}}

\newcommand{\EmpiricalDisplacement}[1]
{\Delta\bm p_{\mathrm{emp},#1}^{\mathsf w}}

\newcommand{\TestForce}[1]
{\Delta\bm f_{#1}^{\mathsf w}}

\newtheorem{theorem}{Theorem}
\newtheorem{proposition}{Proposition}

\newtheorem{corollary}{Corollary}
\newtheorem{remark}{Remark}

\usetikzlibrary{
    arrows.meta,
    positioning,
    calc,
    fit,
    backgrounds,
    shapes.geometric,
    decorations.pathmorphing
}

\tikzset{
    theoryblock/.style={
        draw=blue!65!black,
        fill=blue!5,
        rounded corners=1.5pt,
        align=center,
        minimum height=10mm,
        text width=28mm,
        inner sep=3pt,
        font=\footnotesize
    },
    coordinateblock/.style={
        draw=black!65,
        fill=black!4,
        rounded corners=1.5pt,
        align=center,
        minimum height=10mm,
        text width=28mm,
        inner sep=3pt,
        font=\footnotesize
    },
    equilibriumblock/.style={
        draw=orange!80!black,
        fill=orange!10,
        rounded corners=1.5pt,
        align=center,
        minimum height=10mm,
        text width=30mm,
        inner sep=3pt,
        font=\footnotesize
    },
    inputblock/.style={
        draw=orange!80!black,
        fill=orange!5,
        rounded corners=1.5pt,
        align=center,
        inner sep=3pt,
        font=\footnotesize
    },
    layerlabel/.style={
        font=\footnotesize\bfseries,
        text=black!75,
        anchor=east,
        align=right
    },
    diagramnote/.style={
        font=\scriptsize,
        text=black!70,
        align=center
    },
    flowarrow/.style={
        -{Latex[length=2mm]},
        thick,
        draw=black!75
    },
    dependence/.style={
        -{Latex[length=1.8mm]},
        dashed,
        draw=orange!80!black
    },
    representation/.style={
        -{Latex[length=1.8mm]},
        densely dashed,
        draw=black!60
    },
    feasibleblock/.style={
        draw=green!50!black,
        fill=green!8,
        rounded corners=1.5pt,
        align=center,
        minimum height=8mm,
        inner sep=4pt,
        font=\footnotesize
    },
    warningblock/.style={
        draw=red!70!black,
        fill=red!6,
        rounded corners=1.5pt,
        align=center,
        minimum height=8mm,
        inner sep=4pt,
        font=\footnotesize
    }
}

\newcommand{\QuadrotorGlyph}[4]{%
    \begin{scope}[
        shift={#1},
        rotate=#2,
        scale=#3,
        yscale=#4
    ]

        \draw[
            gray!55,
            line width=0.65pt,
            line cap=round
        ]
            (-0.46,-0.46) -- (0.46,0.46);

        \draw[
            gray!65,
            line width=0.75pt,
            line cap=round
        ]
            (-0.46,0.46) -- (0.46,-0.46);

        \draw[
            gray!65,
            fill=gray!12,
            line width=0.6pt,
            rounded corners=0.5pt
        ]
            (-0.13,-0.09)
            rectangle
            (0.13,0.09);

        \foreach \x/\y in {
            -0.46/-0.46,
             0.46/ 0.46,
            -0.46/ 0.46,
             0.46/-0.46
        }{
            \draw[
                gray!55,
                fill=gray!8,
                line width=0.55pt
            ]
                (\x,\y)
                ellipse[
                    x radius=0.22cm,
                    y radius=0.075cm
                ];
        }

        \draw[
            -{Latex[length=0.8mm]},
            gray!55,
            line width=0.45pt
        ]
            (-0.62,-0.46)
            arc[
                start angle=180,
                end angle=20,
                x radius=0.16cm,
                y radius=0.055cm
            ];

        \draw[
            -{Latex[length=0.8mm]},
            gray!55,
            line width=0.45pt
        ]
            (0.62,0.46)
            arc[
                start angle=0,
                end angle=200,
                x radius=0.16cm,
                y radius=0.055cm
            ];

        \draw[
            -{Latex[length=0.8mm]},
            gray!45,
            line width=0.4pt
        ]
            (-0.30,0.46)
            arc[
                start angle=0,
                end angle=-160,
                x radius=0.16cm,
                y radius=0.055cm
            ];

        \draw[
            -{Latex[length=0.8mm]},
            gray!45,
            line width=0.4pt
        ]
            (0.30,-0.46)
            arc[
                start angle=180,
                end angle=340,
                x radius=0.16cm,
                y radius=0.055cm
            ];

    \end{scope}%
}

\newif\ifshowrevs
\showrevsfalse
\newcommand{\rev}[1]{%
     \ifshowrevs\textcolor{blue!65!black}{#1}\else#1\fi}

\title{Passive Stiffness Shaping in Cable-Suspended Aerial Manipulation via Movable Compliant Anchors}

\author{%
Antonio~Franchi$^{1,2}$ and Amr~Afifi$^{3}$%
\thanks{%
$^{1}$ Robotics and Mechatronics Lab, Faculty of Electrical Engineering,
Mathematics \& Computer Science, University of Twente, Enschede,
The Netherlands.
\href{mailto:schol@r-franchi.eu}{schol@r-franchi.eu}.%
}%
\thanks{%
$^{2}$ Department of Computer, Control and Management Engineering,
Sapienza University of Rome, 00185 Rome, Italy.%
}%
\thanks{%
$^{3}$ Department of Information and Computing Sciences,
Utrecht University, Utrecht, The Netherlands.  \href{a.n.m.g.afifi@uu.nl}{a.n.m.g.afifi@uu.nl}%
}%
}

\begin{document}

\maketitle

\begin{abstract}
Cable-suspended aerial manipulation offers a lightweight architecture
for cooperative transportation and physical interaction, yet the
passive mechanical response perceived at the load remains insufficiently
understood and systematically exploited. This work interprets aerial
vehicles as movable compliant anchors and develops a gravity-aware
quasi-static theory for predicting and shaping the passive Cartesian
stiffness of a suspended load. The formulation applies to an arbitrary
number of aerial vehicles connected to a point load by taut, straight,
inextensible cables. At a selected gravity-loaded equilibrium,
aerial-anchor compliance and transverse cable geometric compliance
combine in series within each leg, while the leg stiffnesses act in
parallel on the load. For isotropic aerial-anchor behavior, each leg is
exactly equivalent to a virtual unilateral elastic cable, revealing an
axial--transverse stiffness decomposition governed by the equilibrium
tension. These results define a nonlinear map from commanded-anchor
configuration to passive load stiffness, whose differential enables
local constraint-preserving shaping through anchor repositioning.
\rev{A dynamic rigid-body validation framework with nonlinear vehicle control,
elastic-damped tendons, and environmental contact is defined to assess
when and to what extent the derived stiffness remains predictive beyond
the assumptions of the analytical model.}
\end{abstract}


\section{Introduction}
\label{sec:introduction}

Cable-suspended aerial manipulation (CSAM) provides a lightweight and
mechanically simple means for teams of aerial robots to transport and
manipulate loads beyond the capabilities of a single vehicle. These
systems are challenging because the load is indirectly actuated,
cable forces are unilateral, and the coupled behavior depends on
geometry, tension, gravity, and vehicle dynamics. In recent years the field has
matured from foundational multi-quadrotor modeling and control to a broad range of transportation and manipulation methods
\cite{SreenathKumar2013RSS,Estevez2024ReviewSuspendedLoads}. 

Against this background, this work introduces a different viewpoint:
the aerial vehicles are regarded as movable compliant anchors whose
commanded positions determine the passive Cartesian stiffness perceived
at the suspended load. Figure~\ref{fig:system_overview} offers an early
preview of this viewpoint and its central mechanism: aerial-anchor
repositioning changes the gravity-loaded cable directions and tensions
that determine, and thereby reshape, the passive stiffness ellipsoid.

\begin{figure}[t]
    \centering

    \begin{tikzpicture}[
        x=1cm,
        y=1cm,
        actualanchor/.style={
            circle,
            draw=blue!70!black,
            fill=blue!65,
            minimum size=2mm,
            inner sep=0pt
        },
        commandedanchor/.style={
            circle,
            draw=blue!70!black,
            fill=white,
            line width=0.8pt,
            minimum size=2.2mm,
            inner sep=0pt
        },
        loadpoint/.style={
            circle,
            draw=black!85,
            fill=black!85,
            minimum size=2mm,
            inner sep=0pt
        },
        complianceconnection/.style={
            draw=blue!65!black,
            line width=0.75pt,
            decorate,
            decoration={
                coil,
                aspect=0.45,
                segment length=1.4mm,
                amplitude=0.8mm
            }
        },
        cable/.style={
            draw=black!80,
            line width=0.9pt
        },
        alternativecable/.style={
            draw=gray!65,
            densely dashed,
            line width=0.75pt
        },
        tensionarrow/.style={
            -{Latex[length=1.5mm]},
            draw=orange!85!black,
            line width=0.9pt
        },
        gravityarrow/.style={
            -{Latex[length=1.7mm]},
            draw=orange!85!black,
            line width=1pt
        },
        worldaxis/.style={
            -{Latex[length=1.25mm]},
            draw=black!70,
            line width=0.65pt
        },
        shapingarrow/.style={
            -{Latex[length=1.8mm]},
            draw=blue!70!black,
            line width=1pt
        },
        paneltitle/.style={
            font=\small\bfseries,
            align=center
        },
        mainlabel/.style={
            font=\footnotesize,
            align=center
        },
        smalllabel/.style={
            font=\footnotesize,
            align=center
        },
        equationnote/.style={
            font=\footnotesize,
            text=black!65,
            align=center
        }
    ]
    \path[use as bounding box]
        (0.00,-4.45)
        rectangle
        (8.40,7.00);

    \node[paneltitle] at (4.20,6.57)
    {(a) Movable compliant aerial-anchor system};
    \QuadrotorGlyph{(1.45,4.90)}{25}{1.22}{0.58}
    \QuadrotorGlyph{(4.70,5.05)}{-2}{1.22}{0.48}
    \QuadrotorGlyph{(6.95,4.90)}{-30}{1.22}{0.58}
    \node[actualanchor] (A1) at (1.45,4.90) {};
    \node[actualanchor] (A2) at (4.70,5.05) {};
    \node[actualanchor] (A3) at (6.95,4.90) {};

    \node[smalllabel, below left=-2mm of A1, xshift=-0mm]
        {$\actualAnchorPoint{1}$};
    \node[smalllabel, right=0mm of A2]
        {$\actualAnchorPoint{2}$};
    \node[smalllabel, below right=0-1mm of A3]
        {$\actualAnchorPoint{3}$};

    \node[commandedanchor] (D1) at (0.95,5.85) {};
    \node[commandedanchor] (D2) at (4.55,6.02) {};
    \node[commandedanchor] (D3) at (7.48,5.65) {};

    \node[smalllabel, left=-1mm of D1]
        {$\desiredAnchorPoint{1}$};
    \node[smalllabel, right=-1mm of D2]
        {$\desiredAnchorPoint{2}$};
    \node[smalllabel, right=-1mm of D3]
        {$\desiredAnchorPoint{3}$};
    \draw[complianceconnection] (A1) -- (D1);
    \draw[complianceconnection] (A2) -- (D2);
    \draw[complianceconnection] (A3) -- (D3);

    \node[
        smalllabel,
        text=blue!70!black,
        text width=20mm
    ] at (2.9,5.43)
        {local controller\\ closed-loop\\stiffness
        $\AnchorStiffness{i}$\\[0.1em]
        Eq.~\eqref{eq:vehicle_equil}};

    \node[loadpoint] (P) at (4.20,2.15) {};

    \node[mainlabel, left=0mm of P]
        {load $\loadPoint$};

    \draw[cable] (P) -- (A1);
    \draw[cable] (P) -- (A2);
    \draw[cable] (P) -- (A3);
    \draw[tensionarrow]
        ($(P)!0.022!(A1)$)
        --
        ($(P)!0.47!(A1)$);

    \node[
        smalllabel,
        anchor=east,
        yshift=-4mm,
        xshift=4mm,
        text width=21mm
    ] at ($(P)!0.58!(A1)+(-0.08,0.03)$)
        {$\cableDir{i}$,\;
        $\cableLength{i}$,\;
        $\cableTension{i}$\\[0.1em]
        Eq.~\eqref{eq:cable_geometry}};
    \node[
        smalllabel,
        text=black!65
    ] at (6.55,3.2)
        {$i=1,\ldots,n$};

    \begin{pgfonlayer}{background}
        \draw[
            draw=green!50!black,
            fill=green!35,
            fill opacity=0.22,
            line width=0.8pt,
            rotate around={18:(P)}
        ]
            (P) ellipse[x radius=1.12cm,y radius=0.43cm];

        \draw[
            draw=green!45!black,
            fill=none,
            densely dashed,
            line width=0.55pt,
            rotate around={18:(P)}
        ]
            (P) ellipse[x radius=0.36cm,y radius=0.43cm];
    \end{pgfonlayer}

    \node[
        smalllabel,
        text=green!35!black,
        anchor=west
    ] at (4.6,2.04)
        {passive stiffness\\
        $\mathcal E_K,\ \TotalStiffness$\\[0.1em]
        Eq.~\eqref{eq:stiffness_compliance_ellipsoids}};

    \draw[gravityarrow]
        (P) -- ++(0,-1.05);

    \node[
        smalllabel,
        text=orange!85!black,
        anchor=west
    ] at (4.34,1.36)
        {$\gravityVector$};

    \node[equationnote] at (4.18,0.82)
        {Eq.~\eqref{eq:gravity_equilibrium}};

    \coordinate (W) at (1.58,1.23);

    \draw[worldaxis] (W) -- ++(0.68,0)
        node[smalllabel, right] {$x_{\mathsf w}$};

    \draw[worldaxis] (W) -- ++(0,0.68)
        node[smalllabel, above] {$z_{\mathsf w}$};

    \draw[worldaxis] (W) -- ++(-0.40,-0.30)
        node[smalllabel, below left] {$y_{\mathsf w}$};

    \node[smalllabel, below=1.2mm of W]
        {$\frameW$};

    \node[commandedanchor] (legendD) at (5.52,1.02) {};
    \node[actualanchor] (legendA) at (5.52,0.67) {};

    \node[smalllabel, anchor=west] at (5.72,1.02)
        {commanded anchor};

    \node[smalllabel, anchor=west] at (5.72,0.67)
        {actual anchor};

    \node[paneltitle] at (4.20,-0.10)
        {(b) Passive load-stiffness shaping by anchor repositioning};

    \begin{scope}[yshift=-3mm]

    \coordinate (PL) at (1.72,-1.73);

    \node[loadpoint] at (PL) {};

    \node[commandedanchor] (LD1) at (0.72,-0.64) {};
    \node[commandedanchor] (LD2) at (1.72,-0.42) {};
    \node[commandedanchor] (LD3) at (2.72,-0.64) {};

    \draw[cable] (PL) -- (LD1);
    \draw[cable] (PL) -- (LD2);
    \draw[cable] (PL) -- (LD3);

    \draw[
        draw=green!50!black,
        fill=green!35,
        fill opacity=0.24,
        line width=0.7pt
    ]
        (PL) ellipse[x radius=0.72cm,y radius=0.28cm];

    \node[mainlabel, align=left] at (0.92,-2.38)
        {initial\\setpoints
        $\setpointTuple$};

    \draw[shapingarrow]
        (3.04,-0.95) -- (5.33,-0.98);

    \node[
        smalllabel,
        text width=25mm
    ] at (4.18,-0.45)
        {aerial-anchor\\repositioning};

    \node[
        smalllabel,
        text width=31mm,
        text=orange!80!black
    ] at (4.18,-1.40)
        {changes directions and\\
        gravity-loaded tensions};

    \coordinate (PR) at (6.62,-1.73);
    \node[loadpoint] at (PR) {};

    \node[commandedanchor] (RD1) at (5.82,-0.42) {};
    \node[commandedanchor] (RD2) at (6.62,-0.76) {};
    \node[commandedanchor] (RD3) at (7.42,-0.42) {};

    \draw[cable] (PR) -- (RD1);
    \draw[cable] (PR) -- (RD2);
    \draw[cable] (PR) -- (RD3);
    \draw[
        draw=green!50!black,
        fill=green!35,
        fill opacity=0.24,
        line width=0.7pt,
        rotate around={58:(PR)}
    ]
        (PR) ellipse[x radius=0.82cm,y radius=0.24cm];

    \node[mainlabel, align=right] at (7.42,-2.38)
        {repositioned\\ setpoints
        ${\setpointTuple}'$};
    \draw[
        gray!65,
        densely dashed,
        line width=0.6pt
    ]
        (PL) -- (PR);

    \node[
        smalllabel,
        fill=white,
        inner sep=1pt,
        text=gray!70!black
    ] at (4.18,-2.67)
        {same regulated equilibrium position};

    \node[
        draw=black!30,
        fill=black!2,
        rounded corners=1.5pt,
        align=center,
        font=\footnotesize,
        text width=71mm,
        minimum height=8mm,
        inner sep=2.5pt
    ] at (4.20,-3.53)
        {$\displaystyle
        \setpointTuple
        \longmapsto
        \{
        \cableDir{i},
        \cableTension{i}
        \}_{i=1}^{n}
        \longmapsto
        \TotalStiffness$\\[0.6mm]
        {\footnotesize
        Eqs.~\eqref{eq:gravity_equilibrium}
        and~\eqref{eq:total_stiffness}}};

    \end{scope}

    \begin{pgfonlayer}{background}
        \path[
            draw=blue!25,
            fill=blue!1,
            rounded corners=2pt
        ]
            (0.10,0.36)
            rectangle
            (8.30,6.92);
        \path[
            draw=green!35!black,
            fill=green!1,
            rounded corners=2pt
        ]
            (0.10,-4.36)
            rectangle
            (8.30,0.20);

    \end{pgfonlayer}

    \end{tikzpicture}

    \caption{System architecture and passive-stiffness-shaping
    principle. Each taut cable connects the load point $\loadPoint$ to
    an actual aerial-anchor point $\actualAnchorPoint{i}$, while the
    finite closed-loop stiffness of aerial vehicle $i$ relates this
    point to its commanded anchor $\desiredAnchorPoint{i}$. Gravity
    determines the equilibrium cable directions and tensions.
    Repositioning the commanded anchors changes these operating-point
    quantities and thereby reshapes the passive Cartesian stiffness
    ellipsoid at the load. The three-vehicle drawing represents the
    general $n$-vehicle system.}

    \label{fig:system_overview}
\end{figure}
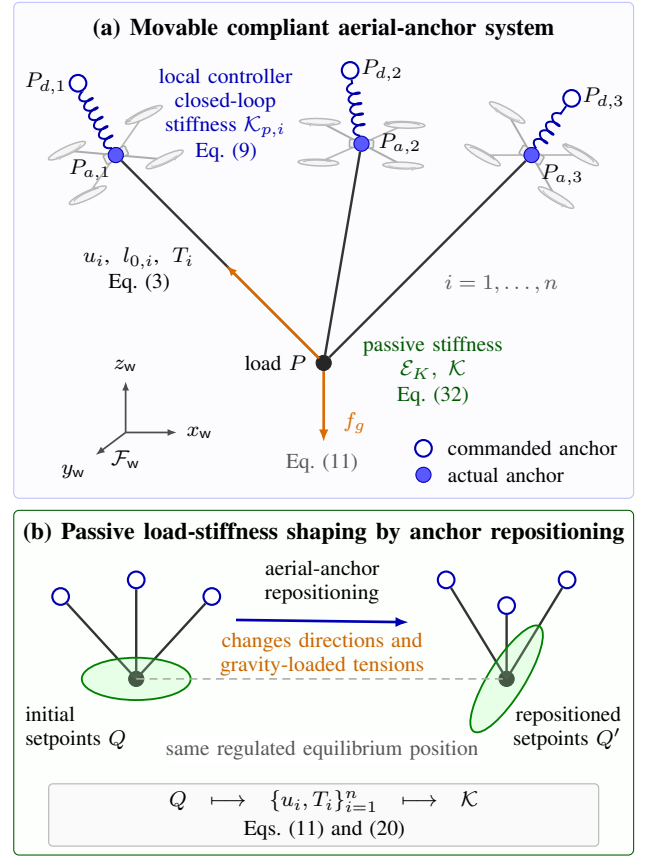

\subsubsection{A Quick Tour of Cable-Suspended Aerial Manipulation}

Foundational studies established geometric control and differential
flatness for a quadrotor carrying a cable-suspended load, together with
trajectory optimization for agile suspended-payload maneuvers
\cite{SreenathLeeKumar2013CDC,Foehn2017RSS}. For cooperative
transportation, control and vision-based coordination were developed
for two-vehicle systems, while robust position tracking addressed
unknown wind forces
\cite{PereiraDimarogonas2017CDC,Gassner2017ICRA,
PereiraCortesDimarogonas2020TAC}.

Subsequent research extended the focus from load translation to
cooperative manipulation and physical interaction. Internal cable
forces were characterized in relation to equilibrium, stability, and
passivity, and robust full-pose control was developed for uncertain
multi-UAV suspended-load systems
\cite{TognonGabellieri2018RAL,Sanalitro2020RAL}. Inertial estimation,
energy-aware control, and precise pick-and-place supported increasingly
accurate manipulation
\cite{Petitti2020ICUAS,JimenezCano2022JIRS}. Indirect force control,
equilibrium-sensitivity analysis, and force-based pose regulation
further addressed interaction forces and modeling uncertainty
\cite{Sanalitro2022RAL,Gabellieri2023TRO,Gabellieri2023IROS}.

Recent work has increased the agility, scalability, and autonomy of
cooperative suspended-load systems. Distributed trajectory
optimization, nonlinear model predictive control, and whole-body
kinodynamic planning have enabled scalable coordination and agile
coupled motion
\cite{Jackson2020RAL,LiLoianno2023IROS,DeCarli2025RAL,
Sun2025ScienceRobotics}. Integrated planning and distributed tracking
have further addressed safe transportation in complex environments
\cite{WangGao2025TRO}. Learning-based alternatives include
decentralized multi-agent policies and direct motor-control policies,
while physical human--robot collaboration has also been investigated
with multiple aerial robots
\cite{Zeng2025CoRL,Lorentz2026ICRA,Li2025TRO}.

\subsubsection{An Important Question Remains Open}

The literature reviewed above has established powerful methods for
motion feasibility, trajectory tracking, agility, coordination, and
robustness in cable-suspended aerial systems. These achievements remain
fundamental to reliable transportation and manipulation. As aerial
robots increasingly engage in physical interaction, however, the
mechanical response of the system to contact and disturbances becomes
an equally important consideration. Compliance can accommodate
geometric uncertainty, limit interaction forces, and reduce task
sensitivity to tracking errors.

Existing aerial-robotic approaches have generated compliant behavior
mainly through active feedback. Passivity-based compliance control has
been developed for aerial vehicles carrying articulated manipulators,
while hierarchical control has been used to impose compliant
interaction on a cable-suspended aerial manipulator
\cite{Kim2018PassiveComplianceAerialManipulators,
Gabellieri2020ComplianceSAM}. Indirect force control and precise
pick-and-place further demonstrate the importance of
feedback-controlled interaction behavior
\cite{Sanalitro2022RAL,JimenezCano2022JIRS}.

\emph{Active} compliance necessarily acts through sensing, computation,
actuation, and finite closed-loop bandwidth. Shaping the \emph{passive}
compliance of the physical system offers a complementary
opportunity: it modifies the local mechanical response at the load
before an outer feedback loop reacts. Passive stiffness shaping may
therefore be combined with active control to improve interaction
performance, disturbance accommodation, and safety. This viewpoint is
central to soft robotics and variable-stiffness actuation, where
adjustable physical compliance supports the performance--safety
trade-off
\cite{BicchiTonietti2004FastSoftArm}. In CSAM, it motivates the following
questions:
\begin{quote}
    {\it How can the passive compliance generated by the physical
    aerial-vehicle--cable network be \emph{predicted}, and how can it be
    systematically \emph{shaped} in a CSAM multi-aerial-robot system?}
\end{quote}
The remainder of this introduction identifies the relevant foundations, the distinguishing difficulty of CSAM, and the route followed in this work to answer these questions.

\subsubsection{What Can Be Learned from Neighboring Fields?}

Cable-driven parallel robots (CDPRs) provide the closest established
framework for addressing these questions. General parallel-robot and
cable-robot theory develops the underlying concepts of kinetostatics,
unilateral actuation, workspace, force distribution, and stiffness
\cite{Merlet2006ParallelRobots,Pott2018CableDrivenParallelRobots}.
Reference volumes and recent reviews consolidate the modeling, design,
control, and technological development of these systems
\cite{BruckmannPott2013CableRobots,GosselinCardou2017CableRobots,
Wang2025ReviewCDPR}.

Within this field, interval methods have addressed wrench-feasible
workspaces and direct geometrico-static computation
\cite{GouttefardeDaneyMerlet2011TRO,BertiMerletCarricato2016IJRR}.
Force control and cable dynamics have been studied using elastic,
flexible, and time-varying-length cable models
\cite{Kraus2016Thesis,Tempel2019Thesis}. Variable-stiffness
formulations and internal-force-based impedance control explicitly
treat stiffness as a design and control objective
\cite{KhosraviTaghirad2014VariableStiffness,Reichert2014Impedance},
while recent contributions have investigated stiffness-oriented
formation design and experimentally validated planar cable-robot
stiffness models
\cite{Dona2025StiffnessCDPR,Arslan2025StiffnessCDPR}.
Related aerial architectures include quadrotor-driven cable parallel
robots, cable-towed systems, and suspended macro--mini manipulators.
Their wrench capability, configuration planning, architectural
variability, and dynamic control have been investigated in
\cite{Erskine2019JMR,Erskine2019ICRA,LiErskine2020RAL,
Yigit2021IROS,Yigit2023TRO}. These results reinforce the relevance of
cable-robot concepts to aerial systems while also exposing the
architectural differences considered next.

\subsubsection{The Key Difference between CSAM and CDPR: Movable  Anchors}

CDPRs generally keep the cable-anchor positions fixed and vary
the cable lengths or tensions. CSAM systems instead use cables of
prescribed length whose aerial endpoints move in three dimensions:
fixed anchor and variable cable length become variable anchor position
and fixed cable length.
This difference changes the nature of cable-force generation. A
winch-driven actuator acts directly along a one-dimensional cable
coordinate. An aerial vehicle moves the cable endpoint through its
controlled spatial motion, arising from the coupled translational,
rotational, actuation, and feedback dynamics of the vehicle. The
endpoint therefore cannot be treated immediately as either an ideal
fixed anchor or a direct scalar cable actuator. \emph{How much of this
closed-loop aerial-vehicle behavior must be retained} to predict the
mechanical response perceived at the suspended load, and \emph{which
mathematically tractable local representation preserves the relevant physics}?

A second difficulty follows from the mobility of the anchors. Moving
an aerial endpoint changes the cable direction and the force required
to support the load under gravity. Conversely, a load displacement
changes the cable forces acting on the aerial endpoints. Anchor motion,
cable geometry, gravity-loaded equilibrium, and the passive response at
the load are therefore coupled rather than independently assignable.
CDPR theory provides \emph{essential concepts} and \emph{analytical tools} for
formulating these relationships, but \emph{it does not directly resolve them
for CSAM}.

\subsubsection{Contributions: What This Work Establishes, and Where It Stops}

The complexity identified above makes some reduction unavoidable if a
general analytical answer is sought. The central modeling hypothesis is that, near a
selected operating point and under quasi-static evolution, the
aerial-vehicle behavior relevant to the passive load response is its
local closed-loop translational force--displacement relation at the
cable endpoint. We represent this relation by a symmetric
positive-definite Cartesian stiffness map, while abstracting the
rotational, actuation, and transient dynamics that produce it. This
modeling choice retains three-dimensional, direction-dependent
endpoint compliance and provides a tractable description of an
arbitrary number of movable compliant aerial anchors
(Section~\ref{sec:system_model}).

Within this reduced model, passive stiffness can first be
\emph{predicted}. Gravity selects the equilibrium cable directions and
tensions at which the local response is evaluated. Aerial-anchor
compliance and transverse cable geometric compliance then combine in
series within each leg, while the leg stiffnesses act in parallel on
the load. This structure yields the passive Cartesian stiffness
perceived at the suspended load and identifies the physical mechanisms
contributing to it
(Section~\ref{sec:passive_cartesian_stiffness}).

The isotropic case further exposes the structure hidden in the general
expression. Each leg becomes exactly equivalent to a virtual unilateral
elastic cable connected directly to the commanded anchor. The
resulting axial--transverse decomposition shows how cable direction and
gravity-loaded tension determine the magnitude and anisotropy of the
passive stiffness
(Sections~\ref{sec:virtual_elastic_cable_equivalence}
and~\ref{sec:stiffness_decomposition}).

Prediction then makes systematic \emph{shaping} meaningful. Along a
selected equilibrium branch, the commanded-anchor configuration
induces a nonlinear gravity-aware stiffness map. Its differential
characterizes the stiffness variations locally generated by
commanded-anchor motion and supports constraint-preserving local
regulation. Thus, the theory answers the two central questions in
sequence: it derives the passive stiffness generated by the physical
network and then identifies how aerial-anchor repositioning can modify
it
(Sections~\ref{sec:gravity_aware_stiffness_map}
and~\ref{sec:quasistatic_stiffness_control}).

This analytical clarity carries a corresponding risk. The theory is
quasi-static and translational and considers a point load, taut,
straight, inextensible cables, and locally constant aerial-anchor
stiffness maps. It excludes rigid-load rotation, vehicle and load
transients, cable elasticity and damping, slack--taut transitions, and
environmental contact. The derived stiffness could therefore be exact
for the reduced model yet fail to capture the dominant response of the
richer dynamic system
(Section~\ref{subsec:scope_validation}).

\rev{A dynamic rigid-body validation campaign is defined as a \emph{stress test},}
rather than as a removal of the theoretical limitations. They
reintroduce rigid-payload motion, nonlinear closed-loop aerial-vehicle
dynamics, elastic-damped tendons, and environmental contact. By
comparing analytical and empirically identified stiffnesses, testing
aerial-anchor repositioning, and exploiting task-aligned compliance,
the campaign is intended to assess when and to what extent the reduced theory captures a
mechanically significant structure that persists beyond its defining
assumptions
(Section~\ref{sec:dynamic_simulation_validation}).
\section{System Model}
\label{sec:system_model}

\subsection{Affine and Coordinate Conventions}
\label{subsec:affine_coordinate_convention}

Let $\Eaff$ be an oriented three-dimensional Euclidean affine space,
and let $\Vtrans$ be its translation vector space, endowed with the
inner product $\inner{\cdot}{\cdot}$ and induced norm $\norm{\cdot}$.
Physical positions are points of $\Eaff$, whereas displacements,
velocities, virtual displacements, and cable directions are vectors in
$\Vtrans$. 
The Euclidean inner product induces the musical
isomorphisms $\flat:\Vtrans\to\Vdual$ and
$\sharp:\Vdual\to\Vtrans$, which identify vectors and covectors through
$v\flatop(w)=\inner{v}{w}$ and
$(\alpha\sharpop)\flatop=\alpha$.
For $P_1,P_2\in\Eaff$, the displacement from $P_1$ to $P_2$
is denoted by $\disp{P_1}{P_2}\in\Vtrans$. The affine action of
$v\in\Vtrans$ on $P\in\Eaff$ is denoted by $P+v\in\Eaff$.

Since $\Eaff$ is affine, $T_P\Eaff$ is canonically identified with
$\Vtrans$ for every $P\in\Eaff$. Hence, if $s\mapsto P(s)$ is a smooth
curve with $P(0)=P$, its tangent vector
$\delta p:=\left.\mathrm dP(s)/\mathrm ds\right|_{s=0}$ is regarded as
an element of $\Vtrans$.

Let
$\frameW=\{\originW,x_{\mathsf w},y_{\mathsf w},z_{\mathsf w}\}$
be a right-handed orthonormal inertial frame. The coordinate array of
$v\in\Vtrans$ and the position-coordinate array of $P\in\Eaff$ are
defined by
\begin{equation}
\begin{aligned}
    \coord{\fw}{v}
    &:=
    \begin{bmatrix}
        \inner{x_{\mathsf w}}{v} &
        \inner{y_{\mathsf w}}{v} &
        \inner{z_{\mathsf w}}{v}
    \end{bmatrix}\topT,
    \\
    \pointcoord{\fw}{P}
    &:=
    \coord{\fw}{\disp{\originW}{P}},
    \qquad
    \coord{\fw}{\disp{P_1}{P_2}}
    =
    \pointcoord{\fw}{P_2}-\pointcoord{\fw}{P_1}.
\end{aligned}
\label{eq:coordinate_convention}
\end{equation}
Intrinsic points, vectors, covectors, and maps are unbolded. Coordinate
arrays and matrices are bold and carry their expressing-frame
superscript. For a smooth map $\phi:\mathcal M\to\mathcal N$, its
differential at $x\in\mathcal M$ is denoted by
$d\phi_x:T_x\mathcal M\to T_{\phi(x)}\mathcal N$.

\begin{remark}[Intrinsic and coordinate descriptions]
\label{rem:intrinsic_coordinate_descriptions}
The intrinsic and coordinate descriptions serve distinct and
complementary purposes. The intrinsic formulation preserves the
physical types of points, displacements, forces, and stiffness maps,
making frame changes, dual operations, and pullbacks unambiguous.
Coordinate representations provide the arrays and matrices required
for computation and control. Maintaining both levels prevents
formally plausible but physically inconsistent operations, such as
re-expressing a displacement in another frame while leaving its
stiffness matrix unchanged, and facilitates extensions to generalized
coordinates and rigid-body models.
\end{remark}
\subsection{Geometry and Cable Forces}
\label{subsec:geometry_cable_forces}

Consider a point load of mass $\loadMass>0$ at $\loadPoint\in\Eaff$,
connected to $n\geq1$ aerial vehicles through taut, straight,
inextensible cables. Cable $i$ has fixed length
$\cableLength{i}>0$, actual aerial-anchor point
$\actualAnchorPoint{i}\in\Eaff$, and commanded aerial-anchor point
$\desiredAnchorPoint{i}\in\Eaff$. Their world-frame position arrays
are denoted by $\loadPos$, $\actualAnchorPos{i}$, and
$\desiredAnchorPos{i}$, respectively.
Figure~\ref{fig:system_overview} summarizes the system notation and the passive-stiffness-shaping mechanism investigated in this work.

The taut inextensible-cable constraint manifold is
\begin{equation}
    \CableManifold{i}
    :=
    \left\{
    (P,A)\in\Eaff\times\Eaff:
    \norm{\disp{P}{A}}=\cableLength{i}
    \right\}.
    \label{eq:cable_constraint_manifold}
\end{equation}
Let $\cableDir{i}\in\SphTwo\subset\Vtrans$ be the unit direction from
the load to vehicle $i$. Intrinsically and in world-frame coordinates,
\begin{equation}
\begin{aligned}
    \disp{\loadPoint}{\actualAnchorPoint{i}}
    &=
    \cableLength{i}\cableDir{i},
    \\
    \actualAnchorPos{i}
    &=
    \loadPos+\cableLength{i}\cableDirCoord{i},
    \qquad
    \norm{\cableDirCoord{i}}=1.
\end{aligned}
\label{eq:cable_geometry}
\end{equation}
The parametrization
$(P,u)\mapsto(P,P+\cableLength{i}u)$ identifies
$\CableManifold{i}$ with $\Eaff\times\SphTwo$.

The tangent space at $\cableDir{i}$ is
$T_{\cableDir{i}}\SphTwo
=\{v\in\Vtrans:\inner{\cableDir{i}}{v}=0\}$.
Let $\PerpProjector{i}$ denote the orthogonal projector onto this
space:
\[
    \PerpProjector{i}v
    :=
    v-\inner{\cableDir{i}}{v}\cableDir{i},
    \qquad
    \PerpProjectorMat{i}
    =
    \eye{3}
    -
    \cableDirCoord{i}\big(\cableDirCoord{i}\big)\topT.
\]
For an endpoint variation, differentiation of the normalized
endpoint-difference map gives
\begin{equation}
\begin{aligned}
    \delta u_i
    &=
    \frac{1}{\cableLength{i}}
    \PerpProjector{i}
    (\delta p_{a,i}-\delta p),
    \\
    \delta\cableDirCoord{i}
    &=
    \frac{1}{\cableLength{i}}
    \PerpProjectorMat{i}
    \left(
    \delta\actualAnchorPos{i}-\delta\loadPos
    \right).
\end{aligned}
\label{eq:cable_direction_differential}
\end{equation}
For a variation tangent to $\CableManifold{i}$, the relative endpoint
variation is already transverse, and therefore
$\delta p_{a,i}-\delta p=\cableLength{i}\delta u_i$.

Let $\cableTension{i}>0$ be the cable tension. The force applied by
cable $i$ on the load is intrinsically the covector
$\cableForceCovector{i}\in\Vdual$. Using the Euclidean identification
between $\Vtrans$ and $\Vdual$, define its dual vector
$\cableForceVector{i}\in\Vtrans$. These quantities satisfy
\begin{equation}
    \cableForceCovector{i}
    =
    \cableTension{i}\big(\cableDir{i}\big)\flatop,
    \qquad
    \cableForceVector{i}
    =
    \cableTension{i}\cableDir{i},
    \qquad
    \cableForceCoord{i}
    =
    \cableTension{i}\cableDirCoord{i}.
    \label{eq:cable_force}
\end{equation}
Hereafter, force relations are written using Euclidean-dual force
vectors or their coordinate arrays unless the covector character is
relevant.

\begin{remark}[Virtual-work interpretation of cable tension]
\label{rem:virtual_work_cable_tension}
Let
$\CableLengthMap{i}(P_1,P_2):=\norm{\disp{P_1}{P_2}}$.
At a configuration
$(\loadPoint,\actualAnchorPoint{i})\in\CableManifold{i}$,
its differential acts on arbitrary ambient endpoint variations as
\begin{equation}
    d\CableLengthMap{i}(\delta p,\delta p_{a,i})
    =
    \inner{\cableDir{i}}{\delta p_{a,i}-\delta p}
    =\inner{\cableDir{i}}{\delta p_{a,i}}-\inner{\cableDir{i}}{\delta p}
    .
\label{eq:cable_length_differential}
\end{equation}
Thus,
$d\CableLengthMap{i}
=
(-(\cableDir{i})\flatop,(\cableDir{i})\flatop)$
under the product-space identification
$T^*(\Eaff\times\Eaff)\simeq\Vdual\times\Vdual$.
The ideal internal cable-force covector is therefore
\[
    -\cableTension{i}d\CableLengthMap{i}
    =
    \left(
    \cableTension{i}(\cableDir{i})\flatop,
    -\cableTension{i}(\cableDir{i})\flatop
    \right),
\]
whose components are the equal-and-opposite forces applied to the load
and vehicle endpoints. For every tangent variation of
$\CableManifold{i}$, one has
$d\CableLengthMap{i}(\delta p,\delta p_{a,i})=0$; hence the ideal cable
force performs zero net virtual work on admissible variations of the
two-endpoint system.
\end{remark}

The transverse geometric-stiffness coefficient of cable $i$ is
\begin{equation}
    \geometricStiffnessCoeff{i}
    :=
    \cableLength{i}^{-1}\cableTension{i},
    \label{eq:geometric_stiffness_coefficient}
\end{equation}
with units of $\mathrm{N/m}$. It quantifies the restoring effect of the
equilibrium tension against transverse cable-direction variations.

\subsection{Compliant Aerial-Anchor Model}
\label{subsec:compliant_anchor_model}

The closed-loop translational behavior of aerial vehicle $i$ is
modeled quasi-statically by a linear stiffness map
$\AnchorStiffness{i}:\Vtrans\to\Vdual$, whose associated bilinear form
$(v,w)\mapsto(\AnchorStiffness{i}v)(w)$ is symmetric and positive
definite. Its inverse
$\AnchorCompliance{i}:=\AnchorStiffness{i}^{-1}:\Vdual\to\Vtrans$
is the corresponding linear compliance map.
Their world-frame matrices satisfy
\begin{equation}
    \AnchorStiffnessMat{i}\in\Spp(3),
    \qquad
    \AnchorComplianceMat{i}
    =
    \big(\AnchorStiffnessMat{i}\big)^{-1}.
    \label{eq:vehicle_stiffness}
\end{equation}
Symmetry and positive definiteness are understood with respect to the
dual pairing, i.e.,
$(\AnchorStiffness{i}v_1)(v_2)
=(\AnchorStiffness{i}v_2)(v_1)$ and
$(\AnchorStiffness{i}v)(v)>0$ for every nonzero
$v\in\Vtrans$.

At equilibrium, the vehicle-controller restoring force balances the
cable force. The intrinsic and coordinate relations are
\begin{equation}
\begin{aligned}
    \AnchorStiffness{i}
    \left(
    \disp{\actualAnchorPoint{i}}{\desiredAnchorPoint{i}}
    \right)
    &=
    \cableForceCovector{i},
    \\
    \AnchorStiffnessMat{i}
    \left(
    \desiredAnchorPos{i}-\actualAnchorPos{i}
    \right)
    &=
    \cableForceCoord{i}.
\end{aligned}
\label{eq:vehicle_equil}
\end{equation}
Equivalently,
\begin{equation}
\begin{aligned}
    \disp{\actualAnchorPoint{i}}{\desiredAnchorPoint{i}}
    &=
    \AnchorCompliance{i}\big(\cableForceCovector{i}\big),
    \\
    \actualAnchorPos{i}
    &=
    \desiredAnchorPos{i}
    -
    \AnchorComplianceMat{i}\cableForceCoord{i}.
\end{aligned}
\label{eq:vehicle_yield}
\end{equation}

\subsection{Gravity-Loaded Equilibrium}
\label{subsec:gravity_loaded_equilibrium}

Let $\gravityVector,\externalVector\in\Vtrans$ be the Euclidean-dual
vectors of gravity and an additional constant external-force covector.
The positive direction $z_{\mathsf w}$ is chosen opposite to gravity,
so that
\[
    \gravityVector=-\loadMass\gravityScalar z_{\mathsf w},
    \qquad
    \gravityCoord
    =
    -\loadMass\gravityScalar
    \begin{bmatrix}0&0&1\end{bmatrix}\topT.
\]
The intrinsic and coordinate load-equilibrium equations are
\begin{equation}
\begin{aligned}
    \sum_{i=1}^{n}\cableForceVector{i}
    +\gravityVector+\externalVector
    &=
    0_{\Vtrans},
    \\
    \sum_{i=1}^{n}\cableForceCoord{i}
    +\gravityCoord+\externalCoord
    &=
    \zerovec{3}.
\end{aligned}
\label{eq:gravity_equilibrium}
\end{equation}

Let
$\setpointTuple
:=
(\desiredAnchorPoint{1},\ldots,\desiredAnchorPoint{n})
\in\Eaff^n$
collect the commanded aerial-anchor points, with stacked world-frame
coordinate array
\begin{equation}
    \setpointCoord
    :=
    \col{
        \desiredAnchorPos{1},
        \ldots,
        \desiredAnchorPos{n}
    }
    \in\R^{3n}.
    \label{eq:q_definition}
\end{equation}

For a prescribed $\setpointTuple$, let
$\EquilibriumSet(\setpointTuple)$ denote the set of tuples
\[
    \left(
    \loadPoint,
    \{\actualAnchorPoint{i},\cableTension{i}\}_{i=1}^{n}
    \right)
\]
satisfying, for every $i=1,\ldots,n$,
\begin{equation}
\begin{aligned}
    \disp{\loadPoint}{\actualAnchorPoint{i}}
    &=
    \cableLength{i}\cableDir{i},
    &\cableDir{i}&\in\SphTwo,
    \\
    \disp{\actualAnchorPoint{i}}{\desiredAnchorPoint{i}}
    &=
    \AnchorCompliance{i}
    \left(
        \cableTension{i}
        (\cableDir{i})\flatop
    \right),
    &\cableTension{i}&>0,
    \\
    \sum_{j=1}^{n}
    \cableTension{j}\cableDir{j}
    +\gravityVector+\externalVector
    &=
    0_{\Vtrans}. &&
\end{aligned}
\label{eq:equilibrium_set}
\end{equation}
All coordinate equations defining the set $\EquilibriumSet(\setpointTuple)$ are given by
\eqref{eq:cable_geometry}, \eqref{eq:vehicle_yield}, and
\eqref{eq:gravity_equilibrium}.

Depending on $\setpointTuple$, the set
$\EquilibriumSet(\setpointTuple)$ may be empty, contain one or several
isolated equilibria, or contain an equilibrium continuum. The
stiffness analysis below is local and is evaluated at a selected
element of this set or along a selected smooth equilibrium branch.

\begin{remark}[Equilibrium before stiffness]
Passive stiffness is a local force--displacement relation at a
selected gravity-loaded equilibrium. Gravity is not an additive
stiffness term, but it determines the operating-point cable directions
and tensions entering the stiffness expression.
\end{remark}
\section{Passive Cartesian Stiffness at a Gravity-Loaded Equilibrium}
\label{sec:passive_cartesian_stiffness}

Select a taut gravity-loaded equilibrium
\[
    \left(
    \loadPointEq,
    \{\actualAnchorPointEq{i},T_{i,\mathrm e}\}_{i=1}^{n}
    \right)
    \in
    \EquilibriumSet(\setpointTuple).
\]
All quantities in this section are evaluated at this equilibrium, and
the equilibrium subscript is omitted for readability. The vehicle
setpoints are held fixed while the load and actual aerial anchors
undergo infinitesimal quasi-static variations.

\subsection{Geometric Compliance of One Leg}
\label{subsec:single_leg_geometric_compliance}

Hereafter, \emph{leg $i$} denotes the serial interconnection formed by
cable $i$ and its compliant aerial anchor. The Euclidean-dual
cable-force vector of leg $i$ satisfies
$\cableForceVector{i}
=\cableTension{i}\cableDir{i}$. 
Its variation is 
\begin{equation}
    \delta\cableForceVector{i}
    =
    \delta\cableTension{i}\cableDir{i}
    +
    \cableTension{i}\delta u_i.
    \label{eq:cable_force_variation}
\end{equation}
Since $\delta u_i\in T_{\cableDir{i}}\SphTwo$, projection onto the
transverse space gives
\[
    \PerpProjector{i}\delta\cableForceVector{i}
    =
    \cableTension{i}\delta u_i.
\]
Combining this relation with
$\delta p_{a,i}-\delta p=\cableLength{i}\delta u_i$ yields
\begin{equation}
    \delta p_{a,i}-\delta p
    =
    \GeometricCompliance{i}
    \big(\delta\cableForceCovector{i}\big),
    \qquad
    \GeometricCompliance{i}
    :=
    \frac{\cableLength{i}}{\cableTension{i}}
    \PerpProjector{i}\circ\sharp.
    \label{eq:intrinsic_geometric_compliance}
\end{equation}
Thus,
$\GeometricCompliance{i}:\Vdual\to\Vtrans$ is the transverse
geometric-compliance map of cable $i$. Its world-frame coordinate
matrix is
\begin{equation}
    \GeometricComplianceMat{i}
    =
    \frac{\cableLength{i}}{\cableTension{i}}
    \PerpProjectorMat{i}
    =
    \geometricStiffnessCoeff{i}^{-1}
    \PerpProjectorMat{i}.
    \label{eq:geometric_compliance_matrix}
\end{equation}
The map is positive semidefinite and vanishes in the cable direction, consistently with cable inextensibility.

\subsection{Series Combination with Aerial-Anchor Compliance}
\label{subsec:series_compliance}

At fixed commanded point, differentiation of
\eqref{eq:vehicle_yield} gives
\begin{equation}
    \delta p_{a,i}
    =
    -\AnchorCompliance{i}
    \big(\delta\cableForceCovector{i}\big),
    \qquad
    \delta\actualAnchorPos{i}
    =
    -\AnchorComplianceMat{i}
    \delta\cableForceCoord{i}.
    \label{eq:anchor_variation_fixed_setpoint}
\end{equation}
Substitution into \eqref{eq:intrinsic_geometric_compliance} yields
\[
    -\delta p
    =
    \left(
    \AnchorCompliance{i}
    +
    \GeometricCompliance{i}
    \right)
    \big(\delta\cableForceCovector{i}\big).
\]

\begin{theorem}[Single-leg passive stiffness]
\label{thm:single_leg_passive_stiffness}
At a taut gravity-loaded equilibrium, the passive Cartesian stiffness
contribution of leg $i$ is the symmetric positive-definite linear map
\begin{equation}
    \LegStiffness{i}
    :=
    \left(
    \AnchorCompliance{i}
    +
    \GeometricCompliance{i}
    \right)^{-1}
    :
    \Vtrans\to\Vdual,
    \label{eq:intrinsic_single_leg_stiffness}
\end{equation}
defined by the restoring-force relation
$\delta\cableForceCovector{i}
=-\LegStiffness{i}(\delta p)$.
Its world-frame coordinate matrix is
\begin{equation}
    \LegStiffnessMat{i}
    =
    \left[
    \AnchorComplianceMat{i}
    +
    \frac{\cableLength{i}}{\cableTension{i}}
    \left(
    \eye{3}
    -
    \cableDirCoord{i}
    \big(\cableDirCoord{i}\big)\topT
    \right)
    \right]^{-1}.
    \label{eq:single_leg_stiffness}
\end{equation}
\end{theorem}

\begin{proof}
The geometric and aerial-anchor compliances map the same cable-force
variation to consecutive displacement contributions and therefore
combine in series. The map $\AnchorCompliance{i}$ is symmetric positive
definite, whereas $\GeometricCompliance{i}$ is symmetric positive
semidefinite. Their sum is consequently symmetric positive definite
and invertible, and its inverse $\LegStiffness{i}$ is also symmetric
positive definite. The restoring-force relation then gives
\eqref{eq:intrinsic_single_leg_stiffness} and its coordinate
representation \eqref{eq:single_leg_stiffness}.
\end{proof}

\subsection{Total Passive Stiffness}
\label{subsec:total_passive_stiffness}

The total cable-force variation is the sum of the individual
leg-force variations. Hence,
\[
    \delta\varphi_c
    =
    \sum_{i=1}^{n}\delta\cableForceCovector{i}
    =
    -\sum_{i=1}^{n}\LegStiffness{i}(\delta p).
\]

\begin{corollary}[Total passive Cartesian stiffness]
\label{cor:total_passive_stiffness}
The passive Cartesian stiffness of the suspended load is
\begin{equation}
    \TotalStiffness
    =
    \sum_{i=1}^{n}\LegStiffness{i},
    \qquad
    {
    \TotalStiffnessMat
    =
    \sum_{i=1}^{n}\LegStiffnessMat{i}.
    }
    \label{eq:total_stiffness}
\end{equation}
Thus, the geometric and aerial-anchor compliances combine in series
within each leg, whereas the leg stiffnesses act in parallel on the
load. The associated total restoring-force relation is
$\delta\varphi_c=-\TotalStiffness(\delta p)$.
\end{corollary}

\begin{remark}[Role of gravity]
The gravitational and constant external forces have zero variation in
the inertial frame and therefore do not contribute additive stiffness
terms. Their effect is contained in the selected operating point:
they determine the equilibrium cable directions and tensions entering
\eqref{eq:single_leg_stiffness}.
\end{remark}
\section{Virtual Elastic-Cable Equivalence}
\label{sec:virtual_elastic_cable_equivalence}

The preceding stiffness result allows anisotropic aerial-anchor
compliance. A more explicit representation is obtained when each
aerial-anchor stiffness is isotropic.

\subsection{Isotropic Aerial-Anchor Stiffness}
\label{subsec:isotropic_anchor_stiffness}

Assume that, for some $\isotropicStiffness{i}>0$,
\begin{equation}
    \AnchorStiffness{i}
    =
    \isotropicStiffness{i}\flat,
    \qquad
    \AnchorCompliance{i}
    =
    \isotropicStiffness{i}^{-1}\sharp,
    \qquad
    \AnchorStiffnessMat{i}
    =
    \isotropicStiffness{i}\eye{3}.
    \label{eq:isotropic_vehicle}
\end{equation}
Let
\[
    \virtualCableVector{i}
    :=
    \disp{\loadPoint}{\desiredAnchorPoint{i}}
    \in\Vtrans,
    \qquad
    \virtualDistance{i}:=\norm{\virtualCableVector{i}}.
\]
By \eqref{eq:vehicle_equil} and \eqref{eq:cable_geometry},
\begin{equation}
    \virtualCableVector{i}
    =
    \left(
    \cableLength{i}
    +
    \isotropicStiffness{i}^{-1}\ \cableTension{i}
    \right)
    \cableDir{i}.
    \label{eq:v_parallel_u}
\end{equation}
Hence, the commanded aerial-anchor point lies along the physical cable
direction, and
\begin{equation}
    \cableTension{i}
    =
    \isotropicStiffness{i}
    \left(
    \virtualDistance{i}-\cableLength{i}
    \right),
    \qquad
    \virtualDistance{i}>\cableLength{i}.
    \label{eq:tension_virtual}
\end{equation}
The inequality is equivalent to positive cable tension.

\subsection{Virtual Elastic-Cable Force Law}
\label{subsec:virtual_elastic_force}

For $\virtualDistance{i}>\cableLength{i}$, $\cableDir{i}=\virtualCableVector{i}/\virtualDistance{i}$. Hence, the force vector and its world-frame coordinates are
\begin{equation}
\begin{aligned}
    \cableForceVector{i}
    &=
    \isotropicStiffness{i}
    \left(
    1-\virtualDistance{i}^{-1}\ \cableLength{i}
    \right)\virtualCableVector{i},
    \\
    \cableForceCoord{i}
    &=
    \isotropicStiffness{i}
    \left(
    1-
    \frac{\cableLength{i}}
    {\norm{\desiredAnchorPos{i}-\loadPos}}
    \right)
    \left(
    \desiredAnchorPos{i}-\loadPos
    \right).
\end{aligned}
    \label{eq:virtual_elastic_force}
\end{equation}

\begin{proposition}[Virtual elastic-cable equivalence]
\label{prop:virtual_elastic_cable_equivalence}
Under \eqref{eq:isotropic_vehicle}, leg $i$ is quasi-statically
equivalent, on its taut domain, to a virtual unilateral elastic cable
connecting $\loadPoint$ directly to $\desiredAnchorPoint{i}$, with
stiffness $\isotropicStiffness{i}$ and rest length
$\cableLength{i}$.
\end{proposition}

\begin{proof}
The physical cable displacement and the compliant-anchor displacement
are both parallel to $\cableDir{i}$, giving
\eqref{eq:v_parallel_u}. Taking its norm yields
\eqref{eq:tension_virtual}; substitution into
$\cableForceVector{i}
=\cableTension{i}\cableDir{i}$
gives \eqref{eq:virtual_elastic_force}.
\end{proof}

Figure~\ref{fig:physical_virtual_leg} summarizes the physical
series-compliance mechanism and its exact virtual elastic-cable
representation under isotropic aerial-anchor stiffness.

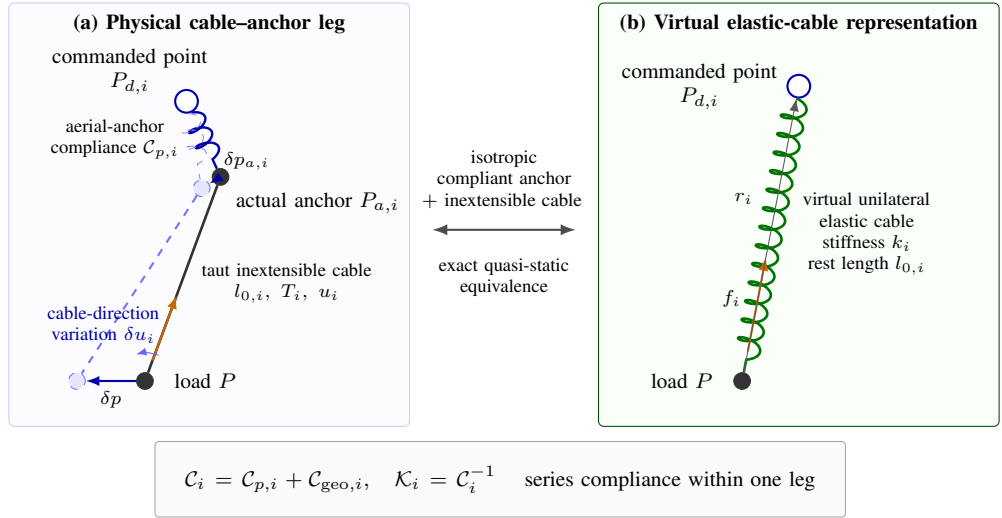
\begin{figure*}[t]
    \centering
    \begin{tikzpicture}[
        x=1cm,
        y=1cm,
        point/.style={
            circle,
            draw=black!80,
            fill=black!80,
            minimum size=2.2mm,
            inner sep=0pt
        },
        commandedpoint/.style={
            circle,
            draw=blue!70!black,
            fill=white,
            line width=0.8pt,
            minimum size=3mm,
            inner sep=0pt
        },
        displacedpoint/.style={
            circle,
            draw=blue!55,
            fill=blue!10,
            dashed,
            minimum size=2.2mm,
            inner sep=0pt
        },
        physicalcable/.style={
            draw=black!80,
            line width=1pt
        },
        perturbedcable/.style={
            draw=blue!55,
            dashed,
            line width=0.8pt
        },
        spring/.style={
            draw=blue!65!black,
            line width=0.9pt,
            decorate,
            decoration={
                coil,
                aspect=0.45,
                segment length=2.2mm,
                amplitude=1.4mm
            }
        },
        virtualspring/.style={
            draw=green!45!black,
            line width=1pt,
            decorate,
            decoration={
                coil,
                aspect=0.45,
                segment length=2.5mm,
                amplitude=1.6mm
            }
        },
        motionarrow/.style={
            -{Latex[length=1.8mm]},
            draw=blue!65!black,
            thick
        },
        forcearrow/.style={
            -{Latex[length=1.8mm]},
            draw=orange!80!black,
            thick
        },
        equivalencearrow/.style={
            {Latex[length=2mm]}-{Latex[length=2mm]},
            draw=black!70,
            thick
        },
        paneltitle/.style={
            font=\footnotesize\bfseries,
            align=center
        },
        figurelabel/.style={
            font=\footnotesize,
            align=center
        },
        smalllabel/.style={
            font=\scriptsize,
            align=center
        }
    ]

    \begin{scope}[xshift=0cm]
        \node[paneltitle] at (2.3,5.28)
        {(a) Physical cable--anchor leg};
        \node[commandedpoint] (Pda) at (2.0,4.25) {};
        \node[
            figurelabel,
            above right =4mm of Pda,
            anchor=east
        ]
            {commanded point\\
            $\desiredAnchorPoint{i}$};
        \node[point] (Paa) at (2.45,3.25) {};
        \node[figurelabel, below right=0mm of Paa]
            {actual anchor $\actualAnchorPoint{i}$};
        \draw[spring] (Pda) -- (Paa);

        \node[smalllabel, left=2mm of $(Pda)!0.5!(Paa)$]
            {aerial-anchor\\compliance
            $\AnchorCompliance{i}$};
        \node[point] (Pa) at (1.45,0.55) {};
        \node[figurelabel, right=1.5mm of Pa]
            {load $\loadPoint$};
        \draw[physicalcable] (Paa) -- (Pa);

        \node[smalllabel, right=1.5mm of $(Paa)!0.52!(Pa)$]
            {taut inextensible cable\\
            $\cableLength{i},\ \cableTension{i},\ \cableDir{i}$};
        \draw[forcearrow]
            ($(Pa)!0.10!(Paa)$)
            --
            ($(Pa)!0.42!(Paa)$);
        \node[displacedpoint] (Pap) at (0.55,0.55) {};
        \node[displacedpoint] (Paap) at (2.20,3.10) {};
        \draw[
            blue!45,
            dashed,
            decorate,
            decoration={
                coil,
                aspect=0.45,
                segment length=2.2mm,
                amplitude=1.4mm
            }
        ] (Pda) -- (Paap);

        \draw[perturbedcable] (Paap) -- (Pap);
        \draw[motionarrow]
            (Pa) --
            node[smalllabel, below] {$\delta p$}
            (Pap);

        \draw[motionarrow]
            (Paa) --
            node[smalllabel, above right=0.5mm] {$\delta p_{a,i}$}
            (Paap);
        \draw[
            -{Latex[length=1.4mm]},
            draw=blue!60,
            thin
        ]
        ($(Pa)+(0.18,0.34)$)
        arc[start angle=70,end angle=105,radius=5mm];

        \node[smalllabel, text=blue!65!black]
            at (0.9,1.30)
            {cable-direction\\variation $\delta u_i$};

    \end{scope}

    \draw[equivalencearrow]
        (5.25,2.55) -- (7.15,2.55);

    \node[smalllabel, text width=22mm]
        at (6.20,3.22)
        {isotropic compliant anchor\\
        $+$ inextensible cable};

    \node[smalllabel]
        at (6.20,1.92)
        {exact quasi-static\\equivalence};

    \begin{scope}[xshift=7.8cm]
        \node[paneltitle] at (2.3,5.28)
            {(b) Virtual elastic-cable representation};
        \node[commandedpoint] (Pdb) at (2.3,4.45) {};
        \node[
            figurelabel,
            left=0mm of Pdb,
            anchor=east
        ]
            {commanded point\\
            $\desiredAnchorPoint{i}$};
        \node[point] (Pb) at (1.55,0.55) {};
        \node[figurelabel, left=1.5mm of Pb]
            {load $\loadPoint$};
        \draw[virtualspring] (Pdb) -- (Pb);

        \node[
            smalllabel,
            right=-3mm of $(Pdb)!0.5!(Pb)$,
            text width=29mm
        ]
            {virtual unilateral elastic cable\\
            stiffness $\isotropicStiffness{i}$\\
            rest length $\cableLength{i}$};
        \draw[forcearrow]
            ($(Pb)!0.10!(Pdb)$)
            --
            ($(Pb)!0.42!(Pdb)$);

        \node[smalllabel, left=1mm of $(Pb)!0.28!(Pdb)$]
            {$\cableForceVector{i}$};
        \draw[
            -{Latex[length=1.7mm]},
            draw=black!65,
            thin
        ]
            (Pb) -- (Pdb);

        \node[smalllabel, left=2mm of $(Pb)!0.62!(Pdb)$]
            {$\virtualCableVector{i}$};

    \end{scope}

    \node[
        draw=black!35,
        fill=black!2,
        rounded corners=1.5pt,
        align=center,
        font=\footnotesize,
        text width=89mm,
        minimum height=10mm
    ] at (6.15,-0.75)
        {$\displaystyle
        \mathcal C_i
        =
        \AnchorCompliance{i}
        +
        \GeometricCompliance{i},
        \quad
        \LegStiffness{i}
        =
        \mathcal C_i^{-1}$ \quad
        series compliance within one leg};
        
        \begin{pgfonlayer}{background}
            \path[
                draw=blue!20,
                fill=blue!1,
                rounded corners=2pt
            ]
                (-0.35,-0.05)
                rectangle
                (4.95,5.55);
            \path[
                draw=green!30!black,
                fill=green!2,
                rounded corners=2pt
            ]
                (7.45,-0.05)
                rectangle
                (12.75,5.55);
        
        \end{pgfonlayer}    

    \end{tikzpicture}

    \caption{Physical and virtual representations of one cable--anchor leg.
    The physical leg combines aerial-anchor compliance and the transverse
    geometric compliance of the taut inextensible cable in series. Under
    isotropic aerial-anchor stiffness, the leg is quasi-statically equivalent
    to a virtual unilateral elastic cable connecting the load directly to the
    commanded anchor point.}

    \label{fig:physical_virtual_leg}
\end{figure*}

\subsection{Gravity-Loaded Equilibrium Residual}
\label{subsec:virtual_equilibrium_residual}

Define the taut virtual-cable domain
\[
    \mathcal D
    :=
    \left\{
    (\loadPoint,\setpointTuple)\in\Eaff\times\Eaff^n:
    \norm{\disp{\loadPoint}{\desiredAnchorPoint{i}}}
    >
    \cableLength{i},\
    i=1,\ldots,n
    \right\}.
\]
The isotropic virtual-cable representation defines the intrinsic
equilibrium residual
\begin{equation}
    \EquilibriumResidual:
    \mathcal D\to\Vtrans,
    \qquad
    \EquilibriumResidual(\loadPoint,\setpointTuple)
    :=
    \sum_{i=1}^{n}\cableForceVector{i}
    +
    \gravityVector+\externalVector.
    \label{eq:intrinsic_equilibrium_residual}
\end{equation}
Its world-frame coordinate representation is
\begin{equation}
\begin{aligned}
    \EquilibriumResidualCoord(\loadPos,\setpointCoord)
    &:=
    \sum_{i=1}^{n}
    \isotropicStiffness{i}
    \left(
    1-
    \frac{\cableLength{i}}
    {\norm{\desiredAnchorPos{i}-\loadPos}}
    \right)
    \left(
    \desiredAnchorPos{i}-\loadPos
    \right)
    \\
    &\quad+
    \gravityCoord+\externalCoord.
\end{aligned}
    \label{eq:h_def}
\end{equation}
A taut gravity-loaded equilibrium satisfies
\begin{equation}
    \EquilibriumResidual(\loadPointEq,\setpointTuple)
    =
    0_{\Vtrans},
    \quad\text{equivalently}\quad
    \EquilibriumResidualCoord(\loadPosEq,\setpointCoord)
    =
    \zerovec{3}.
    \label{eq:h_equilibrium}
\end{equation}
When several solutions exist, all subsequent sensitivity and
stiffness-shaping constructions are applied locally to a selected
smooth equilibrium branch.
\section{Stiffness Decomposition at Gravity-Loaded Equilibria}
\label{sec:stiffness_decomposition}

Throughout this section, the aerial-anchor stiffness is isotropic as
in \eqref{eq:isotropic_vehicle}. Substitution into
\eqref{eq:single_leg_stiffness} gives
\begin{equation}
    \LegStiffnessMat{i}
    =
    \left[
    \isotropicStiffness{i}^{-1}\eye{3}
    +
    \cableTension{i}^{-1}\cableLength{i}
    \PerpProjectorMat{i}
    \right]^{-1}.
    \label{eq:isotropic_leg_inverse}
\end{equation}
Let
\[
    \ParallelProjectorMat{i}
    :=
    \cableDirCoord{i}\big(\cableDirCoord{i}\big)\topT,
    \qquad
    \eye{3}
    =
    \ParallelProjectorMat{i}+\PerpProjectorMat{i}.
\]
Since the axial and transverse projectors are mutually orthogonal,
the matrix inside the inverse in \eqref{eq:isotropic_leg_inverse} becomes
\[
    \isotropicStiffness{i}^{-1}\ParallelProjectorMat{i}
    +
    \left(
    \isotropicStiffness{i}^{-1}
    +
    \cableTension{i}^{-1}\cableLength{i}
    \right)
    \PerpProjectorMat{i}.
\]
Its inverse is therefore obtained by inverting the two scalar
coefficients separately:
\begin{equation}
    \LegStiffnessMat{i}
    =
    \isotropicStiffness{i}\ParallelProjectorMat{i}
    +
    \transverseStiffness{i}\PerpProjectorMat{i},
    \qquad
    \transverseStiffness{i}
    :=
    \frac{
        \isotropicStiffness{i}\cableTension{i}
    }{
        \cableTension{i}
        +
        \isotropicStiffness{i}\cableLength{i}
    }.
    \label{eq:leg_decomp_axial_transverse}
\end{equation}
Using \eqref{eq:tension_virtual}, the transverse stiffness coefficient is also
\begin{equation}
    \transverseStiffness{i}
    =
    \isotropicStiffness{i}
    \left(
    1-
    \frac{\cableLength{i}}{\virtualDistance{i}}
    \right).
    \label{eq:gamma_distance}
\end{equation}
Thus, the Euclidean-identified stiffness operator
$\sharp\circ\LegStiffness{i}:\Vtrans\to\Vtrans$ has the simple
eigenvalue $\isotropicStiffness{i}$ along
$\operatorname{span}\{\cableDir{i}\}$ and the eigenvalue
$\transverseStiffness{i}$, with algebraic and geometric multiplicity
two, on $T_{\cableDir{i}}\SphTwo$.

Figure~\ref{fig:axial_transverse_stiffness} illustrates this
axial--transverse eigenstructure and its evolution with the
operating-point tension.

\begin{figure}[t]
    \centering
    \begin{tikzpicture}[
        x=1cm,
        y=1cm,
        axisarrow/.style={
            -{Latex[length=1.6mm]},
            draw=blue!70!black,
            line width=0.9pt
        },
        transverseaxis/.style={
            -{Latex[length=1.3mm]},
            draw=green!45!black,
            line width=0.7pt
        },
        preloadarrow/.style={
            -{Latex[length=1.6mm]},
            draw=black!65,
            thick
        },
        cableaxis/.style={
            draw=blue!70!black,
            line width=1pt
        },
        planefill/.style={
            draw=green!45!black,
            fill=green!10,
            line width=0.7pt
        },
        limitshape/.style={
            draw=blue!60!black,
            fill=blue!8,
            line width=0.7pt
        },
        centerpoint/.style={
            circle,
            draw=black!80,
            fill=black!80,
            minimum size=1.8mm,
            inner sep=0pt
        },
        mainlabel/.style={
            font=\scriptsize,
            align=center
        },
        smalllabel/.style={
            font=\scriptsize,
            align=center
        },
        equationbox/.style={
            draw=black!35,
            fill=black!2,
            rounded corners=1.5pt,
            align=center,
            font=\scriptsize,
            inner sep=2.5pt
        }
    ]

    \path[use as bounding box] (0.05,-2.02) rectangle (8.45,6.82);

    \node[
        mainlabel,
        font=\footnotesize\bfseries
    ] at (4.25,6.55)
        {(a) Single-leg eigenstructure};
    \draw[planefill]
        (2.05,4.35)
        ellipse[x radius=1.35cm,y radius=0.43cm];
    \draw[cableaxis]
        (2.05,2.95) -- (2.05,5.95);

    \draw[axisarrow]
        (2.05,4.35) -- (2.05,5.80);

    \node[centerpoint] at (2.05,4.35) {};
    \node[
        mainlabel,
        anchor=west,
        text=blue!70!black,
        text width=22mm
    ] at (1.90,5.47)
        {cable direction
        $\cableDir{i}$\\
        axial eigenvalue 
        $\isotropicStiffness{i}$\\
        multiplicity one};
    \draw[transverseaxis]
        (2.05,4.35) -- (3.18,4.35);

    \draw[transverseaxis]
        (2.05,4.35) -- (0.92,4.35);

    \draw[transverseaxis]
        (2.05,4.35) -- (2.63,4.72);
    \node[
        mainlabel,
        text=green!35!black,
        text width=30mm,
        anchor=north
    ] at (1.05,3.95)
        {transverse plane\\
        $T_{\cableDir{i}}\SphTwo$\\
        eigenvalue
        $\transverseStiffness{i}$\\
        multiplicity two};
    \node[
        equationbox,
        text width=35mm
    ] at (6.20,5.70)
        {$\displaystyle
        \sharp\circ\LegStiffness{i}
        =
        \isotropicStiffness{i}
        \mathcal P_{u_i}^{\parallel}        +
        \transverseStiffness{i}
        \mathcal P_{u_i}^{\perp}$\\[1mm]
        {\scriptsize
        Eq.~\eqref{eq:leg_decomp_axial_transverse}}};
    \node[
        equationbox,
        text width=35mm
    ] at (6.20,4.25)
        {$\displaystyle
        \transverseStiffness{i}
        =
        \frac{
            \isotropicStiffness{i}\cableTension{i}
        }{
            \cableTension{i}
            +
            \isotropicStiffness{i}\cableLength{i}
        }$\\[0.5mm]
        $\displaystyle
        0<
        \transverseStiffness{i}
        <
        \isotropicStiffness{i}$\\[0.8mm]
        {\scriptsize
        Eq.~\eqref{eq:leg_decomp_axial_transverse}}};

    \node[
        smalllabel,
        text width=36mm
    ] at (6.20,2.9)
        {Axial stiffness comes from the aerial anchor;
        transverse stiffness combines anchor and
        geometric compliance in series.};

    \node[
        mainlabel,
        font=\footnotesize\bfseries
    ] at (4.25,1.85)
        {(b) Evolution with operating-point tension};
    \draw[limitshape]
        (1.05,0.05) -- (1.05,1.08);

    \node[centerpoint] at (1.05,0.565) {};

    \node[
        smalllabel,
        text width=21mm
    ] at (1.05,-0.48)
        {barely taut\\
        $\cableTension{i}\to0^+$\\
        $\transverseStiffness{i}\to0$};

    \node[
        smalllabel,
        text width=21mm
    ] at (1.05,-1.18)
        {nearly rank-one\\
        contribution};
    \draw[preloadarrow]
        (1.92,0.565) -- (3.14,0.565);

    \node[
        smalllabel,
        text width=15mm
    ] at (2.53,0.97)
        {increasing\\tension};
    \draw[limitshape]
        (4.05,0.565)
        ellipse[x radius=0.48cm,y radius=0.78cm];

    \draw[cableaxis]
        (4.05,-0.34) -- (4.05,1.47);

    \node[centerpoint] at (4.05,0.565) {};

    \node[
        smalllabel,
        text width=22mm
    ] at (4.05,-0.68)
        {finite preload\\
        $0<\transverseStiffness{i}
        <\isotropicStiffness{i}$};

    \node[
        smalllabel,
        text width=22mm
    ] at (4.05,-1.18)
        {axisymmetric\\
        contribution};
    \draw[preloadarrow]
        (4.92,0.565) -- (6.14,0.565);

    \node[
        smalllabel,
        text width=15mm
    ] at (5.53,0.97)
        {increasing\\tension};
    \draw[limitshape]
        (7.08,0.565)
        circle[radius=0.69cm];

    \node[centerpoint] at (7.08,0.565) {};

    \node[
        smalllabel,
        text width=23mm
    ] at (7.08,-0.58)
        {high preload\\
        $\cableTension{i}\to\infty$\\
        $\transverseStiffness{i}
        \to\isotropicStiffness{i}$};

    \node[
        smalllabel,
        text width=22mm
    ] at (7.08,-1.27)
        {isotropic\\
        contribution};
    \node[
        smalllabel,
        text=black!65
    ] at (4.25,-1.78)
        {Limits: Section~\ref{subsec:stiffness_limit_cases}};

    \begin{pgfonlayer}{background}

        \path[
            draw=blue!20,
            fill=blue!1,
            rounded corners=2pt
        ]
            (0.05,2.25)
            rectangle
            (8.45,6.82);

        \path[
            draw=black!20,
            fill=black!1,
            rounded corners=2pt
        ]
            (0.05,-2.02)
            rectangle
            (8.45,2.05);

    \end{pgfonlayer}

    \end{tikzpicture}

    \caption{Axial--transverse eigenstructure of an isotropic
    cable--anchor leg. The Euclidean-identified stiffness operator has
    the simple eigenvalue $\isotropicStiffness{i}$ along the cable
    direction and the eigenvalue $\transverseStiffness{i}$, with
    algebraic and geometric multiplicity two, on the transverse plane.
    Increasing the operating-point tension changes the leg contribution
    from nearly rank one toward isotropic.}

    \label{fig:axial_transverse_stiffness}
\end{figure}
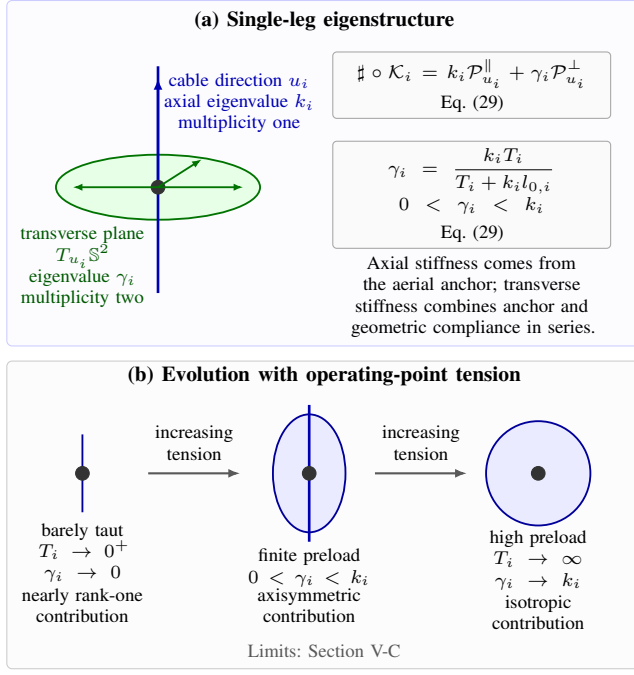

\subsection{Isotropic Baseline and Directional Shaping}
\label{subsec:isotropic_directional_decomposition}

Since
$\PerpProjectorMat{i}
=
\eye{3}
-
\cableDirCoord{i}(\cableDirCoord{i})\topT$,
\eqref{eq:leg_decomp_axial_transverse} is equivalently
\[
    \LegStiffnessMat{i}
    =
    \transverseStiffness{i}\eye{3}
    +
    \left(
    \isotropicStiffness{i}-\transverseStiffness{i}
    \right)
    \cableDirCoord{i}\big(\cableDirCoord{i}\big)\topT.
\]
Summing over all legs yields
\begin{equation}
    \TotalStiffnessMat
    =
    \left(
    \sum_{i=1}^{n}\transverseStiffness{i}
    \right)\eye{3}
    +
    \sum_{i=1}^{n}
    \left(
    \isotropicStiffness{i}-\transverseStiffness{i}
    \right)
    \cableDirCoord{i}\big(\cableDirCoord{i}\big)\topT.
    \label{eq:shape_decomp}
\end{equation}
The first term provides an isotropic stiffness baseline, whereas the
second comprises directional rank-one contributions aligned with the
cables. Both depend on the selected gravity-loaded equilibrium through
the cable directions and tensions.

\subsection{Stiffness and Compliance Ellipsoids}
\label{subsec:stiffness_compliance_ellipsoids}

The intrinsic stiffness unit-level set is
\[
    \mathcal E_K
    :=
    \left\{
    v\in\Vtrans:
    \big(\TotalStiffness v\big)(v)=1
    \right\}.
\]
In world-frame coordinates, the stiffness and reciprocal compliance
ellipsoids are represented by
\begin{equation}
\begin{aligned}
    \mathcal E_K^{\mathsf w}
    &:=
    \left\{
    \bm v^{\mathsf w}\in\Rthree:
    (\bm v^{\mathsf w})\topT
    \TotalStiffnessMat
    \bm v^{\mathsf w}
    =1
    \right\},
    \\
    \mathcal E_C^{\mathsf w}
    &:=
    \left\{
    \bm f^{\mathsf w}\in\Rthree:
    (\bm f^{\mathsf w})\topT
    \EndEffectorComplianceMat
    \bm f^{\mathsf w}
    =1
    \right\},
    \
    \EndEffectorComplianceMat
    :=
    {\TotalStiffnessMat}^{-1}.
\end{aligned}
\label{eq:stiffness_compliance_ellipsoids}
\end{equation}
The first is a displacement-space level set of the stiffness
quadratic form; the second is the corresponding force-space level set
of the compliance quadratic form.

\subsection{Limit Cases}
\label{subsec:stiffness_limit_cases}

Since $0<\transverseStiffness{i}<\isotropicStiffness{i}$, the barely taut limit $\cableTension{i}\to0^+$ gives
\[
    \cableTension{i}\to0^+
    \quad\Longrightarrow\quad
    \transverseStiffness{i}\to0,
    \qquad
    \LegStiffnessMat{i}
    \to
    \isotropicStiffness{i}
    \cableDirCoord{i}\big(\cableDirCoord{i}\big)\topT.
\]
The leg then supplies stiffness only along the cable direction. The
high-preload limit $\cableTension{i}\to\infty$ gives
\[
    \cableTension{i}\to\infty
    \quad\Longrightarrow\quad
    \transverseStiffness{i}\to\isotropicStiffness{i},
    \qquad
    \LegStiffnessMat{i}
    \to
    \isotropicStiffness{i}\eye{3}.
\]
The leg therefore approaches the isotropic stiffness of the aerial
anchor. In practice, cable strength, vehicle thrust, safety margins,
and energy consumption bound the attainable preload.

These limits describe the dependence of the leg stiffness on the operating-point tension, with cable direction and aerial-anchor stiffness held fixed. In the full equilibrium problem, cable tensions and directions are jointly determined by the setpoints, anchor stiffnesses, gravity, and external forces.
\section{Gravity-Aware Stiffness Map}
\label{sec:gravity_aware_stiffness_map}

\subsection{Equilibrium Branch and Stiffness Map}
\label{subsec:equilibrium_dependent_stiffness_map}

Consider an open set $\mathcal U\subset\Eaff^n$ of commanded
aerial-anchor configurations and a selected smooth taut equilibrium
branch
\[
    \EquilibriumBranch{b}:\mathcal U\to\Eaff,
    \qquad
    \EquilibriumResidual
    \big(
    \EquilibriumBranch{b}(\setpointTuple),
    \setpointTuple
    \big)
    =
    0_{\Vtrans}.
\]
The corresponding coordinate branch is denoted by
\[
    \EquilibriumBranchCoord{b}:
    \mathcal U^{\mathsf w}\subset\R^{3n}
    \to\Rthree.
\]
The branch index is omitted below when no ambiguity arises.

At each selected equilibrium, the intrinsic stiffness is an element of
the space of symmetric positive-definite linear maps from $\Vtrans$ to
$\Vdual$. The gravity-aware stiffness map is therefore
\begin{equation}
    \StiffnessMap{b}:
    \mathcal U
    \to
    \Sym^{+}(\Vtrans,\Vdual),
    \quad
    \StiffnessMap{b}(\setpointTuple)
    :=
    \TotalStiffness
    \big(
    \EquilibriumBranch{b}(\setpointTuple),
    \setpointTuple
    \big).
    \label{eq:intrinsic_stiffness_map}
\end{equation}
Its six-dimensional world-frame representation is
\begin{equation}
    \StiffnessCoordMap{b}(\setpointCoord)
    :=
     \bm k(\setpointCoord)
    :=
    \vech\!\left(
    \TotalStiffnessMat
    \big(
    \EquilibriumBranchCoord{b}(\setpointCoord),
    \setpointCoord
    \big)
    \right)
    \in\R^6.
    \label{eq:Phi}
\end{equation}
This map is nonlinear because the setpoints affect stiffness directly
through the virtual aerial anchors and indirectly through the
gravity-loaded equilibrium branch.

Figure~\ref{fig:gravity_aware_stiffness_pipeline} summarizes this
dependency and the passage from the intrinsic physical formulation to
the coordinate representation used for sensitivity analysis and
control.

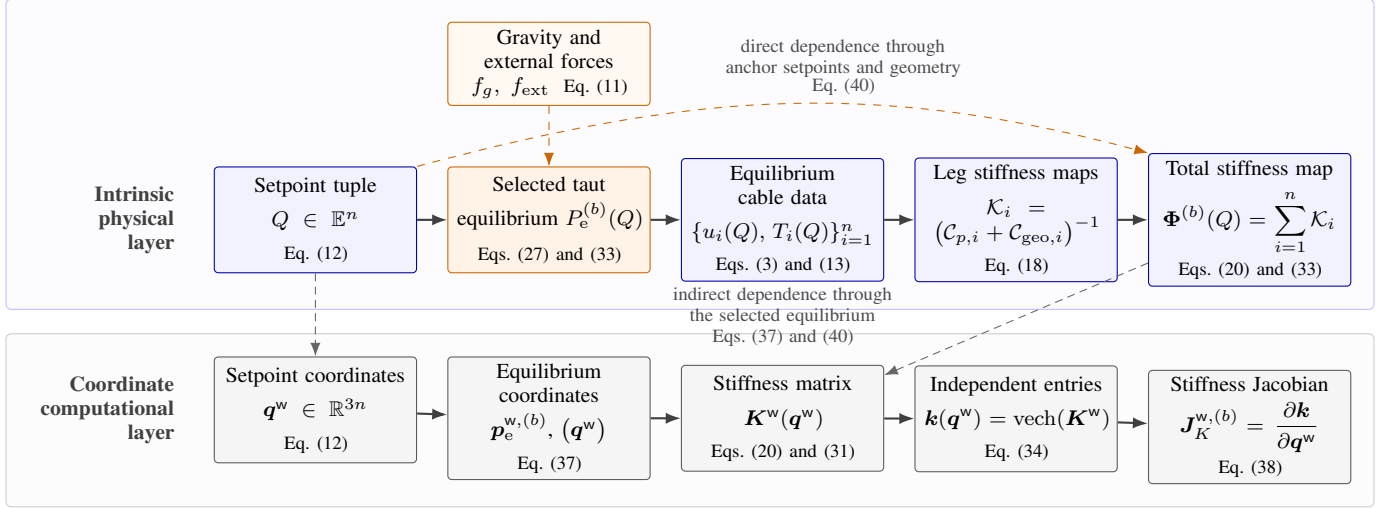
\begin{figure*}[t]
    \centering
\begin{tikzpicture}[
        node distance=11mm and 4mm,
        theoryblock/.append style={
            text width=25mm,
            minimum height=14mm,
            inner sep=2.5pt
        },
        equilibriumblock/.append style={
            text width=25mm,
            minimum height=14mm,
            inner sep=2.5pt
        },
        coordinateblock/.append style={
            text width=25mm,
            minimum height=14mm,
            inner sep=2.5pt
        },
        inputblock/.append style={
            text width=25mm,
            minimum height=11mm,
            inner sep=2.5pt
        },
        layerlabel/.append style={
            text width=18mm
        },
        diagramnote/.append style={
            text width=34mm
        }
    ]

    \node[layerlabel] (intrinsiclabel)
        {Intrinsic\\physical\\layer};

    \node[theoryblock, right=4mm of intrinsiclabel] (Q)
        {Setpoint tuple\\[0.5em]
        $\setpointTuple\in\Eaff^n$\\[0.4em]
        {\scriptsize Eq.~\eqref{eq:q_definition}}};

    \node[equilibriumblock, right=of Q] (Pe)
        {Selected taut\\[0.5em] equilibrium
        $\EquilibriumBranch{b}
        (\setpointTuple)$\\[0.5em]
        {\scriptsize
        Eqs.~\eqref{eq:h_equilibrium}
        and~\eqref{eq:intrinsic_stiffness_map}}};

    \node[theoryblock, right=of Pe] (state)
        {Equilibrium cable data\\[0.5em]
        $\{\cableDir{i}(\setpointTuple),$
        $\cableTension{i}(\setpointTuple)\}_{i=1}^n$\\[0.4em]
        {\scriptsize Eqs.~\eqref{eq:cable_geometry}
        and~\eqref{eq:equilibrium_set}}};

    \node[theoryblock, right=of state] (legs)
        {Leg stiffness maps\\[0.5em]
        $\displaystyle
        \LegStiffness{i}
        =
        \big(
        \AnchorCompliance{i}
        +
        \GeometricCompliance{i}
        \big)^{-1}$\\[0.4em]
        {\scriptsize Eq.~\eqref{eq:intrinsic_single_leg_stiffness}}};

    \node[theoryblock, right=of legs] (K)
        {Total stiffness map\\[0.5em]
        $\displaystyle
        \StiffnessMap{b}
        (\setpointTuple)$
        $\displaystyle
        =
        \sum_{i=1}^{n}
        \LegStiffness{i}$\\[0.4em]
        {\scriptsize Eqs.~\eqref{eq:total_stiffness}
        and~\eqref{eq:intrinsic_stiffness_map}}};

    \draw[flowarrow] (Q) -- (Pe);
    \draw[flowarrow] (Pe) -- (state);
    \draw[flowarrow] (state) -- (legs);
    \draw[flowarrow] (legs) -- (K);
    \node[inputblock, above=8mm of Pe] (loads)
        {Gravity and external forces\\
        $\gravityVector,\ \externalVector$ 
        \ {\scriptsize Eq.~\eqref{eq:gravity_equilibrium}}};

    \draw[dependence] (loads) -- (Pe);
    \draw[dependence]
        (Q.north east)
        to[out=15,in=165]
        node[
            diagramnote,
            above,
            xshift=8mm,
            text width=31mm
        ]
        {direct dependence through\\
        anchor setpoints and geometry\\[-0.2em]
        {\scriptsize Eq.~\eqref{eq:K_chain_rule}}}
        (K.north west);
    \node[
        diagramnote,
        below=-0.5mm of state,
        text width=44mm
    ] (indirectnote)
        {indirect dependence through\\
        the selected equilibrium\\
        {\scriptsize
        Eqs.~\eqref{eq:equilibrium_sensitivity}
        and~\eqref{eq:K_chain_rule}}};

    \node[coordinateblock, below=11mm of Q] (q)
        {Setpoint coordinates\\[0.5em]
        $\setpointCoord\in\R^{3n}$\\[0.4em]
        {\scriptsize Eq.~\eqref{eq:q_definition}}};

    \node[coordinateblock, below=11mm of Pe] (pew)
        {Equilibrium coordinates\\[0.5em]
        $\EquilibriumBranchCoord{b}$,
        $\big(\setpointCoord\big)$\\[0.4em]
        {\scriptsize Eq.~\eqref{eq:equilibrium_sensitivity}}};

    \node[coordinateblock, below=11mm of state] (Kw)
        {Stiffness matrix\\[0.5em]
        $\TotalStiffnessMat
        (\setpointCoord)$\\[0.4em]
        {\scriptsize Eqs.~\eqref{eq:total_stiffness}
        and~\eqref{eq:shape_decomp}}};

    \node[coordinateblock, below=11mm of legs] (kvec)
        {Independent entries\\[0.5em]
        $\displaystyle
        \bm k(\setpointCoord)$
        $\displaystyle
        =
        \vech(\TotalStiffnessMat)$\\[0.4em]
        {\scriptsize Eq.~\eqref{eq:Phi}}};

    \node[coordinateblock, below=11mm of K] (JK)
        {Stiffness Jacobian\\[0.5em]
        $\displaystyle
        \StiffnessJacobian{b}$
        $\displaystyle
        =
        \frac{\partial\bm k}
        {\partial\setpointCoord}$\\[0.4em]
        {\scriptsize Eq.~\eqref{eq:JK_def}}};

    \node[layerlabel, left=4mm of q] (coordinatelabel)
        {Coordinate\\computational\\layer};

    \draw[flowarrow] (q) -- (pew);
    \draw[flowarrow] (pew) -- (Kw);
    \draw[flowarrow] (Kw) -- (kvec);
    \draw[flowarrow] (kvec) -- (JK);

    \draw[representation] (Q) -- (q);

    \draw[representation] (K) -- (Kw);

    \begin{scope}[on background layer]
        \node[
            draw=blue!25,
            fill=blue!1,
            rounded corners=2pt,
            fit=(intrinsiclabel)(Q)(Pe)(state)(legs)(K)(loads),
            inner sep=3mm
        ] {};

        \node[
            draw=black!20,
            fill=black!1,
            rounded corners=2pt,
            fit=(coordinatelabel)(q)(pew)(Kw)(kvec)(JK),
            inner sep=3mm
        ] {};
    \end{scope}

\end{tikzpicture}

    \caption{Gravity-aware stiffness pipeline and parallel intrinsic and
    coordinate descriptions. The aerial-anchor setpoints determine a selected
    gravity-loaded equilibrium branch and thereby the equilibrium cable
    directions and tensions. These quantities determine the leg and total
    stiffness maps. Their coordinate representations yield the
    stiffness-coordinate vector and its Jacobian for sensitivity analysis and
    control. Gravity affects stiffness through the selected equilibrium rather
    than through an additive stiffness term.}

    \label{fig:gravity_aware_stiffness_pipeline}
\end{figure*}

\subsection{Equilibrium Sensitivity}
\label{subsec:equilibrium_sensitivity}

Let $\delta q\in T_{\setpointTuple}\Eaff^n\simeq\Vtrans^n$ be a
setpoint variation, and let
$\delta p_{\mathrm e}\in\Vtrans$ be the induced variation of the
selected equilibrium. Differentiating the branch equilibrium
condition gives
\begin{equation}
    d_P\EquilibriumResidual
    \big(\delta p_{\mathrm e}\big)
    +
    d_Q\EquilibriumResidual
    \big(\delta q\big)
    =
    0_{\Vtrans},
    \label{eq:equilibrium_differential_intrinsic}
\end{equation}
where both partial differentials are evaluated at
$(\EquilibriumBranch{b}(\setpointTuple),\setpointTuple)$.

By the restoring-force definition
$\delta\varphi_c=-\TotalStiffness(\delta p)$ established in
Corollary~\ref{cor:total_passive_stiffness}, a load variation at fixed
setpoints satisfies
\[
    d_P\EquilibriumResidual(\delta p)
    =
    -(\sharp\circ\TotalStiffness)(\delta p).
\]
Therefore,
$d_P\EquilibriumResidual=-\TotalStiffness^\sharp$, where
$\TotalStiffness^\sharp:=\sharp\circ\TotalStiffness:
\Vtrans\to\Vtrans$.
Since $\TotalStiffness$ is positive definite, this operator is
invertible, and \eqref{eq:equilibrium_differential_intrinsic} gives
\begin{equation}
    \delta p_{\mathrm e}
    =
    \big(\TotalStiffness^{\sharp}\big)^{-1}
    d_Q\EquilibriumResidual
    \big(\delta q\big).
    \label{eq:equilibrium_sensitivity_intrinsic}
\end{equation}

In world-frame coordinates, the same relation is
\begin{equation}
    \delta\loadPosEq
    =
    {\TotalStiffnessMat}^{-1}
    \frac{\partial\EquilibriumResidualCoord}
         {\partial\setpointCoord}
    \setpointVariation,
    \qquad
    \frac{\partial\EquilibriumBranchCoord{b}}
         {\partial\setpointCoord}
    =
    {\TotalStiffnessMat}^{-1}
    \frac{\partial\EquilibriumResidualCoord}
         {\partial\setpointCoord}.
    \label{eq:equilibrium_sensitivity}
\end{equation}
Thus, the stiffness governing the local force response also determines
the first-order sensitivity of the gravity-loaded equilibrium to
aerial-anchor setpoint variations.

\subsection{Stiffness Jacobian}
\label{subsec:stiffness_jacobian}

The stiffness Jacobian is the coordinate matrix of the differential of
$\StiffnessCoordMap{b}$:
\begin{equation}
    \StiffnessJacobian{b}(\setpointCoord)
    :=
    \frac{\partial\StiffnessCoordMap{b}}
         {\partial\setpointCoord}
    \in\R^{6\times3n}.
    \label{eq:JK_def}
\end{equation}
Hence, for a setpoint variation,
\begin{equation}
    \delta\bm k
    =
    d\StiffnessCoordMap{b}_{\setpointCoord}
    (\setpointVariation)
    =
    \StiffnessJacobian{b}(\setpointCoord)
    \setpointVariation.
    \label{eq:stiffness_jacobian_variation}
\end{equation}

The Jacobian includes both the direct dependence of stiffness on the
setpoints and the indirect dependence induced by the equilibrium
shift. In coordinates,
\begin{equation}
    \delta\TotalStiffnessMat
    =
    \left[
    \frac{\partial\TotalStiffnessMat}{\partial\setpointCoord}
    +
    \frac{\partial\TotalStiffnessMat}{\partial\loadPosEq}
    \frac{\partial\EquilibriumBranchCoord{b}}
         {\partial\setpointCoord}
    \right]
    \setpointVariation.
    \label{eq:K_chain_rule}
\end{equation}

The formulation applies to any number $n$ of aerial anchors. The
locally achievable stiffness variations at a given taut equilibrium
are characterized by
\[
    \im\StiffnessJacobian{b}(\setpointCoord)
    \subseteq\R^6.
\]
The number of vehicles influences the dimension of the setpoint space,
the rank of the stiffness Jacobian, the availability of null-space
motions, and positive-tension feasibility.
\section{Constraint-Preserving Local Stiffness Regulation}
\label{sec:quasistatic_stiffness_control}

The stiffness Jacobian in \eqref{eq:JK_def} characterizes the
first-order stiffness variations generated by commanded-anchor motion.
The optimization-based local controller presented in this section uses that differential relation to attempt regulating the stiffness
locally while preserving tautness, the selected equilibrium branch,
and the imposed operational constraints. Thanks to such safety mechanisms, the desired stiffness may be assigned directly and need not be known to be globally realizable.

\subsection{Stiffness Error and Admissible Velocities}
\label{subsec:stiffness_objective_admissibility}

Let $\DesiredStiffnessMat(t)\in\Spp(3)$ be the desired world-frame
stiffness matrix. Define the stiffness error and its weighted energy by
\begin{equation}
\begin{aligned}
    \StiffnessError(\setpointCoord,t)
    &:=
    \bm k(\setpointCoord)
    -
    \vech\!\left(\DesiredStiffnessMat(t)\right),
    \\
    \StiffnessErrorEnergy(\setpointCoord,t)
    &:=
    \frac{1}{2}
    \StiffnessError\topT
    \StiffnessWeightMat
    \StiffnessError,
    \qquad
    \StiffnessWeightMat\in\Spp(6).
\end{aligned}
\label{eq:stiffness_error_function}
\end{equation}
For a fixed desired stiffness, \eqref{eq:stiffness_jacobian_variation} gives
\begin{equation}
    \dot{\StiffnessErrorEnergy}
    =
    \StiffnessError\topT
    \StiffnessWeightMat
    \StiffnessJacobian{b}
    \setpointVelocity,
    \label{eq:instantaneous_stiffness_error_derivative}
\end{equation}
where the configuration arguments are omitted when unambiguous.

Let $\mathcal U^{\mathsf w}\subset\R^{3n}$ be the coordinate domain of
the selected smooth taut equilibrium branch. Operational requirements
are represented by continuously differentiable functions
\[
    \AdmissibilityCoordFunction{r}:
    \mathcal U^{\mathsf w}\to\R,
    \qquad
    r=1,\ldots,\NumConstraints,
\]
with $\AdmissibilityCoordFunction{r}\geq0$ denoting admissibility, and
corresponding domain 
\begin{equation}
    \AdmissibleBranchCoordDomain{b}
    :=
    \left\{
        \setpointCoord\in\mathcal U^{\mathsf w}:
        \AdmissibilityCoordFunction{r}(\setpointCoord)\geq0,\ 
        r=1,\ldots,\NumConstraints
    \right\}.
    \label{eq:admissible_branch_domain}
\end{equation}
These functions may encode tension bounds, flight regions, collision
and cable-separation margins, thrust limits, and distance from
equilibrium-branch singularities. For example,
\begin{equation}
\begin{aligned}
    \MinTensionConstraint{i}
    &:=
    \cableTension{i}(\setpointCoord)-\MinTension,
    &
    \MaxTensionConstraint{i}
    &:=
    \MaxTension-\cableTension{i}(\setpointCoord).
\end{aligned}
\label{eq:tension_admissibility_functions}
\end{equation}
The selected-branch tension
$\cableTension{i}(\setpointCoord)$ includes the equilibrium displacement
induced by the commanded-anchor configuration.

Let each $\AdmissibilityRateFunction{r}$ be locally Lipschitz and
strictly increasing, with $\AdmissibilityRateFunction{r}(0)=0$. The
locally admissible candidate velocities $\CandidateSetpointVelocity$ are asked to satisfy
\begin{equation}
\begin{aligned}
    \nabla_{\setpointCoord}
    {\AdmissibilityCoordFunction{r}}\topT
    \CandidateSetpointVelocity
    &\geq
    -
    \AdmissibilityRateFunction{r}
    \big(\AdmissibilityCoordFunction{r}\big),
    \\
    \MinSetpointVelocity
    \leq
    \CandidateSetpointVelocity
    &\leq
    \MaxSetpointVelocity,
    \qquad
    r=1,\ldots,\NumConstraints.
\end{aligned}
\label{eq:locally_admissible_velocity_set}
\end{equation}
All functions in \eqref{eq:locally_admissible_velocity_set} are
evaluated at $\setpointCoord$, and the velocity bounds are understood
componentwise. They are chosen to admit
$\CandidateSetpointVelocity=\zerovec{3n}$ throughout
$\AdmissibleBranchCoordDomain{b}$.

\subsection{Constraint-Preserving Regulator}
\label{subsec:constraint_preserving_stiffness_regulation}

The nominal stiffness velocity is chosen to be
\[
    \NominalStiffnessVelocity
    =
    -\StiffnessRateGain\StiffnessError,
    \qquad
    \StiffnessRateGain>0.
\]
Since this velocity may lie outside
$\im\StiffnessJacobian{b}$ or require inadmissible commanded-anchor
motion, the controller solves
\begin{equation}
\begin{aligned}
    \OptimalSetpointVelocity
    \in
    \argmin_{\CandidateSetpointVelocity\in\R^{3n}}
    \quad&
    \frac{1}{2}
    \WeightedNorm{
        \StiffnessJacobian{b}
        \CandidateSetpointVelocity
        +
        \StiffnessRateGain\StiffnessError
    }{\StiffnessWeightMat}^{2}
    +
    \frac{\VelocityRegularization}{2}
    \norm{\CandidateSetpointVelocity}^{2}
    \\
    \text{subject to}\quad&
    \eqref{eq:locally_admissible_velocity_set},
    \\
    &
    \StiffnessError\topT
    \StiffnessWeightMat
    \StiffnessJacobian{b}
    \CandidateSetpointVelocity
    \leq0,
\end{aligned}
\label{eq:constraint_preserving_stiffness_control}
\end{equation}
where $\VelocityRegularization>0$, and all quantities are
evaluated at $(\setpointCoord,t)$. The commanded-anchor velocity is
\begin{equation}
    \setpointVelocity
    =
    \OptimalSetpointVelocity.
    \label{eq:optimal_setpoint_velocity}
\end{equation}

The objective selects the admissible commanded-anchor velocity that
best approximates the nominal stiffness evolution while penalizing
unnecessary motion. The final inequality guarantees
$\dot{\StiffnessErrorEnergy}\leq0$ for a fixed target. Since zero
velocity satisfies all constraints, the optimization remains feasible
at every admissible configuration.

\begin{proposition}[Constraint-preserving local regulation]
\label{prop:constraint_preserving_local_stiffness_regulation}
Suppose that
$\setpointCoord(0)\in\AdmissibleBranchCoordDomain{b}$, the selected
equilibrium branch and the functions entering
\eqref{eq:constraint_preserving_stiffness_control} are continuously
differentiable, and the commanded evolution admits an absolutely
continuous solution. Then
$\setpointCoord(t)\in\AdmissibleBranchCoordDomain{b}$ throughout its
solution interval. For a fixed desired stiffness,
$\dot{\StiffnessErrorEnergy}\leq0$ almost everywhere.
\end{proposition}

\begin{proof}
Along the commanded evolution,
\(
    \frac{\mathrm d}{\mathrm dt}
    \AdmissibilityCoordFunction{r}
    \geq
    -
    \AdmissibilityRateFunction{r}
    \big(\AdmissibilityCoordFunction{r}\big)
\)
almost everywhere. The comparison principle preserves the
nonnegativity of every admissibility function. The last constraint in
\eqref{eq:constraint_preserving_stiffness_control}, together with
\eqref{eq:instantaneous_stiffness_error_derivative}, gives
$\dot{\StiffnessErrorEnergy}\leq0$.
\end{proof}

\subsection{Local Blockage Status and Feedforward Extension}
\label{subsec:controller_status_extensions}

Convergence is reported when
$\norm{\StiffnessError}\leq\StiffnessErrorTolerance$. Otherwise, local
blockage is declared when
\begin{equation}
\begin{aligned}
    \norm{\OptimalSetpointVelocity}
    &\leq
    \VelocityTolerance,
    &
    -\dot{\StiffnessErrorEnergy}
    &\leq
    \DescentTolerance
\end{aligned}
\label{eq:local_blockage_condition}
\end{equation}
for a prescribed dwell time. This condition indicates negligible local
progress and may result from active constraints, insufficient rank of
the stiffness Jacobian, or a constrained stationary point. It does not
establish global nonrealizability.

If the desired stiffness-rate feedforward  
\[
    \DesiredStiffnessVelocity(t)
    :=
    \frac{\mathrm d}{\mathrm dt}
    \vech\!\left(\DesiredStiffnessMat(t)\right)
\]
is available, the stiffness-velocity mismatch in
\eqref{eq:constraint_preserving_stiffness_control} becomes
\[
    \StiffnessJacobian{b}\CandidateSetpointVelocity
    -
    \DesiredStiffnessVelocity
    +
    \StiffnessRateGain\StiffnessError,
\]
and the descent constraint is formed using
$\dot{\StiffnessError}
=
\StiffnessJacobian{b}\CandidateSetpointVelocity
-
\DesiredStiffnessVelocity$.
Simultaneous equilibrium-position and stiffness regulation follows by
replacing the stiffness error and Jacobian with the combined task error
and $\CombinedJacobian$; the admissibility constraints remain
unchanged.

A block diagram of the controller is provided in Figure~\ref{fig:constraint_preserving_control_architecture}.

\begin{figure}[t]
    \centering

    \begin{tikzpicture}[
        node distance=5mm,
        controlblock/.style={
            coordinateblock,
            text width=42mm,
            minimum height=12mm,
            inner sep=2.5pt,
            font=\scriptsize
        },
        systemblock/.style={
            equilibriumblock,
            text width=42mm,
            minimum height=11mm,
            inner sep=2.5pt,
            font=\scriptsize
        },
        ioblock/.style={
            draw=black!55,
            fill=white,
            rounded corners=1.5pt,
            text width=34mm,
            minimum height=8mm,
            align=center,
            inner sep=2mm,
            font=\scriptsize
        },
        sideblock/.style={
            feasibleblock,
            text width=12mm,
            inner sep=2mm,
            font=\scriptsize
        },
        feedbackarrow/.style={
            -{Latex[length=1.5mm]},
            draw=blue!65!black,
            dashed,
            line width=0.8pt
        }
    ]

    \node[ioblock] (desired)
        {Desired stiffness\\
        $\DesiredStiffnessMat(t)$};

    \node[controlblock, below=of desired] (controller)
        {\textbf{Constraint-preserving regulator}\\
        local stiffness-error descent\\
        { Eq.~\eqref{eq:constraint_preserving_stiffness_control}}};

    \node[sideblock, left=5mm of controller] (constraints)
        {Admissibility\\
        constraints};

    \node[sideblock, right=5mm of controller] (status)
        {Converged\\
        or blocked};

    \node[ioblock, below=of controller] (command)
        {Commanded-anchor velocity\\
        $\setpointVelocity=\OptimalSetpointVelocity$};

    \node[systemblock, below=of command] (system)
        {\textbf{Load--cable--vehicle system}\\
        selected gravity-loaded branch};

    \node[ioblock, below=of system] (output)
        {Stiffness and constraint margins};

    \draw[flowarrow] (desired) -- (controller);
    \draw[dependence] (constraints) -- (controller);
    \draw[flowarrow] (controller) -- (status);
    \draw[flowarrow] (controller) -- (command);
    \draw[flowarrow] (command) -- (system);
    \draw[flowarrow] (system) -- (output);

    \draw[feedbackarrow]
        (output.east)
        -- ++(0.95,0)
        |- (controller.south east);

    \end{tikzpicture}

    \caption{Constraint-preserving local stiffness-regulation loop.}
    \label{fig:constraint_preserving_control_architecture}
\end{figure}
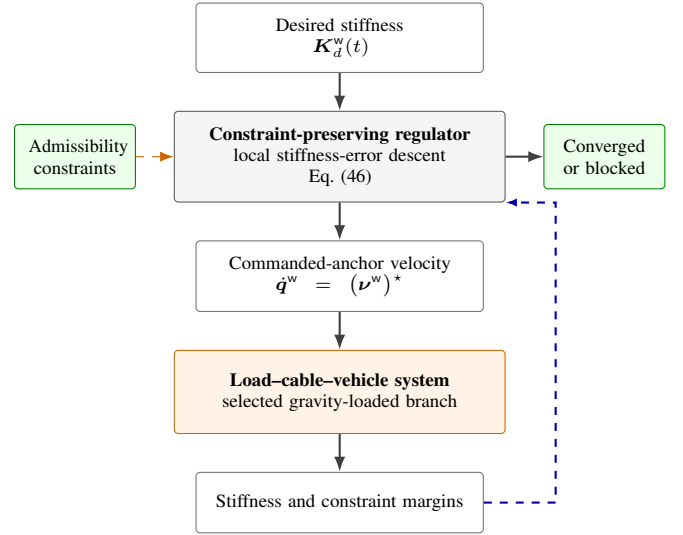
\section{Interpretation, Scope, and Practical Implications}
\label{sec:interpretation_scope}

\subsection{Gravity and the Operating Point}
\label{subsec:interpretation_scope_gravity}

Gravity determines equilibrium feasibility and the tensions supporting
the load. Since
$\geometricStiffnessCoeff{i}
=\cableTension{i}/\cableLength{i}$,
its stiffness contribution is mediated by the gravity-loaded cable
directions and tensions rather than appearing as an additive stiffness
term. The same commanded-anchor geometry may therefore generate
different passive stiffnesses under different load masses or external
forces. Equilibrium and stiffness shaping must consequently be
considered jointly.

In the spatially homogeneous model, equilibrium and stiffness are
invariant under a common translation of the load and all commanded
anchors: such a translation preserves their relative displacements and
therefore the cable directions, tensions, and passive stiffness.
Accordingly, equilibrium-position regulation is implemented by adding to all commanded anchors the same translational
velocity generated from the load-position error, whereas relative
anchor motion performs stiffness shaping. This separation holds before
world-fixed flight-region, obstacle, and other geometric constraints
are imposed; the complete commanded-anchor velocity remains subject to
the admissibility conditions of
\eqref{eq:locally_admissible_velocity_set}.

\subsection{Local Regulation and Constraint Preservation}
\label{subsec:interpretation_scope_local_global}

The regulator accepts an arbitrary instantaneous desired stiffness
without requiring a precomputed stiffness workspace or feasible
reference trajectory. From an admissible commanded-anchor
configuration, it searches locally for a velocity that decreases the
stiffness error while preserving the selected equilibrium branch and
the imposed operational constraints.

The constraints are imposed directly through the admissible velocity
set \eqref{eq:locally_admissible_velocity_set}, rather than treated as
secondary objectives. They may represent tension, geometry, actuation,
and branch-regularity margins. Their gradients account for the direct
setpoint dependence and the dependence induced through the
gravity-loaded equilibrium.

The condition \eqref{eq:local_blockage_condition} reports negligible
progress of the implemented local regulator. Such blockage may result
from active constraints, insufficient rank of the stiffness Jacobian,
or a constrained stationary point. It does not establish global
nonrealizability, since another configuration or equilibrium branch
may realize the desired stiffness.

The constraint-preservation result applies locally to the commanded
quasi-static evolution along the selected equilibrium branch and is
expected to remain physically meaningful when the setpoints evolve
slowly enough that
\begin{equation}
    \loadPos(t)
    \simeq
    \EquilibriumBranchCoord{b}
    \big(\setpointCoord(t)\big).
    \label{eq:quasistatic_branch_condition}
\end{equation}
This condition does not require the actual aerial anchors to coincide
with their commanded points, since their equilibrium deflections
generate the cable tensions. Finite-bandwidth dynamics, tracking and
estimation errors, and disturbances may reduce the physical
admissibility margins; practical implementation therefore requires
sufficiently slow commanded-anchor motion together with positive
tension, separation, and actuation margins.

\subsection{Relation to Cable-Driven Parallel Manipulators}
\label{subsec:relation_cdpm}

The proposed framework is an aerial counterpart of stiffness analysis
in cable-driven parallel manipulators. Conventional systems generally
use fixed frame anchors and actuated cable lengths or tensions. Here,
the aerial vehicles provide movable endpoints whose finite closed-loop
translational stiffness makes them compliant anchors.

Under isotropic aerial-anchor stiffness, an inextensible cable and its
compliant aerial anchor are quasi-statically equivalent to a unilateral
elastic cable connected directly to the commanded anchor point. Flight
therefore shapes passive load stiffness through the coupled variation
of cable directions and gravity-loaded tensions, while introducing
additional workspace, collision, thrust, and equilibrium-branch
constraints.

\subsection{Passive Stiffness and Active Impedance Control}
\label{subsec:relation_active_impedance}

The stiffness studied here is a local equilibrium property of the
cable--vehicle--load interconnection. Active impedance or admittance
control instead imposes dynamic behavior through feedback. The two are
complementary: passive stiffness shaping modifies the immediate
mechanical response, while active control may add damping, regulate
transients, and reject disturbances.

A complete load-port impedance model would additionally include
apparent inertia and damping:
\begin{equation}
    \bm M_{\mathrm{app}}\delta\ddot{\bm p}
    +
    \bm D_{\mathrm{app}}\delta\dot{\bm p}
    +
    \bm K_{\mathrm{app}}\delta\bm p
    =
    \delta\bm f_{\mathrm{ext}}.
    \label{eq:apparent_dynamic_impedance}
\end{equation}
These additional terms depend on the system dynamics, dissipation,
cable properties, actuator bandwidth, and delays. Their derivation
requires a dynamic model and lies outside the present scope.

\subsection{Scope and Validation}
\label{subsec:scope_validation}

The theory is quasi-static and translational and assumes a point load,
taut straight inextensible cables, constant gravity, and finite
Cartesian stiffness at the aerial anchors. It excludes cable sagging,
slack--taut transitions, rigid-load attitude, aerodynamic coupling,
actuator saturation, estimation errors, and contact dynamics.

The dynamic simulations (Sec.~\ref{sec:dynamic_simulation_validation}) test whether the predicted stiffness and local
regulation remain informative under coupled vehicle--load dynamics,
elastic-damped tendons, and environmental interaction. Perturbation
tests compare analytical and identified stiffnesses, while regulation
tests evaluate error reduction and constraint preservation for both
locally attainable and constraint-limited commands. These simulations
stress-test the quasi-static predictions without extending the
invariance guarantee to the full dynamic system.
\section{Dynamic Simulation Validation Framework}
\label{sec:dynamic_simulation_validation}

\subsection{Validation Questions}
\label{subsec:validation_questions}

The dynamic validation campaign is organized around three successive questions:
\begin{enumerate}
    \item Does the quasi-static theory predict the Cartesian stiffness
    empirically observed at the payload under full closed-loop system
    dynamics?

    \item Can aerial-anchor repositioning drive the system toward
    gravity-loaded operating points with prescribed passive stiffness
    while preserving positive cable tensions?

    \item Does task-aligned passive compliance improve the execution of
    contact-rich cable-suspended manipulation tasks?
\end{enumerate}
These questions define a progressive validation chain:
\[
    \text{predict stiffness}
     \longrightarrow\
    \text{shape stiffness}
     \longrightarrow\
    \text{exploit stiffness}.
\]
The first question tests the analytical model beyond its derivation
assumptions. The second evaluates the local stiffness-shaping
controller within the feasible taut domain. The third examines whether
the resulting passive mechanical behavior provides a practical benefit
during physical interaction.

\subsection{Simulation System and Model Fidelity}
\label{subsec:simulation_system_fidelity}

The validation environment is implemented in MuJoCo using rigid-body
dynamics, spatial tendons, and environmental contacts. It comprises
four identical quadrotors connected by elastic-damped tendons to
spatially separated attachment sites on a free rigid payload. Each
vehicle is controlled by a nonlinear geometric controller on $SE(3)$,
with its commanded position acting as the aerial-anchor setpoint.
Table~\ref{tab:simulation_parameters} summarizes the model and
controller parameters.

\begin{table}[t]
    \centering
    \caption{MuJoCo model and controller parameters.}
    \label{tab:simulation_parameters}
    \footnotesize
    \setlength{\tabcolsep}{4pt}
    \renewcommand{\arraystretch}{1.12}
    \begin{tabular}{
        @{}
        p{0.55\columnwidth}
        p{0.37\columnwidth}
        @{}
    }
        \toprule
        \textbf{Parameter}
        & \textbf{Value}
        \\
        \midrule

        System: number of quadrotors and cables
        & $4$
        \\

        Quadrotor: mass
        & $1.28\,\mathrm{kg}$
        \\

        Quadrotor: principal inertia
        & $\operatorname{diag}(0.015,\,0.015,\,0.007)\,
           \mathrm{kg\,m^2}$
        \\

        Payload: mass
        & $2.00\,\mathrm{kg}$
        \\

        Payload: principal inertia
        & $\operatorname{diag}(0.05,\,0.05,\,0.05)\,
           \mathrm{kg\,m^2}$
        \\

        Translational control: position gain $\bm K_p$
        & $\operatorname{diag}(12,\,12,\,14)\,\mathrm{N/m}$
        \\

        Translational control: velocity gain $\bm D_p$
        & $\operatorname{diag}(8,\,8,\,6)\,\mathrm{N\,s/m}$
        \\

        Attitude control: roll--pitch/yaw gains
        & $20.0/4.0\,\mathrm{N\,m/rad}$
        \\

        Attitude control: roll--pitch/yaw rate gains
        & $1.5/0.35\,\mathrm{N\,m\,s/rad}$
        \\

        Tendon: rest length
        & $1.50\,\mathrm{m}$
        \\

        Tendon: axial stiffness
        & $2500\,\mathrm{N/m}$
        \\

        Tendon: axial damping
        & $150\,\mathrm{N\,s/m}$
        \\

        Simulation: integration timestep
        & $2\,\mathrm{ms}$
        \\

        Geometry: vehicle and payload attachment sites
        & Fixed by the simulation model
        \\

        \bottomrule
    \end{tabular}
\end{table}

The simulation retains the principal physical phenomena needed for
validation: vehicle and payload inertia, nonlinear closed-loop
aerial-vehicle dynamics, finite cable elasticity and damping, offset
payload attachments, payload rotation, translation--rotation coupling,
and contact forces. It is therefore richer than the quasi-static
point-load model used to derive the stiffness map. In particular, the
simulation assesses whether the analytical translational stiffness
remains informative in the presence of rigid-payload motion, tendon
elongation, dissipation, and finite settling dynamics.

The simulator remains an approximation of the physical system. Its
tendon model does not represent distributed cable mass, sagging,
bending, aerodynamic drag, or transverse cable vibration. This
distinguishes the gap between theory and simulation from the remaining
gap between simulation and physical reality.

The mildly anisotropic translational gains permit direct comparison
with Theorem~\ref{thm:single_leg_passive_stiffness}; dedicated
isotropic-gain tests evaluate the virtual elastic-cable equivalence.
Nonzero equilibrium errors between actual and commanded vehicle
positions are intentional, since they generate the cable forces
through the finite aerial-anchor stiffness.

\rev{The payload geometry and attachment-site coordinates are fixed in the simulation model. Their numerical specification, together with measured vehicle and coupled-system settling times, belongs to the reproducibility record of the final validation campaign.}

\subsubsection{Theory-to-Simulation and Simulation-to-Reality Gaps}
\label{subsubsec:model_fidelity_gaps}

Table~\ref{tab:model_fidelity_comparison} distinguishes two modeling
gaps with different roles in the validation. The
\emph{theory-to-simulation gap} is intentionally substantial: the
simulation adds rigid-payload motion, closed-loop vehicle dynamics,
tendon elasticity and damping, and environmental contact. Agreement
with the empirical stiffness under these conditions would indicate
that the quasi-static theory captures the dominant local compliance
mechanism beyond the assumptions used in its derivation. Disagreement
can instead reveal when rotational coupling, cable elongation, or
transient dynamics cease to be secondary.

The \emph{simulation-to-reality gap} concerns effects still omitted
from the numerical model, including distributed cable mass and
sagging, actuator imperfections, sensing errors, and aerodynamic
coupling. The simulations therefore provide a dynamic stress test of
the theory rather than a substitute for physical experiments. Their
purpose is to assess stiffness prediction, anchor reconfiguration, and
contact behavior before these remaining effects are examined
experimentally.

\begin{table*}[t]
    \centering
    \caption{Relationship among the analytical model, the MuJoCo
    implementation, and the corresponding physical system.}
    \label{tab:model_fidelity_comparison}
    \renewcommand{\arraystretch}{1.15}
    \begin{tabular}{
        >{\raggedright\arraybackslash}p{0.08\textwidth}
        >{\raggedright\arraybackslash}p{0.21\textwidth}
        >{\raggedright\arraybackslash}p{0.30\textwidth}
        >{\raggedright\arraybackslash}p{0.31\textwidth}}
        \hline
        \textbf{Phenomenon}
        & \textbf{Analytical representation}
        & \textbf{MuJoCo representation}
        & \textbf{Remaining simulation-to-reality gap}
        \\
        \hline

        Payload
        & Point load with translational displacement
        & Free rigid body with mass, inertia, attitude, and four
          spatially separated attachment sites
        & Structural flexibility, uncertain inertial parameters, and
          unmodeled payload aerodynamics
        \\

        Aerial vehicles
        & Movable anchors with prescribed linear Cartesian stiffness
        & Rigid quadrotors with nonlinear $SE(3)$ tracking and finite translational gains
        & Motor dynamics, thrust uncertainty, battery effects, and
          aerodynamic interaction
        \\

        Anchor compliance
        & Symmetric positive-definite linear stiffness map
        & Controller-induced displacement under cable loading
        & Gain variation, saturation, delays, estimation errors, and off-equilibrium nonlinearities
        \\

        Cables
        & Taut, straight, massless, and inextensible
        & Elastic-damped spatial tendons with finite axial stiffness
          and rest length
        & Distributed mass, sagging, bending, drag, transverse
          vibration, and cable contact
        \\

        Dynamics
        & Quasi-static evolution along a selected equilibrium branch
        & Coupled vehicle--payload transients with inertia and
          dissipation
        & Unmodeled high-frequency dynamics and hardware-dependent
          settling behavior
        \\

        Contact
        & Excluded from the stiffness derivation
        & Optimization-based contact between the payload and
          environment
        & Surface compliance, friction uncertainty, impact dynamics,
          and geometry imperfections
        \\

        Sensing
        & Exact operating-point geometry and tension
        & \rev{Simulator-exact payload, vehicle, cable-geometry, and tendon quantities}
        & Sensor noise, calibration errors, latency, packet loss, and
          imperfect tension estimation
        \\
        \hline
    \end{tabular}
\end{table*}

\subsection{Empirical Cartesian Stiffness Identification}
\label{subsec:empirical_stiffness_identification}

The analytical stiffness is compared with an empirical local stiffness
identified directly from the dynamic simulation. For each tested
gravity-loaded equilibrium, the aerial-anchor setpoints are held fixed
and the payload is subjected to small perturbations around its settled
operating point.

Let
$\{\Delta\bm p_j,\Delta\bm f_j\}_{j=1}^{N_{\mathrm p}}$
denote the measured payload-displacement and restoring-force
variations. The empirical stiffness matrix is obtained from
\begin{equation}
    \EmpiricalStiffnessMat
    :=
    \argmin_{\bm X\in\mathbb S^3}
    \sum_{j=1}^{\NumPerturbations}
    \left\|
    \ForcePerturbation{j}
    +
    \bm X\PositionPerturbation{j}
    \right\|^2.
    \label{eq:empirical_stiffness_identification}
\end{equation}
The perturbations must span $\Rthree$ and remain sufficiently small
for the local linear approximation to apply. The identified matrix is
projected onto $\Spp(3)$ only if needed to suppress numerical
asymmetry or measurement noise.

\rev{The identification protocol uses perturbations spanning the three
translational directions and a settling test based on residual payload
motion. Perturbation amplitudes must be verified through a local-linearity
study. Because the simulated payload is a free rigid body, the reported
quantity is an effective translational stiffness that includes any residual
translation--rotation coupling unless attitude is explicitly constrained.}

\subsection{Analytical Stiffness Validation}
\label{subsec:analytical_stiffness_validation}

This experiment evaluates whether the analytical stiffness
$\TotalStiffnessMat_{\mathrm{pred}}$ predicts the empirical stiffness
$\EmpiricalStiffnessMat$ across a broad set of taut gravity-loaded
configurations. The analytical matrix is evaluated from the measured
operating-point cable directions, tensions, and aerial-anchor
stiffnesses, rather than from nominal formation data.

\begin{figure*}[t]
    \centering

    \begin{subfigure}[t]{0.35\textwidth}
        \centering
        \includegraphics[width=\linewidth]{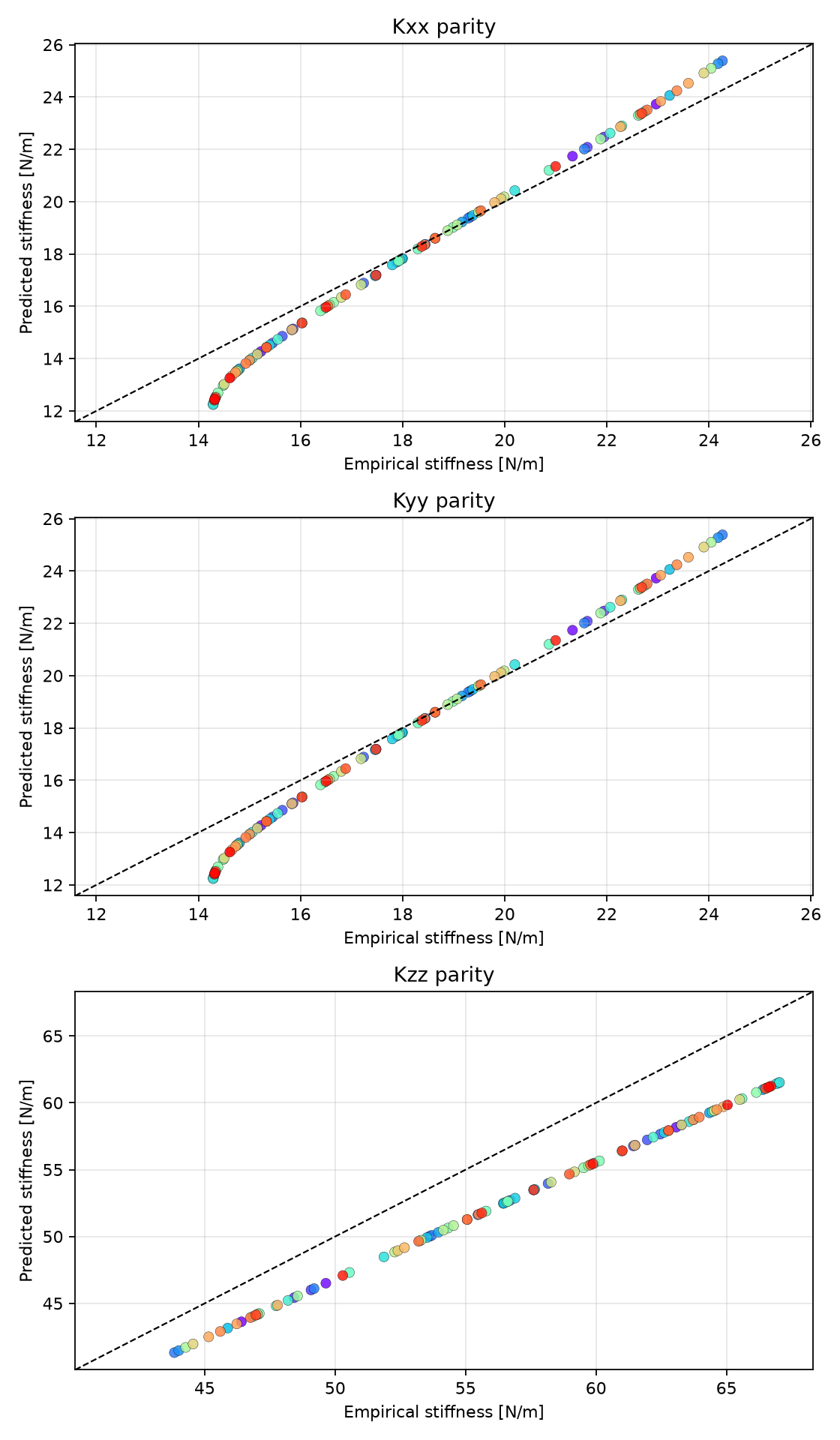}
        \caption{Diagonal stiffness parity.}
        \label{fig:stiffness_parity_diagonal}
    \end{subfigure}
    \hspace{0.01\textwidth}
    \begin{subfigure}[t]{0.35\textwidth}
        \centering
        \includegraphics[width=\linewidth]{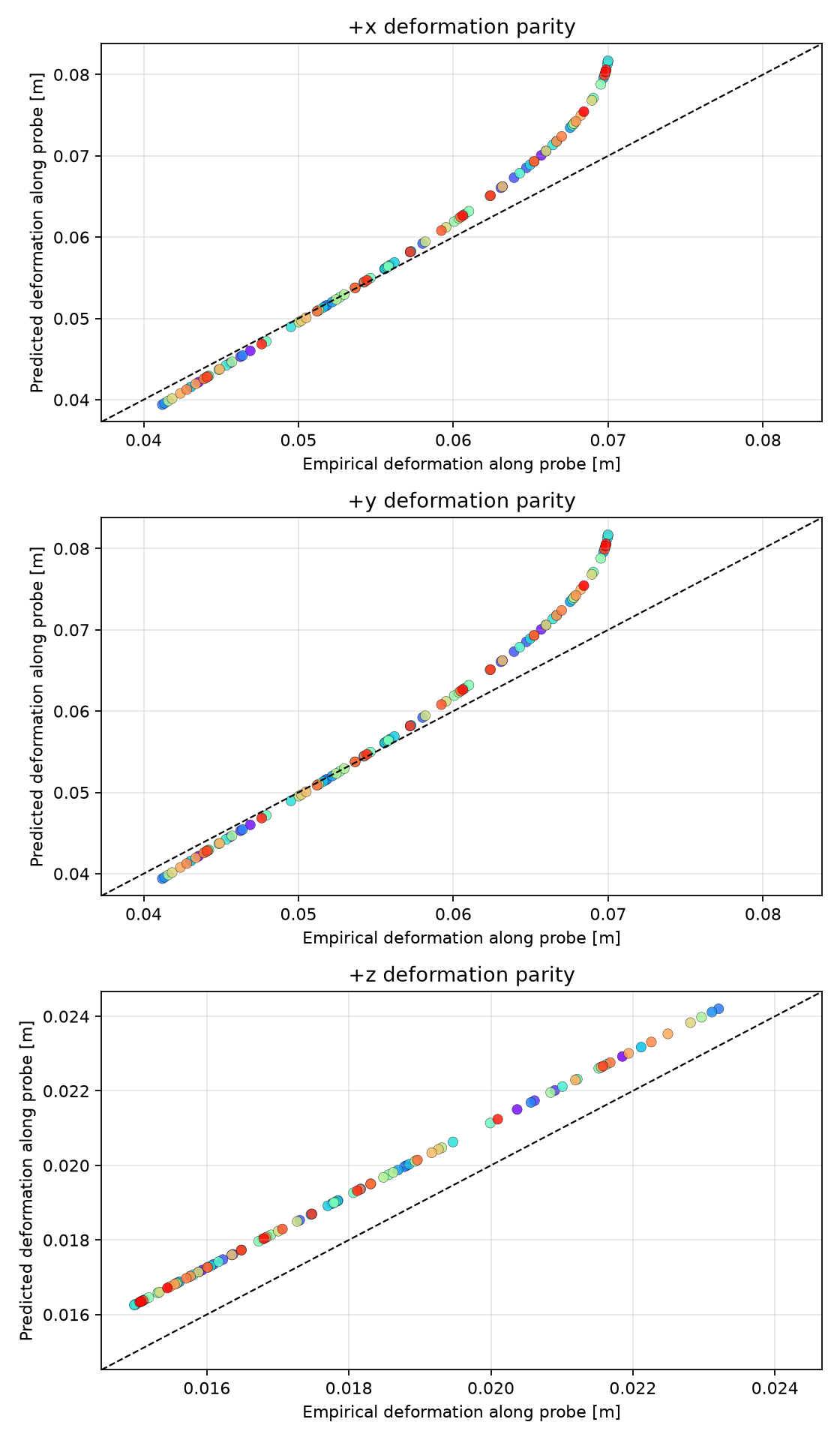}
        \caption{Positive-axis deformation parity.}
        \label{fig:deformation_parity_positive}
    \end{subfigure}

    \caption{Parity-plot comparison across the selected stiffness and deformation quantities.}
    \label{fig:parity_two_panel}
\end{figure*}

\subsubsection{Configuration Campaign}
\label{subsubsec:configuration_campaign}

A set of aerial-anchor configurations is sampled over the admissible
flight region. Each candidate is simulated until the coupled system
settles, after which configurations violating positive-tension,
workspace, collision, or equilibrium requirements are rejected. The
remaining equilibria should cover variations in cable direction,
tension, stiffness anisotropy, and tendon elongation.

\rev{The campaign is designed to report both accepted and rejected
configurations, including cases near low-tension and large-deformation
boundaries. Sampling bounds, acceptance counts, and settling thresholds
are treated as protocol parameters and must accompany the released
numerical dataset.}

\subsubsection{Error Metrics and Statistical Analysis}
\label{subsubsec:stiffness_error_metrics}

For each accepted configuration, define the normalized matrix error
\begin{equation}
    \StiffnessRelativeError
    :=
    \frac{
        \left\|
        \PredictedStiffnessMat-\EmpiricalStiffnessMat
        \right\|_{\mathrm F}
    }{
        \left\|\EmpiricalStiffnessMat\right\|_{\mathrm F}
    }.
    \label{eq:relative_stiffness_error}
\end{equation}

The comparison also considers the principal stiffnesses and principal
directions, because a small aggregate matrix error can conceal a
relevant directional mismatch.

To express the error in task-level units, a common set of test forces
$\Delta\bm f_\ell$ is applied to both compliance matrices:
\begin{equation}
\begin{aligned}
    \PredictedDisplacement{\ell}
    &=
    {\PredictedStiffnessMat}^{-1}\TestForce{\ell},
    \\
    \EmpiricalDisplacement{\ell}
    &=
    {\EmpiricalStiffnessMat}^{-1}\TestForce{\ell}.
\end{aligned}
    \label{eq:deflection_prediction_comparison}
\end{equation}
The resulting displacement error complements the matrix-level metrics
with a directly interpretable physical quantity.

\rev{Parity plots and error distributions are used instead of selected
examples alone. Individual configurations are retained so that aggregate
statistics do not conceal directional or boundary-dependent errors.}

\subsubsection{Failure and Boundary Cases}
\label{subsubsec:stiffness_failure_cases}

Prediction errors are analyzed against the effects excluded from the
analytical model, including tendon elongation, payload rotation,
translation--rotation coupling, proximity to slackness, and incomplete
settling. This analysis serves to identify the operating region in
which the quasi-static inextensible-cable approximation remains
informative.

\rev{Error is examined against minimum cable tension, tendon strain,
payload-attitude deviation, and residual motion. Unfavorable cases remain
part of the analysis because they identify the boundary of the reduced
model's useful operating region.}

\subsection{Stiffness-Shaping Controller Validation}
\label{subsec:controller_validation}

The second validation stage evaluates whether aerial-anchor
repositioning can drive the dynamic system toward prescribed passive
stiffnesses while preserving tautness and the imposed operational
constraints. Desired stiffness matrices are supplied directly to the
local regulator, without prior generation of a feasible stiffness
trajectory. The high-level regulator updates the commanded-anchor
positions, while the low-level $SE(3)$ controllers generate the
corresponding cable-loaded vehicle motion.

The experiments include stiffness commands that are locally attainable
from the initial configuration and commands whose continued pursuit is
limited by tension, flight-region, or other admissibility constraints.
The former evaluate convergence toward the desired stiffness, whereas
the latter evaluate constraint preservation and the local blockage
criterion in \eqref{eq:local_blockage_condition}.

Controller performance is evaluated through the stiffness error
$\StiffnessError(t)$, load-position error, commanded-anchor velocities,
and minimum tension and operational margins. The principal measures are
the initial, final, and minimum attained stiffness errors, load-position
accuracy, convergence time for attained targets, and active constraints
for locally blocked commands. The validity of the quasi-static
approximation is assessed by monitoring the distance between the actual
payload position and the instantaneous selected equilibrium during
aerial-anchor repositioning.

\begin{figure}[t]
    \centering

    \begin{subfigure}[t]{\columnwidth}
        \centering
        \includegraphics[width=\linewidth]{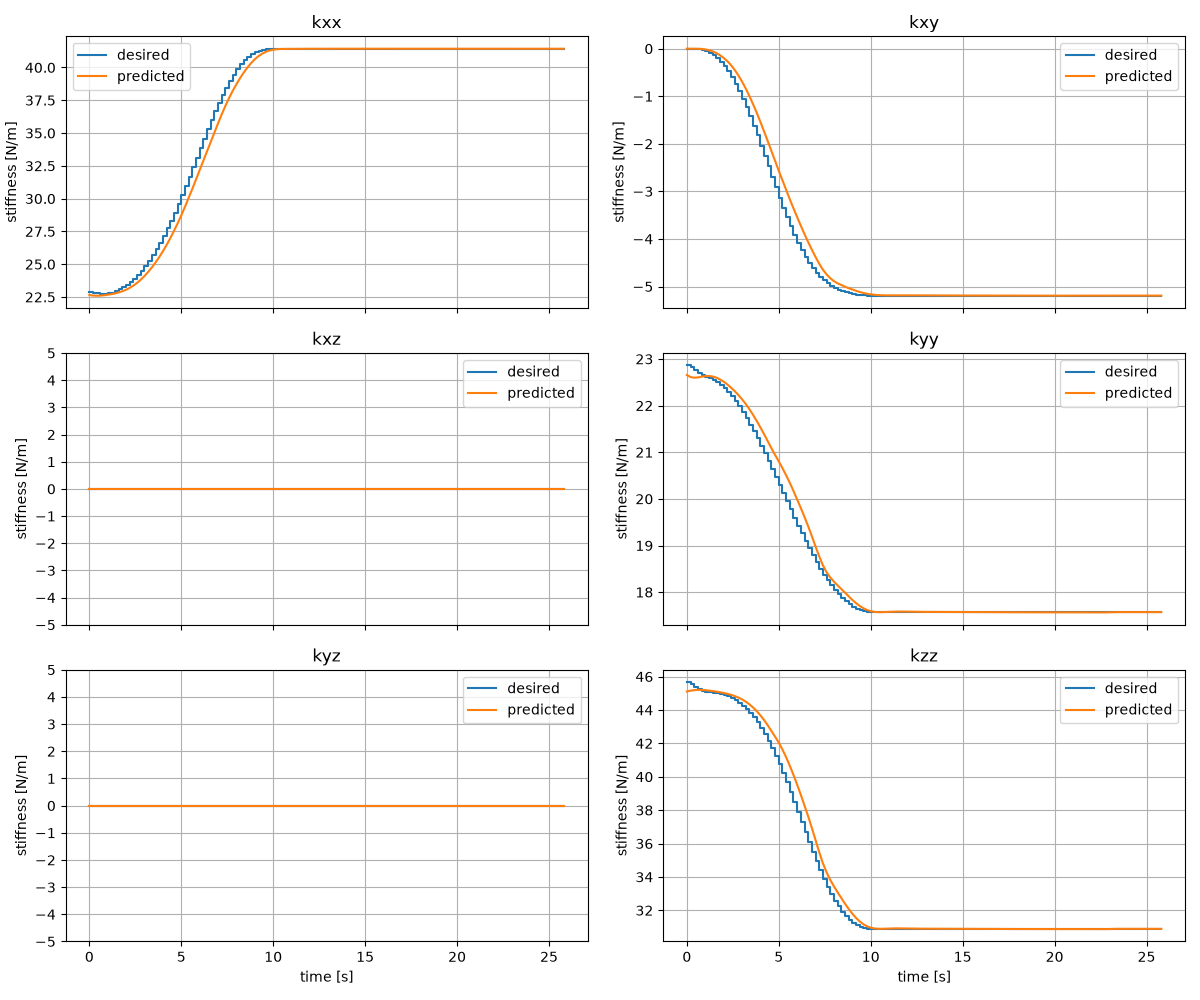}
        \caption{Stiffness regulation: greater stiffness along the longitudinal axis than along the lateral axis}
        \label{fig:stack_1}
    \end{subfigure}

    \vspace{0.5em}

    \begin{subfigure}[t]{\columnwidth}
        \centering
        \includegraphics[width=\linewidth]{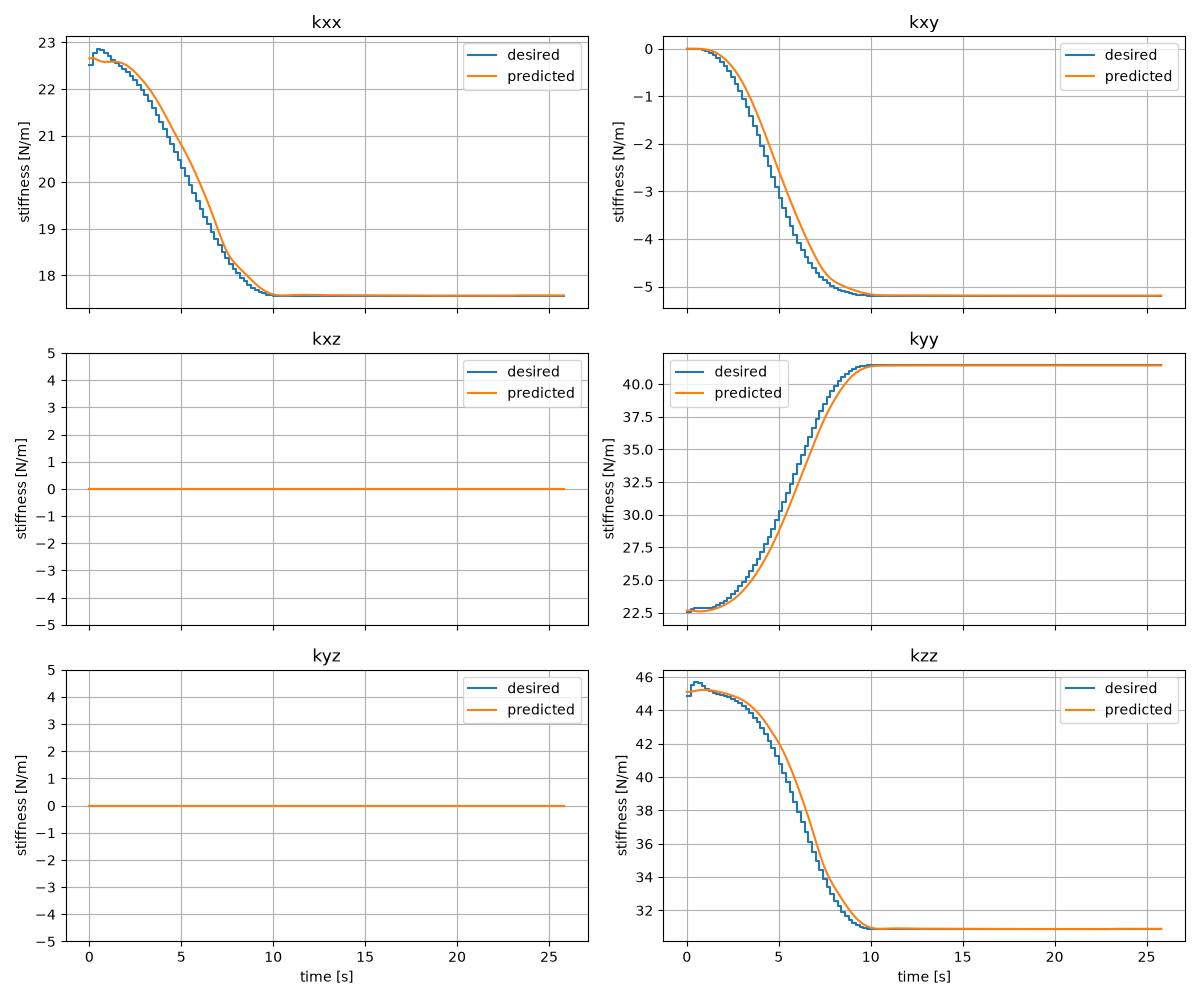}
        \caption{Stiffness regulation: greater stiffness along the lateral axis than along the longitudinal axis}
        \label{fig:stack_2}
    \end{subfigure}
    \caption{Stiffness tracking plots for 2 of the 4 cases used in zigzag sliding}
    \label{fig:stack_asym}
\end{figure}
    
\begin{figure}[t]
    \begin{subfigure}[t]{\columnwidth}
        \centering
        \includegraphics[width=\linewidth]{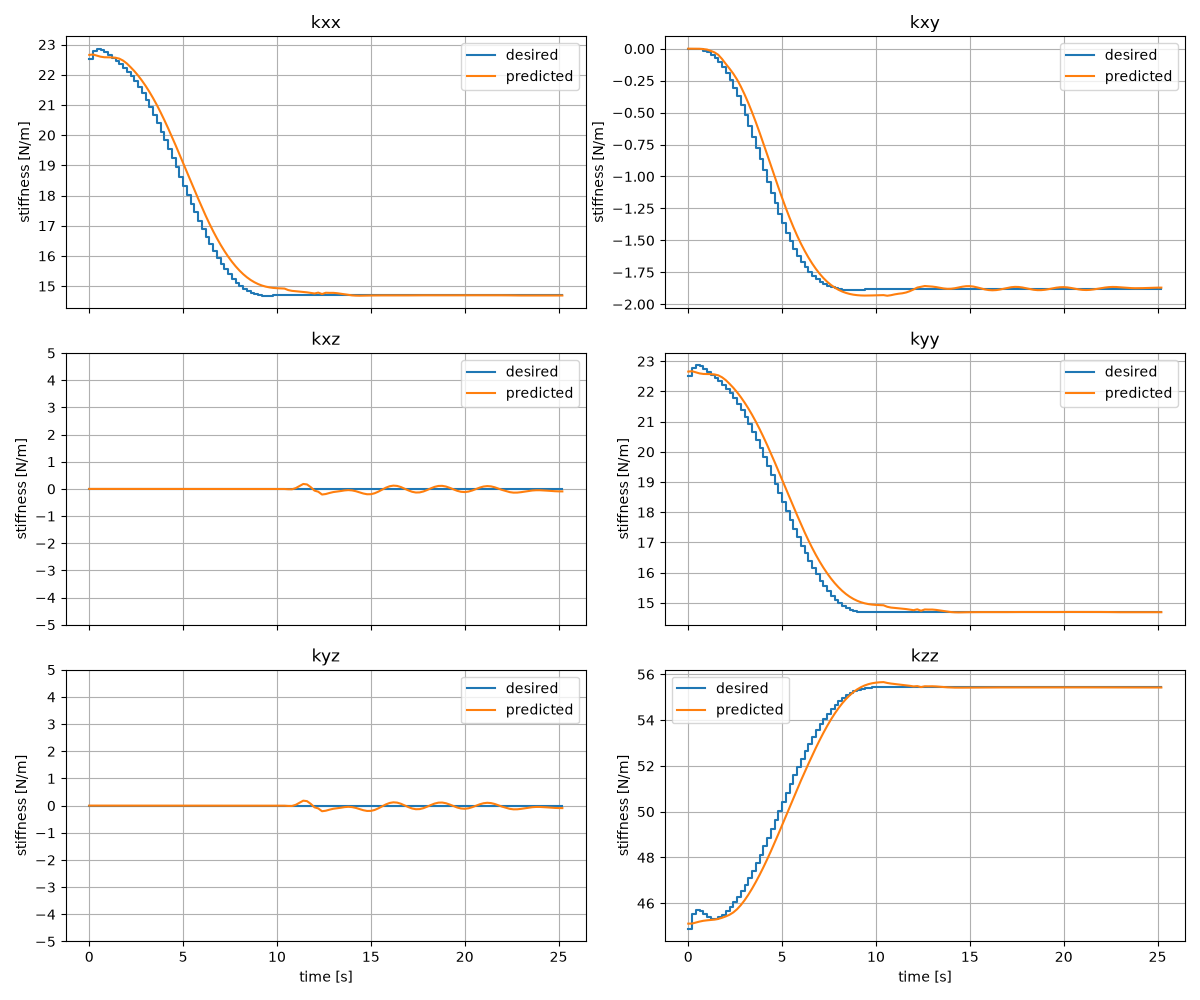}
        \caption{Stiffness regulation: reduced stiffness along both horizontal axes}
        \label{fig:stack_3}
    \end{subfigure}

    \vspace{0.5em}

    \begin{subfigure}[t]{\columnwidth}
        \centering
        \includegraphics[width=\linewidth]{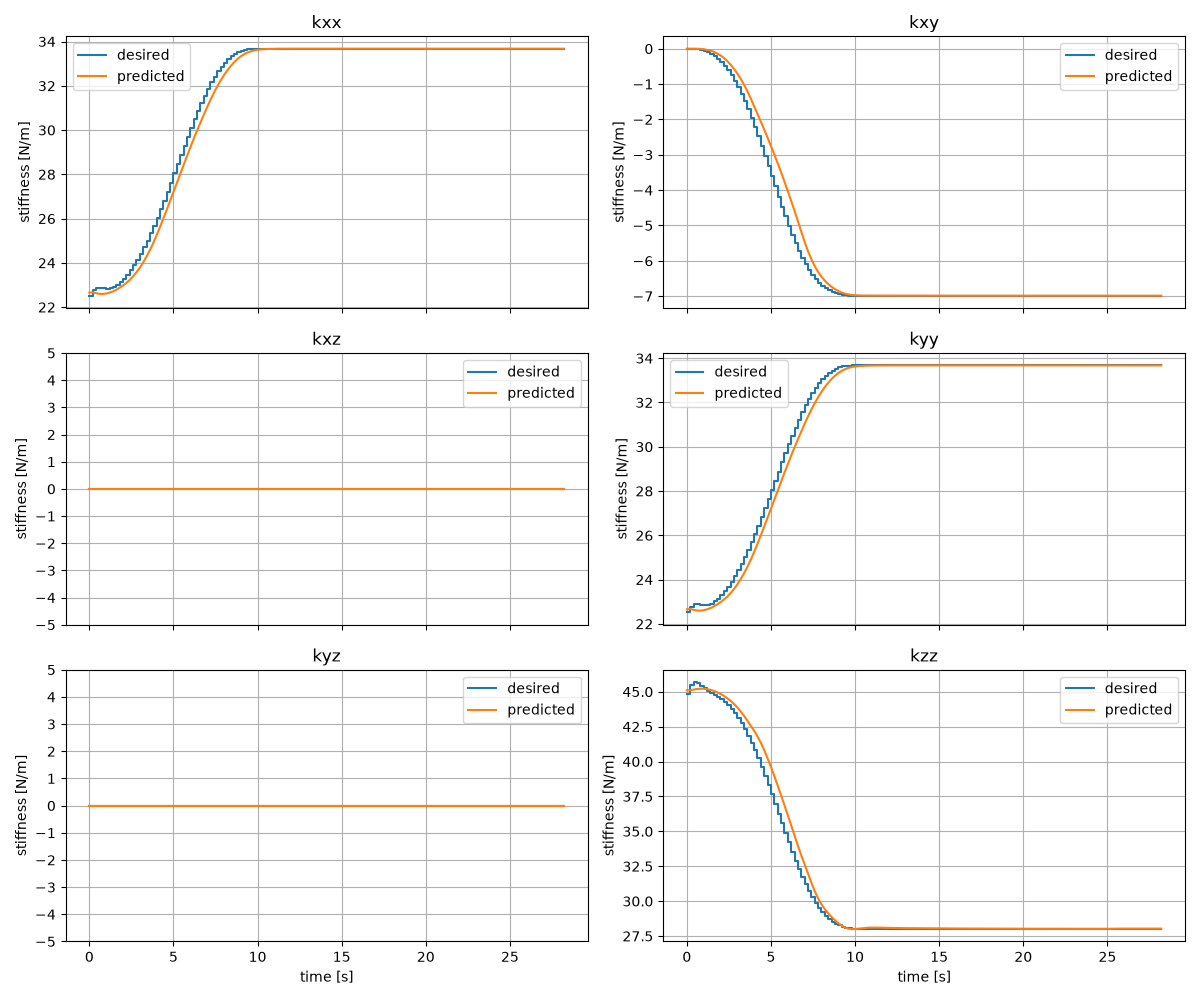}
        \caption{Stiffness regulation: increased stiffness along both horizontal axes}
        \label{fig:stack_4}
    \end{subfigure}

    \caption{Stiffness tracking plots for 2 of the 4 cases used in zigzag sliding}
    \label{fig:stack_sym}
\end{figure}

\begin{figure}[t]
    \centering

    \begin{subfigure}[t]{0.48\columnwidth}
        \centering
        \includegraphics[width=\linewidth]{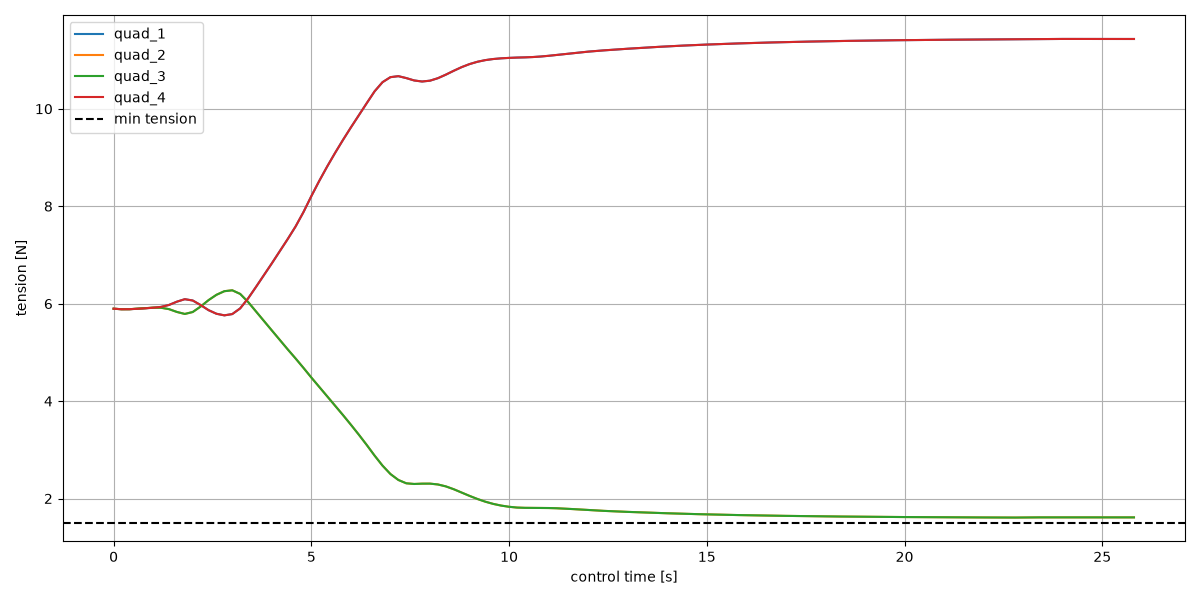}
        \caption{Longitudinally stiff}
        \label{fig:tension_asym_x}
    \end{subfigure}
    \hfill
    \begin{subfigure}[t]{0.48\columnwidth}
        \centering
        \includegraphics[width=\linewidth]{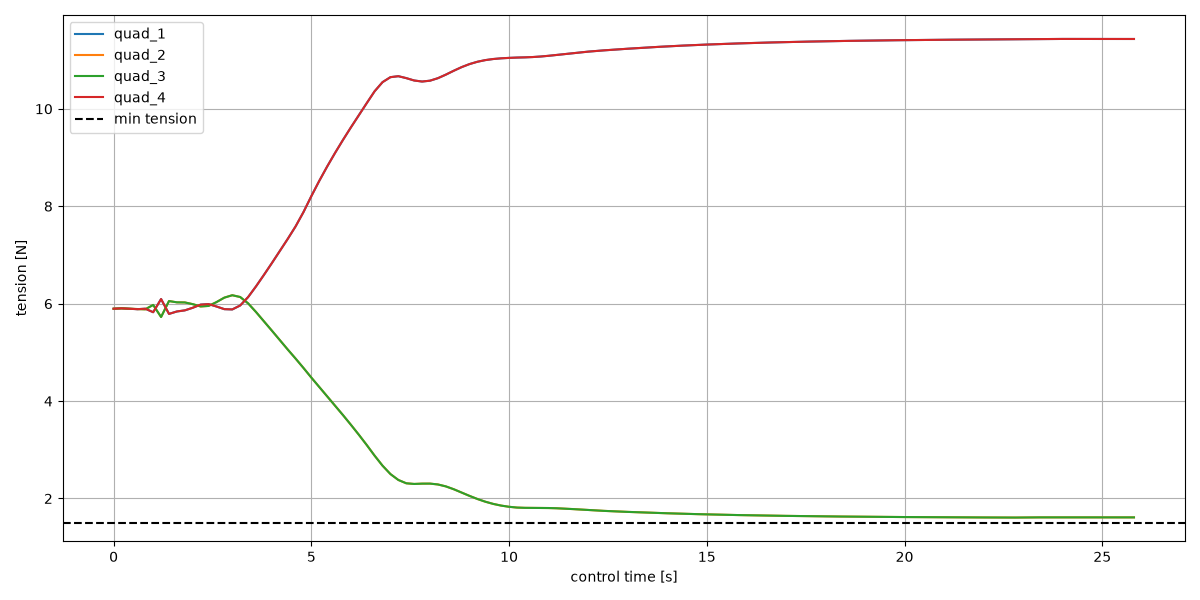}
        \caption{Laterally stiff}
        \label{fig:tension_asym_y}
    \end{subfigure}

    \vspace{0.5em}

    \begin{subfigure}[t]{0.48\columnwidth}
        \centering
        \includegraphics[width=\linewidth]{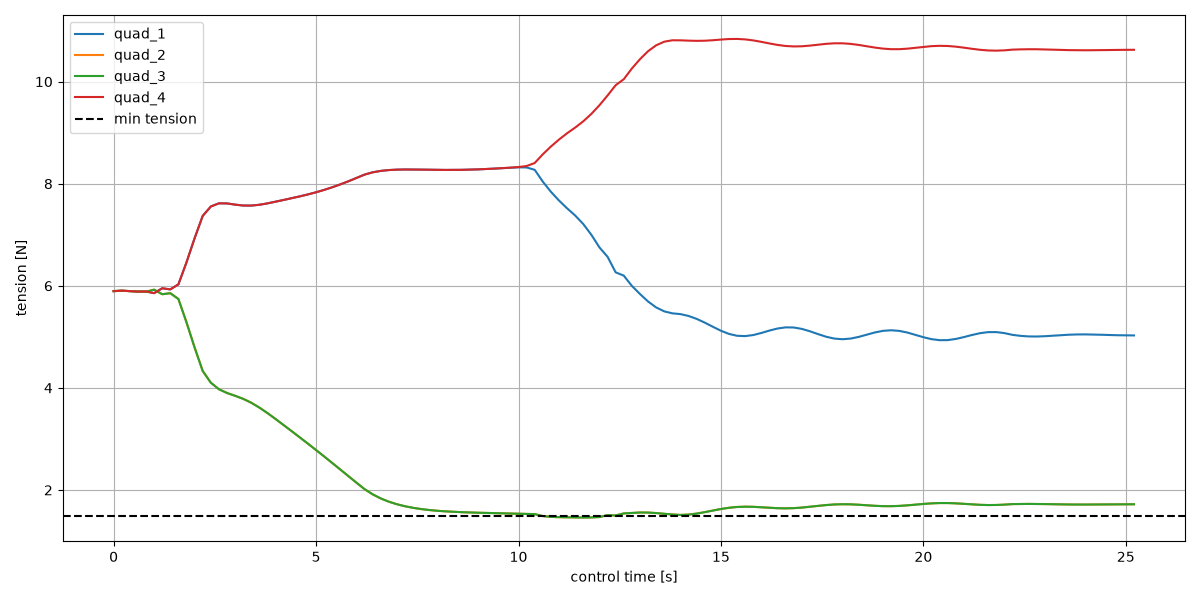}
        \caption{Globally compliant}
        \label{fig:tension_compliant}
    \end{subfigure}
    \hfill
    \begin{subfigure}[t]{0.48\columnwidth}
        \centering
        \includegraphics[width=\linewidth]{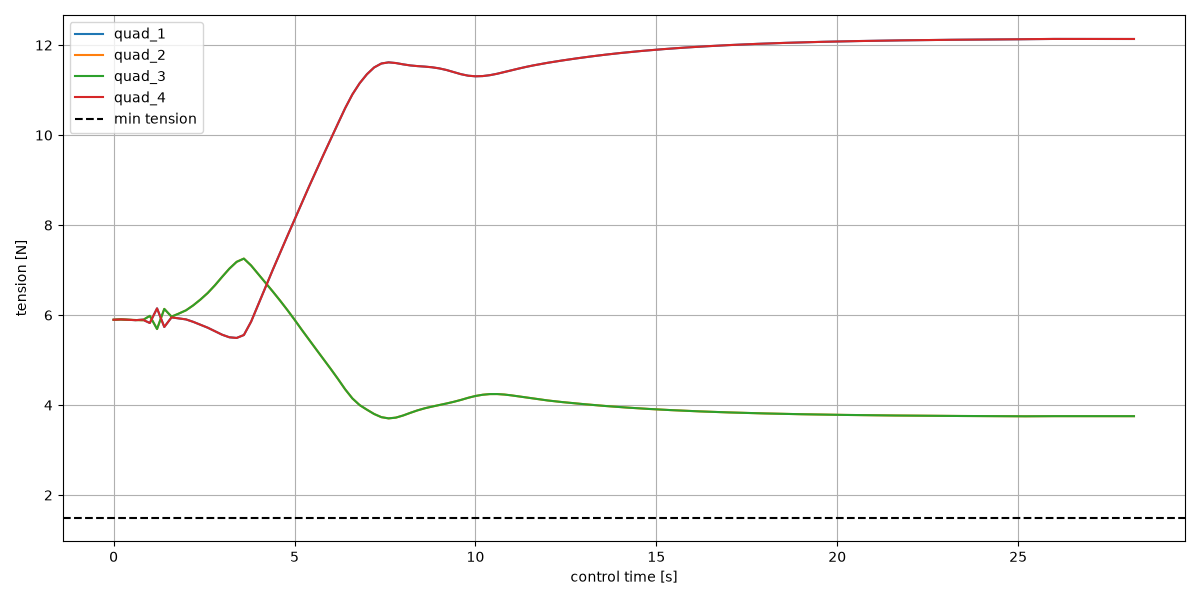}
        \caption{Globally stiff}
        \label{fig:tension_stiff}
    \end{subfigure}

    \caption{Cable tensions during stiffness regulation.}
    \label{fig:tensions_all}
\end{figure}

\begin{figure}[t]
    \centering

    \begin{subfigure}[t]{0.48\columnwidth}
        \centering
        \includegraphics[width=\linewidth]{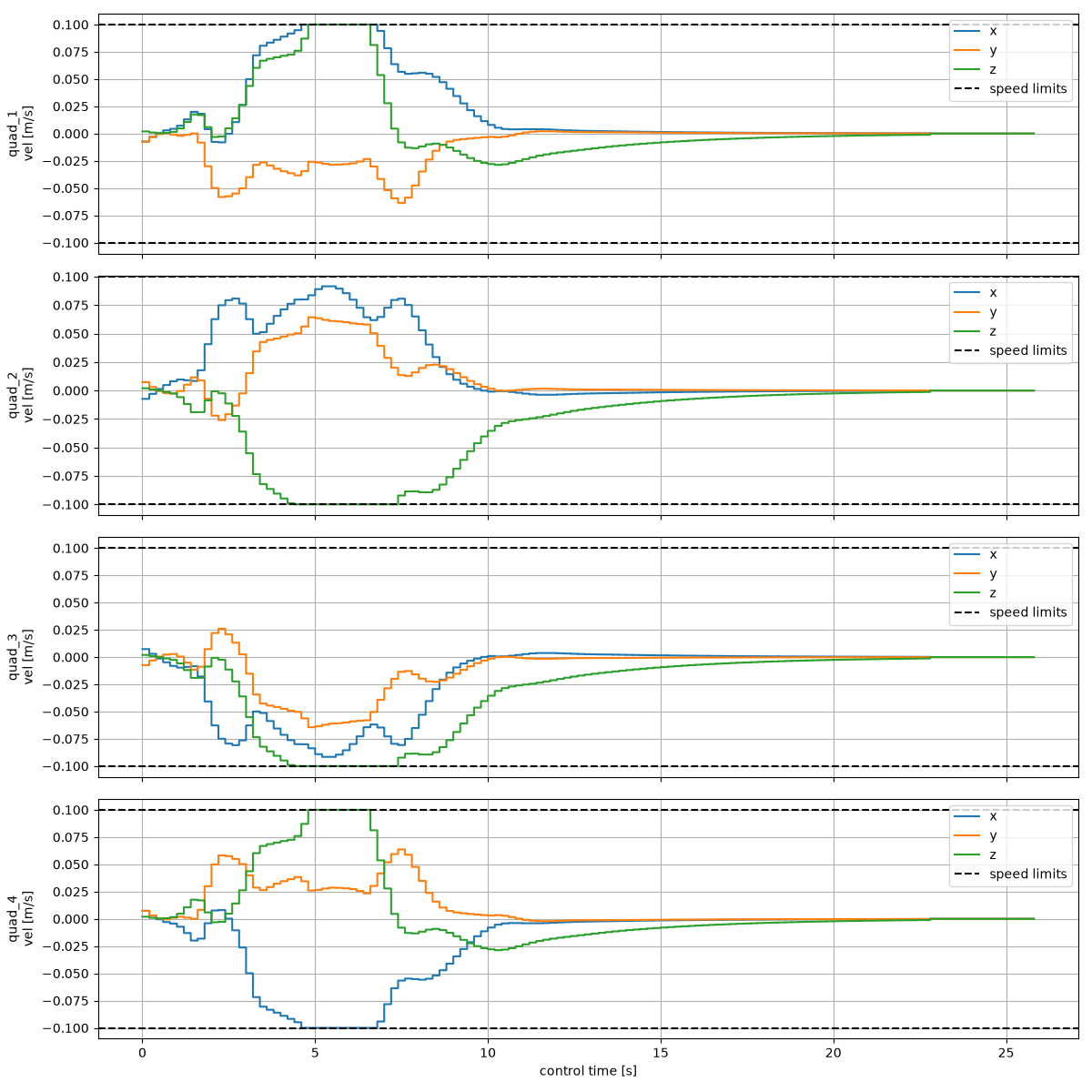}
        \caption{Longitudinally stiff}
        \label{fig:speed_asym_x}
    \end{subfigure}
    \hfill
    \begin{subfigure}[t]{0.48\columnwidth}
        \centering
        \includegraphics[width=\linewidth]{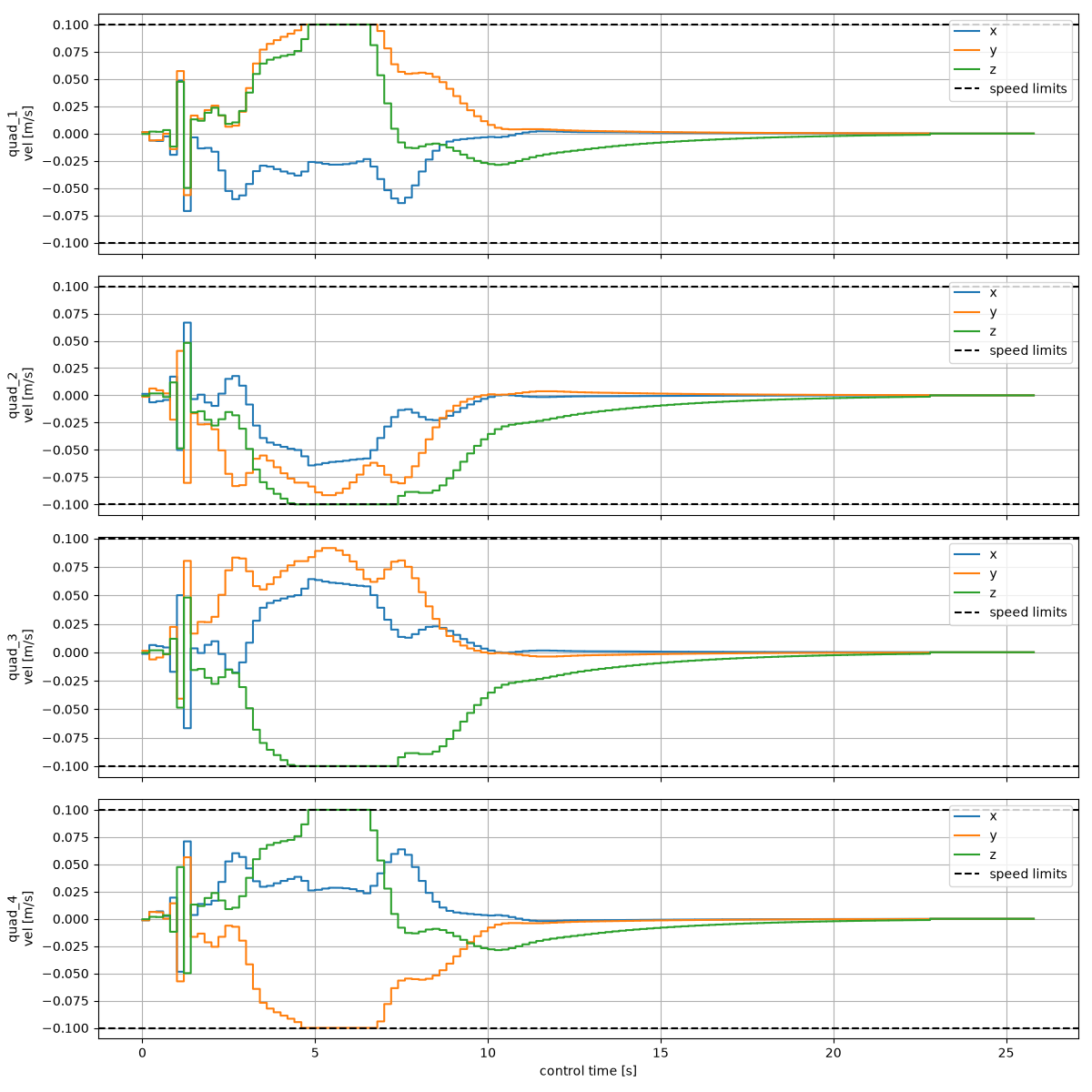}
        \caption{Laterally stiff}
        \label{fig:speed_asym_y}
    \end{subfigure}

    \vspace{0.5em}

    \begin{subfigure}[t]{0.48\columnwidth}
        \centering
        \includegraphics[width=\linewidth]{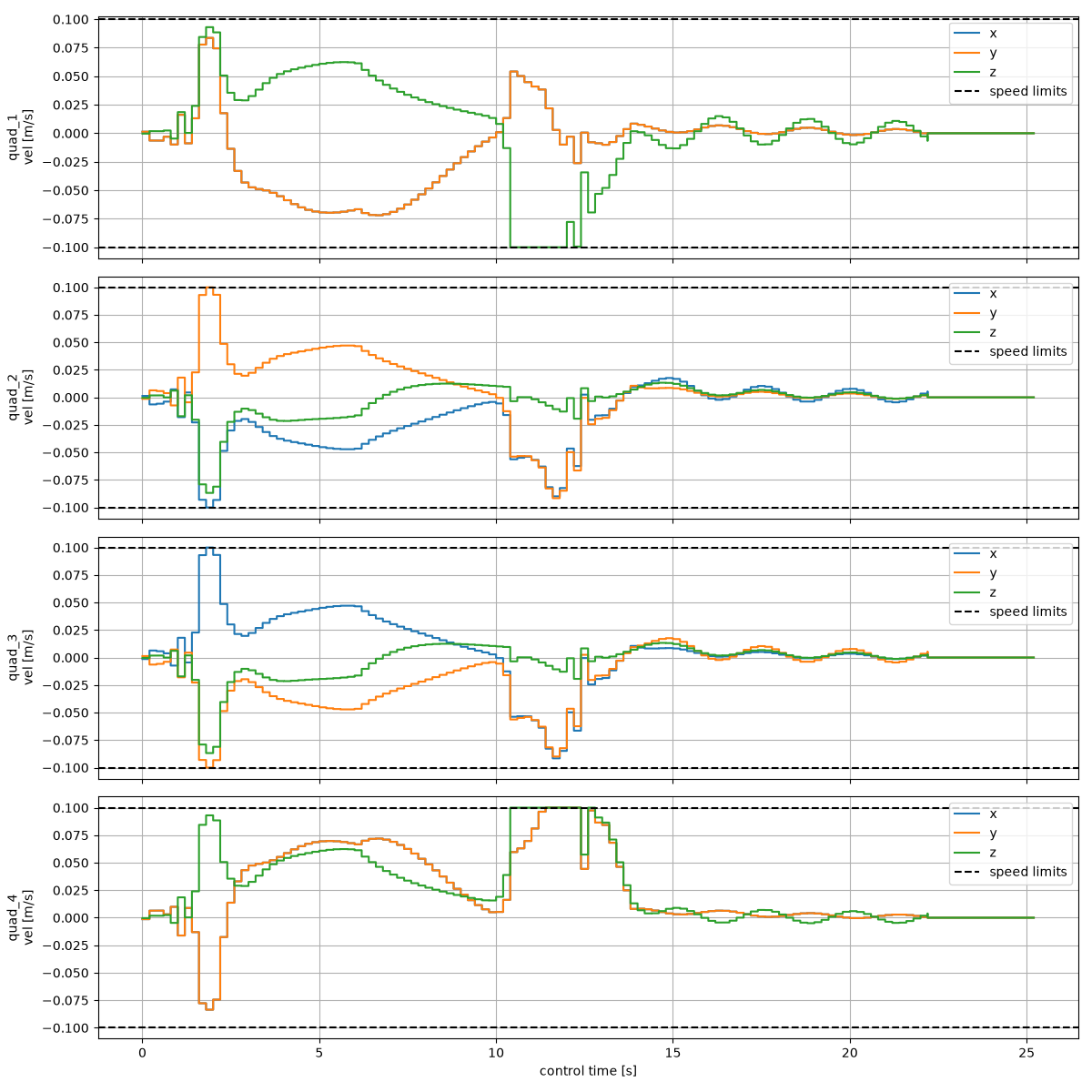}
        \caption{Globally compliant}
        \label{fig:speed_compliant}
    \end{subfigure}
    \hfill
    \begin{subfigure}[t]{0.48\columnwidth}
        \centering
        \includegraphics[width=\linewidth]{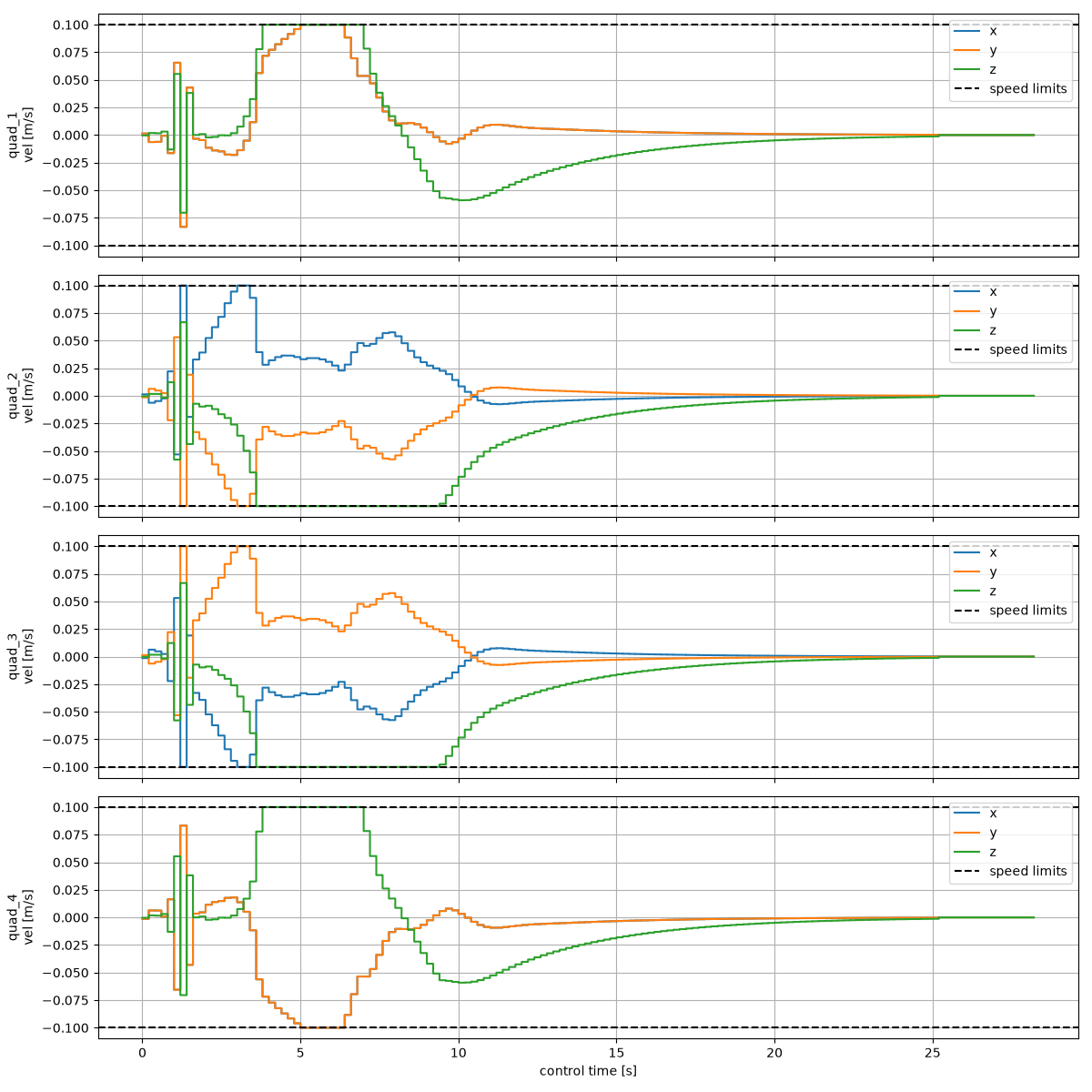}
        \caption{Globally stiff}
        \label{fig:speed_stiff}
    \end{subfigure}

    \caption{Commanded-anchor velocities during stiffness regulation.}
    \label{fig:speeds_all}
\end{figure}

\rev{The implemented gains, weights, active constraints, velocity bounds,
solver, and update period are summarized in
Table~\ref{tab:qp_controller_parameters}. A complete numerical report must
also state the initial configurations, target matrices, Jacobian evaluation
method, solver tolerances, and convergence and blockage thresholds.}
For the simulations reported in this subsection, the high-level planner is implemented as a constrained quadratic program using the local differential quantities introduced in Sections~VI and~VII. Let $q^w \in \mathbb{R}^{3n}$ denote the stacked commanded-anchor coordinates, let $p_e^{w,(b)}(q^w)$ be the selected equilibrium branch, and let $k(q^w)=\mathrm{vech}(K^w(q^w)) \in \mathbb{R}^6$ denote the stiffness-coordinate vector. The payload-position and stiffness errors are defined as
\begin{equation}
e_p(q^w,t) := p_e^{w,(b)}(q^w) - p_d^w(t), \qquad
e_K(q^w,t) := k(q^w) - k_d(t),
\end{equation}
with corresponding Jacobians
\begin{equation}
J_p^w(q^w) := \frac{\partial p_e^{w,(b)}}{\partial q^w}, \qquad
J_K^w(q^w) := \frac{\partial k}{\partial q^w}.
\end{equation}
The implemented controller uses the augmented task Jacobian
\begin{equation}
J_y^w(q^w) :=
\begin{bmatrix}
\sqrt{w_p}\,J_p^w(q^w) \\
\sqrt{w_K}\,J_K^w(q^w)
\end{bmatrix},
\end{equation}
and the corresponding weighted target vector
\begin{equation}
y_{\mathrm{ref}}(q^w,t) :=
\begin{bmatrix}
\sqrt{w_p}\,\lambda_p e_p(q^w,t) \\
\sqrt{w_K}\,\lambda_K e_K(q^w,t)
\end{bmatrix},
\end{equation}
where $w_p>0$ and $w_K>0$ are the payload- and stiffness-objective weights, and $\lambda_p>0$, $\lambda_K>0$ are the corresponding regulation gains. At each planner update, the commanded-anchor velocity $\nu^w \in \mathbb{R}^{3n}$ is obtained as the solution of
\begin{equation}
\label{eq:implemented_qp_controller}
(\nu^w)^\star \in \arg\min_{\nu^w \in \mathbb{R}^{3n}}
\frac{1}{2}\left\|J_y^w(q^w)\,\nu^w + y_{\mathrm{ref}}(q^w,t)\right\|_2^2
+ \frac{\varepsilon_v}{2}\|\nu^w\|_2^2
\end{equation}
subject to the admissibility constraints
\begin{equation}
\nabla_{q^w} c_{T,i}^{\min}(q^w)^\top \nu^w \ge -\eta_T\,c_{T,i}^{\min}(q^w), \qquad
c_{T,i}^{\min}(q^w) := T_i(q^w) - T_{\min},
\end{equation}
for $i=1,\dots,n$, the componentwise anchor-velocity bounds
\begin{equation}
\dot q_{\min}^w \le \nu^w \le \dot q_{\max}^w,
\end{equation}
and the additional stiffness-error descent inequality
\begin{equation}
e_K(q^w,t)^\top w_K J_K^w(q^w)\,\nu^w \le 0.
\end{equation}
Here $\eta_T>0$ is the admissibility gain associated with the minimum-tension margin, $T_{\min}>0$ is the minimum admissible cable tension, and $\varepsilon_v>0$ is the velocity regularization coefficient. In the present implementation, only the lower tension margin is enforced explicitly in the QP; other admissibility functions discussed in Section~VII remain available at the modeling level but are not imposed in the solver for the experiments reported here. After solving~\eqref{eq:implemented_qp_controller}, the commanded-anchor reference is updated according to $q^w \leftarrow q^w + \Delta t_c (\nu^w)^\star$, where $\Delta t_c$ is the high-level controller period.

Table~\ref{tab:qp_controller_parameters} summarizes the default parameters used by the implemented QP controller.

\begin{table*}[t]
\centering
\caption{Default parameters of the implemented QP anchor-setpoint controller.}
\label{tab:qp_controller_parameters}
\begin{tabular}{lll}
\hline
Parameter & Value & Role \\
\hline
Payload regulation gain $\lambda_p$ & 1.0 & Weight on payload-position error \\
Stiffness regulation gain $\lambda_K$ & 1.0 & Weight on stiffness-tracking error \\
Payload objective weight $w_p$ & 1.0 & Relative weight of payload term in the QP cost \\
Stiffness objective weight $w_K$ & 25.0 & Relative weight of stiffness term in the QP cost \\
Velocity regularization $\varepsilon_v$ & $10^{-4}$ & Quadratic penalty on anchor-velocity magnitude \\
Admissibility gain $\eta_T$ & 1.0 & Gain in the minimum-tension inequality \\
Minimum tension threshold $T_{\min}$ & $0.5$~N & Lower cable-tension margin enforced in the QP \\
Componentwise velocity bound $|\nu_j^w|$ & $\leq 0.3$~m/s & Box bound on each anchor-velocity component \\
QP solver & \texttt{qpoases} & Active-set quadratic-program solver \\
Planner update period $\Delta t_c$ & 0.2~s & High-level controller sampling period \\
\hline
\end{tabular}
\end{table*}

Model-based regulation alone does not establish that the desired
stiffness is physically realized in the richer dynamic simulation.
The empirical identification procedure of
Section~\ref{subsec:empirical_stiffness_identification} is therefore
repeated at selected operating points reached during regulation. The
comparison distinguishes
\[
    \DesiredStiffnessMat,
    \qquad
    \PredictedStiffnessMat,
    \qquad
    \EmpiricalStiffnessMat.
\]
This separates regulation error from analytical-model error: the
regulator may accurately attain the model-predicted stiffness while
the empirical stiffness differs because of tendon elasticity,
rigid-payload coupling, or residual dynamic effects.

\rev{Identification at the initial and final equilibria separates
regulation error from model error. Intermediate identification points may
be added when the reconfiguration is sufficiently slow to preserve the
local quasi-static interpretation.}

\subsection{Contact-Rich Guided-Sliding Task}
\label{subsec:guided_sliding_task}

The final validation stage is designed to examine whether stiffness shaping improves
physical interaction rather than only stiffness prediction or
tracking. The task requires the suspended payload to progress through
an alternating guided passage whose geometry forces repeated lateral
contacts.

A common translational component, generated proportionally from the
load-position error, is applied to all commanded anchors to regulate
the payload position without altering the predicted stiffness in the
translation-invariant analytical model.

\subsubsection{Zigzag Environment}
\label{subsubsec:zigzag_environment}

The payload follows a nominal forward path through two alternately
positioned guide surfaces. A collision-free straight trajectory is
incompatible with the passage geometry; progress therefore requires
the payload to yield laterally under contact while maintaining
sufficient stiffness along the forward direction.

This environment exposes the directional role of passive compliance:
high lateral stiffness is expected to increase contact forces and the
risk of jamming, whereas excessive longitudinal compliance may reduce
forward progress.

\rev{The guide geometry, contact parameters, reference motion,
initial condition, and completion criterion form part of the task
definition. Figure~\ref{fig:zigzag_environment} identifies the forward
and lateral directions used to interpret the stiffness profiles.}

\begin{figure}
    \centering
    \includegraphics[width=1.0\linewidth]{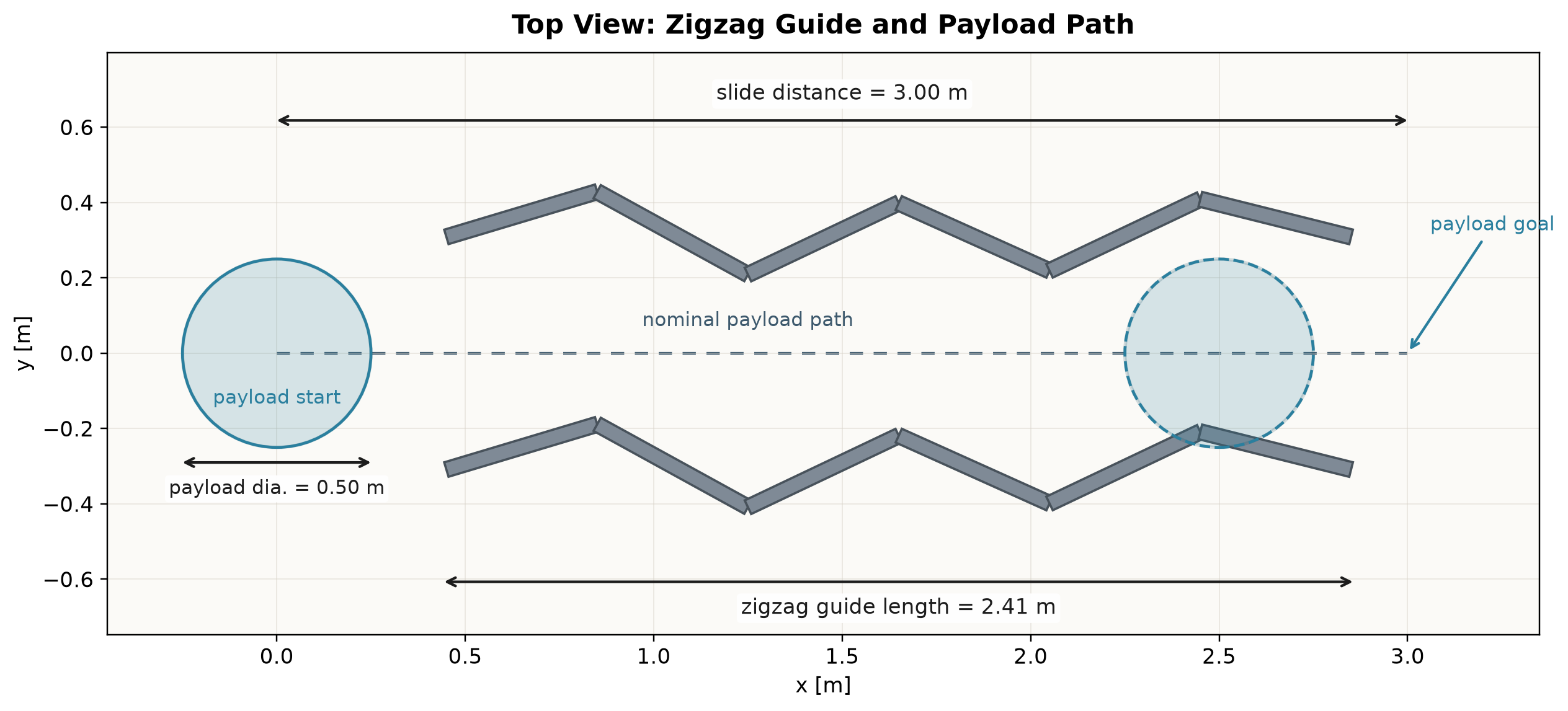}
    \caption{\rev{Schematic of the zigzag guided-sliding task, including the longitudinal direction of progression and the lateral contact direction.}}
    \label{fig:zigzag_environment}
\end{figure}

\begin{figure*}[t]
    \centering

    \begin{subfigure}[t]{0.35\textwidth}
        \centering
        \includegraphics[width=\linewidth]{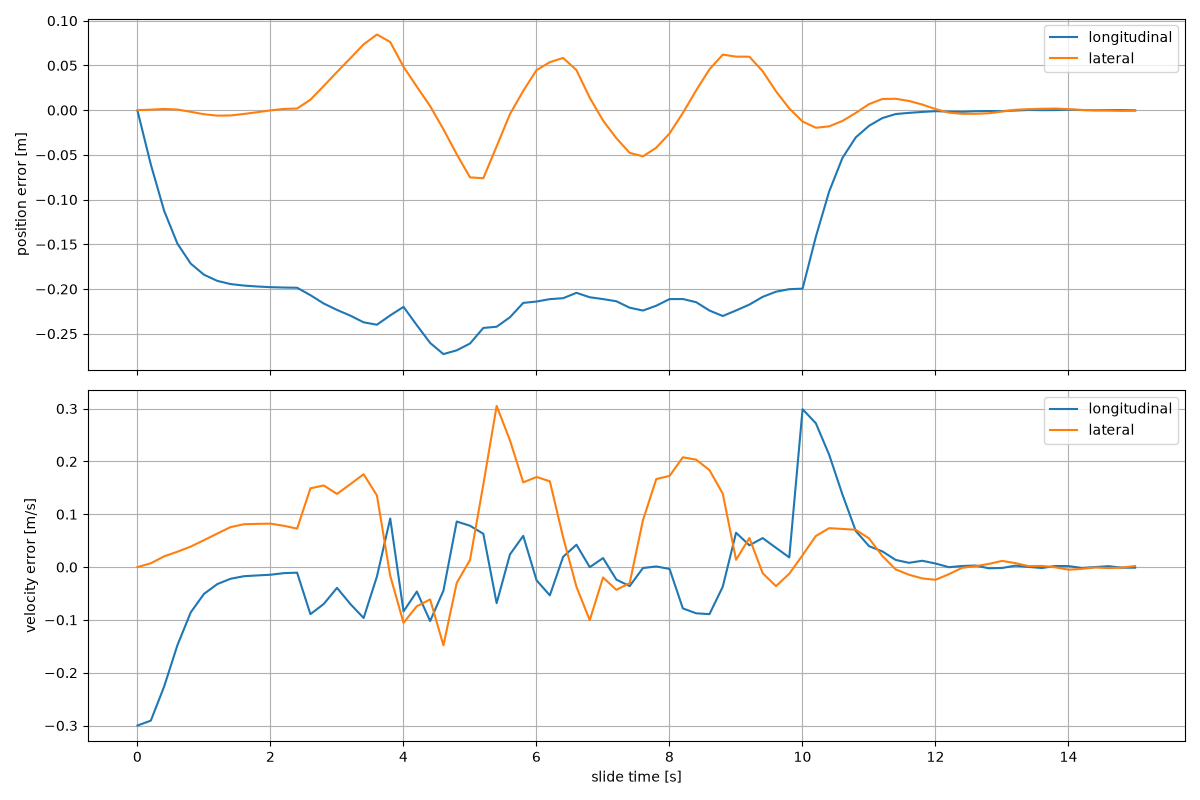}
        \caption{Longitudinally stiff, laterally compliant}
        \label{fig:zigzag_longitudinal_stiff}
    \end{subfigure}
    % \hfill
    \begin{subfigure}[t]{0.35\textwidth}
        \centering
        \includegraphics[width=\linewidth]{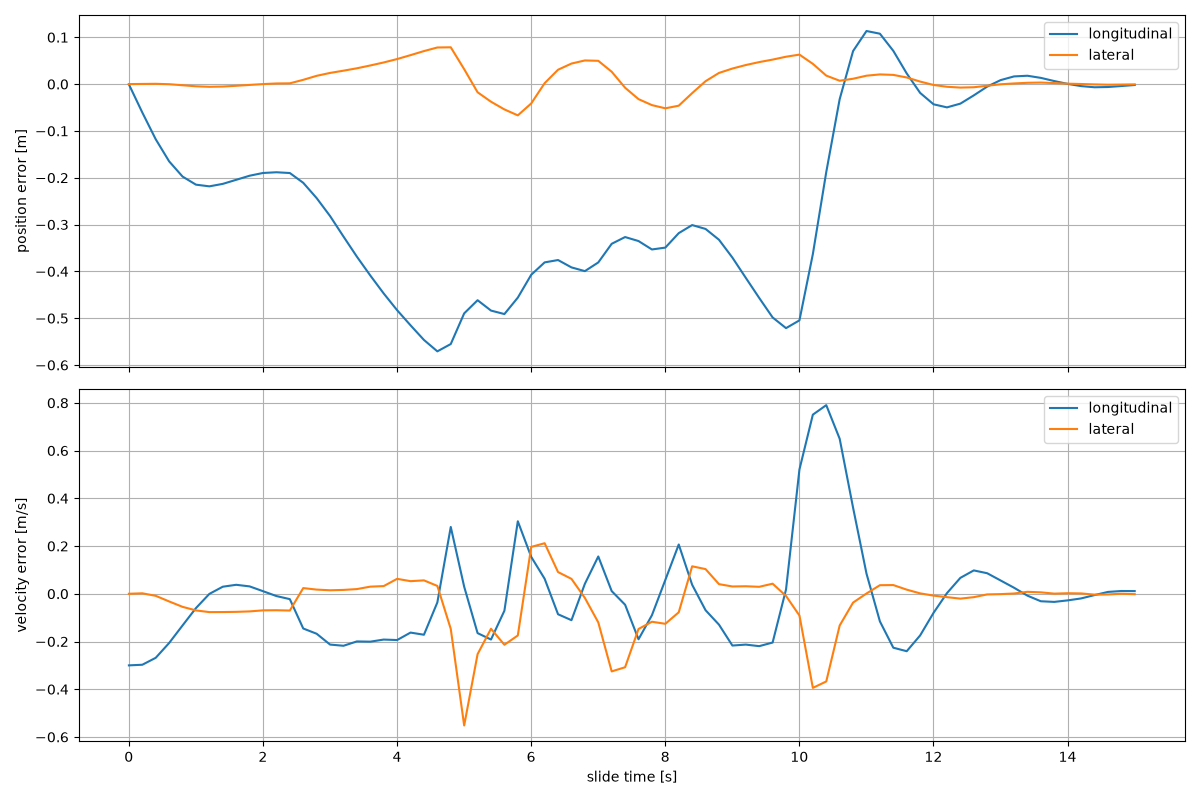}
        \caption{Laterally stiff, longitudinally compliant}
        \label{fig:zigzag_lateral_stiff}
    \end{subfigure}

    \vspace{0.5em}

    \begin{subfigure}[t]{0.35\textwidth}
        \centering
        \includegraphics[width=\linewidth]{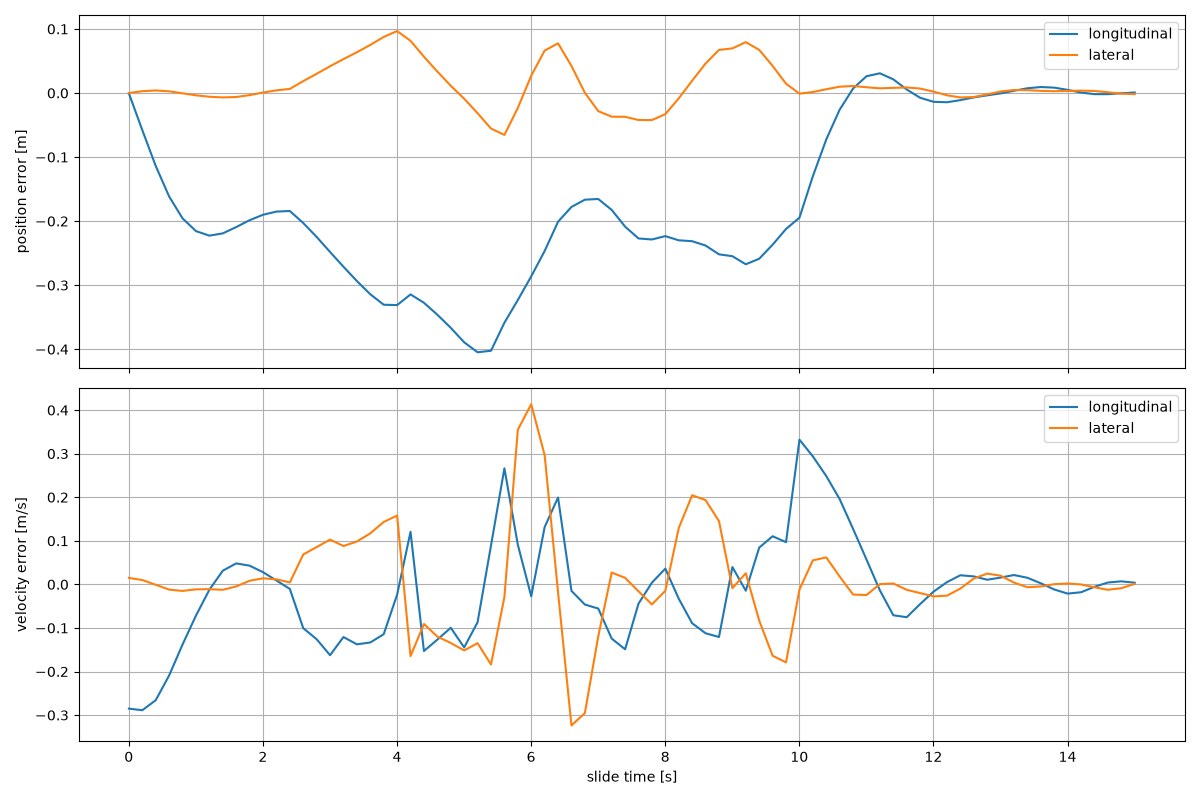}
        \caption{Compliant in both horizontal directions}
        \label{fig:zigzag_compliant}
    \end{subfigure}
    % \hfill
    \begin{subfigure}[t]{0.35\textwidth}
        \centering
        \includegraphics[width=\linewidth]{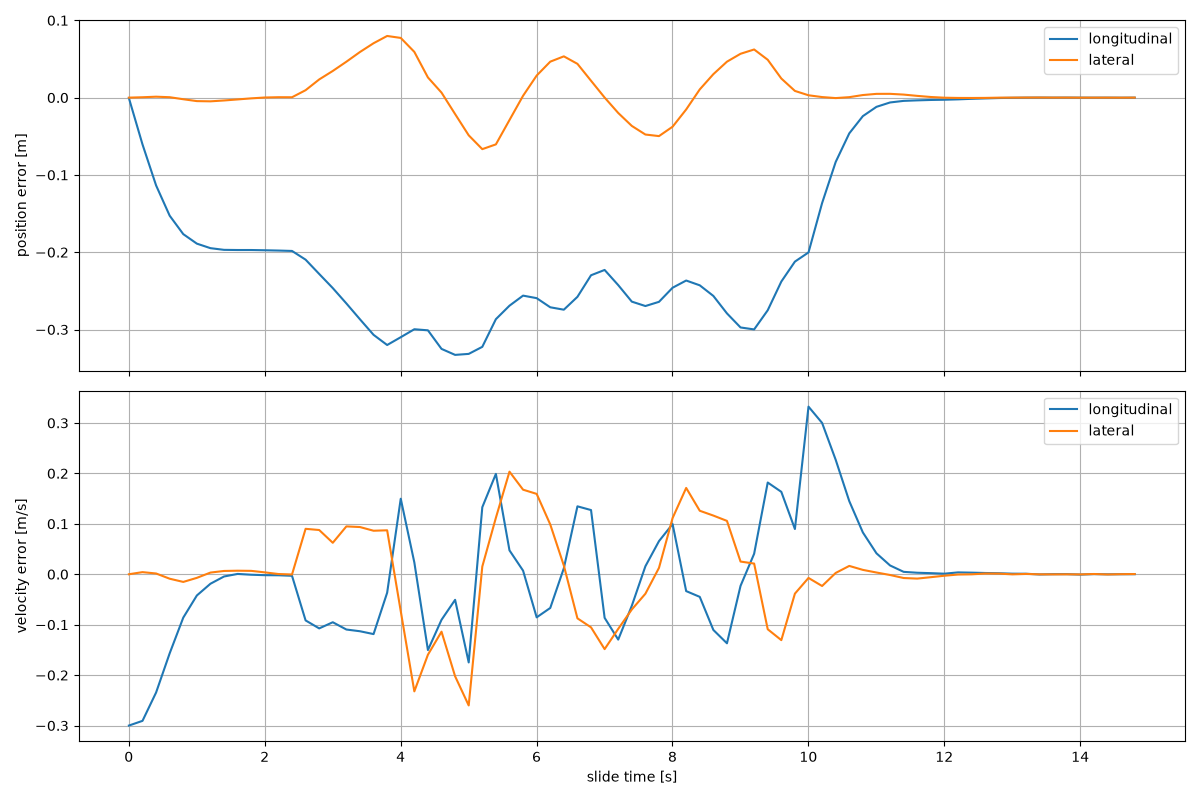}
        \caption{Stiff in both horizontal directions}
        \label{fig:zigzag_stiff}
    \end{subfigure}

    \caption{Longitudinal and lateral tracking errors during guided sliding.}
    \label{fig:zigzag_errors_all}
\end{figure*}

\subsubsection{Passive-Stiffness Profiles}
\label{subsubsec:task_stiffness_profiles}

\begin{figure*}[t]
    \centering

    \begin{subfigure}[t]{0.35\textwidth}
        \centering
        \includegraphics[width=\linewidth]{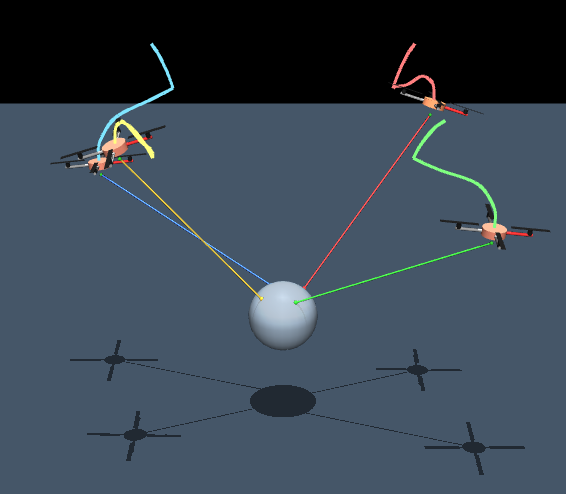}
        \caption{\rev{Longitudinally stiff and laterally compliant formation.}}
        \label{fig:formation_longitudinal_stiff}
    \end{subfigure}
    % \hfill
    \begin{subfigure}[t]{0.35\textwidth}
        \centering
        \includegraphics[width=\linewidth]{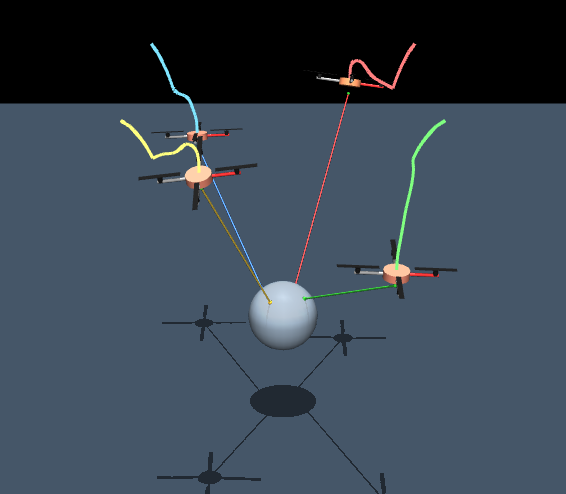}
        \caption{\rev{Laterally stiff and longitudinally compliant formation.}}
        \label{fig:formation_lateral_stiff}
    \end{subfigure}
    % \hfill
    \begin{subfigure}[t]{0.35\textwidth}
        \centering
        \includegraphics[width=\linewidth]{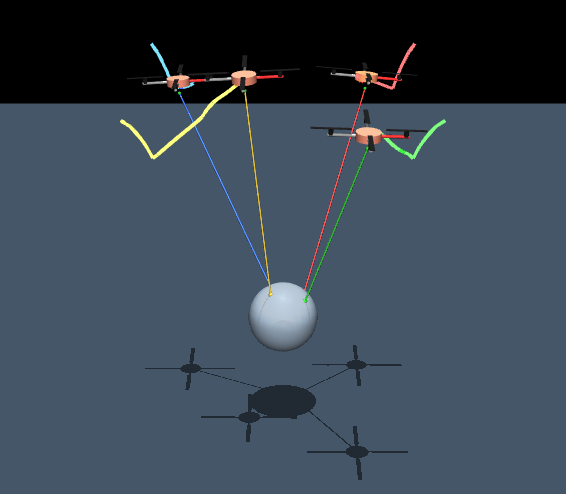}
        \caption{\rev{Longitudinally and laterally compliant formation.}}
        \label{fig:formation_globally_compliant}
    \end{subfigure}
    % \hfill
    \begin{subfigure}[t]{0.35\textwidth}
        \centering
        \includegraphics[width=\linewidth]{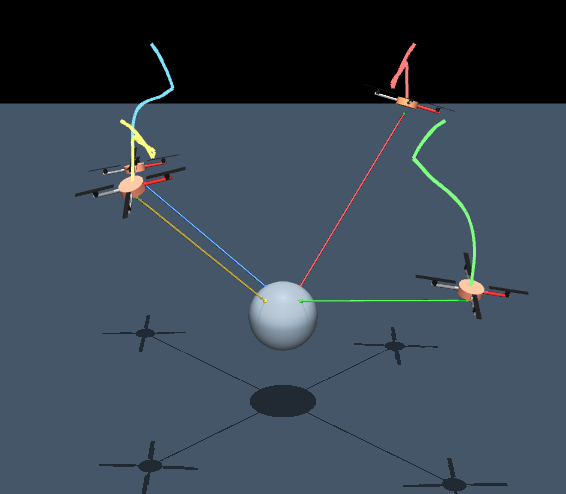}
        \caption{\rev{Longitudinally and laterally stiff formation.}}
        \label{fig:formation_globally_stiff}
    \end{subfigure}

    \caption{\rev{Commanded-anchor formations used to generate the four passive-stiffness profiles for the guided-sliding comparison.}}
    \label{fig:task_stiffness_formations}
\end{figure*}

Four stiffness profiles are defined for comparison:
\begin{enumerate}
    \item longitudinally and laterally stiff;
    \item longitudinally and laterally compliant;
    \item longitudinally stiff and laterally compliant;
    \item longitudinally compliant and laterally stiff.
\end{enumerate}
The first two distinguish global stiffness from global compliance. The
last two test whether task performance depends on aligning the
stiffness anisotropy with the forward-motion and contact directions.

Each profile is realized by aerial-anchor reconfiguration before task
execution. The resulting anchor setpoints are then held fixed, so the
comparison concerns passive mechanical behavior rather than active
stiffness adaptation during contact.

\rev{A conclusive comparison requires the target, predicted, and
empirically identified stiffness matrices for all four profiles, together
with their formations, tensions, and equilibrium load poses. The profiles
must satisfy common initialization and feasibility conditions.}

\subsubsection{Task-Performance Metrics}
\label{subsubsec:task_performance_metrics}

The profiles will be compared using task-completion rate and traversal
time, together with maximum and accumulated contact force, forward
progress, path error, payload-attitude excursion, tendon-tension
margins, and vehicle control effort. Jamming is defined by insufficient
forward progress over a prescribed time interval while persistent
contact is present.

The central comparison is between the task-aligned anisotropic profile
and the globally stiff, globally compliant, and task-misaligned
profiles. Improved performance of the aligned profile would show that
the benefit arises from directional stiffness shaping rather than from
uniformly reducing stiffness.

\rev{Repeated trials or randomized initial conditions are required for
task-level claims. Failed and jammed trials must be retained, with a stated
rule for incomplete traversal times.}
\section{Conclusions}
\label{sec:conclusions}

This work developed a gravity-aware theory for predicting and shaping
passive Cartesian stiffness in cable-suspended aerial manipulation
with movable compliant aerial anchors. For an arbitrary number of
aerial vehicles and taut inextensible cables, the stiffness was derived
at a selected gravity-loaded equilibrium. Within each leg,
aerial-anchor compliance and transverse cable geometric compliance
combine in series, whereas the leg stiffnesses act in parallel on the
load.

For isotropic aerial-anchor stiffness, each inextensible cable and
compliant anchor was shown to be quasi-statically equivalent to a
virtual unilateral elastic cable connected directly to the commanded
anchor. This equivalence reveals an axial--transverse decomposition and
clarifies how cable directions and gravity-loaded tensions determine
the magnitude and anisotropy of the passive load stiffness.

The resulting nonlinear stiffness map relates the commanded-anchor
configuration to the passive stiffness along a selected equilibrium
branch. Its differential characterizes the stiffness variations
locally generated by aerial-anchor motion and supports
constraint-preserving regulation toward an arbitrary desired
stiffness. When further local descent is unavailable, the regulator
preserves admissibility and reports local blockage without asserting
global nonrealizability.

The theory assumes a point load, quasi-static evolution, and taut,
straight, inextensible cables. \rev{A dynamic rigid-body validation framework with nonlinear vehicle control, elastic-damped tendons, and environmental contact specifies how the predicted and shaped stiffness will be stress-tested beyond these assumptions.} Future work will address dynamic
load-port impedance, variable aerial-anchor stiffness and damping,
elastic and sagging cables, slack--taut transitions, rigid-body loads,
global reconfiguration, and experimental validation with aerial-robot
teams.

\bibliographystyle{IEEEtran}
\bibliography{references}

\end{document}